\documentclass{article}

\PassOptionsToPackage{numbers, compress}{natbib}
\usepackage[preprint]{neurips_2026}

\usepackage[utf8]{inputenc} 
\usepackage[T1]{fontenc}    
\usepackage[colorlinks=true,linkcolor=blue,citecolor=blue,urlcolor=blue]{hyperref}
\usepackage{url}            
\usepackage{booktabs}       
\usepackage{amsfonts}       
\usepackage{nicefrac}       
\usepackage{microtype}      
\usepackage{xcolor}         

\usepackage[table]{xcolor}
\usepackage[utf8]{inputenc}
\usepackage{algorithm}
\usepackage{algpseudocode}
\usepackage{aliascnt}
\usepackage{amsfonts}
\usepackage{amsmath}
\usepackage{amssymb}
\usepackage{amsthm}
\usepackage{bm}
\usepackage{booktabs}
\usepackage{comment}
\usepackage[inline]{enumitem}
\usepackage{etoc}
\usepackage{fancyhdr}
\usepackage{graphicx}
\usepackage{ifthen}
\usepackage{makecell}
\usepackage{mathrsfs}
\usepackage{mathtools}
\usepackage{multirow}
\usepackage{parskip}
\usepackage{pdfpages}
\usepackage{pifont}
\usepackage{stmaryrd}
\usepackage{subcaption}
\usepackage{upgreek}
\usepackage{wrapfig}
\usepackage{bbm}
\usepackage{tcolorbox}
\usepackage{refcount}
\usepackage{todonotes}
\newcommand{\Input}{\item[\textbf{Input:}]}
\newcommand{\Output}{\item[\textbf{Output:}]}

\newcommand\blfootnote[1]{%
  \begingroup
  \renewcommand\thefootnote{}\footnote{#1}%
  \addtocounter{footnote}{-1}%
  \endgroup
}

\usepackage[capitalize,noabbrev]{cleveref}

\def\mod{\ \mathrm{mod}}

\def \bfX {\mathbf{X}}
\def \hbfX {\hat{\bfX}}
\def \bfXb {\mathbf{X}^{\rmb}}

\def \bfXf {\mathbf{X}^{\rmf}}

\def\bfXbs{\hat{\bfX}^{\rmb}}
\def\bfXfs{\hat{\bfX}^{\rmf}}
\def \bfB {\mathbf{B}}
\def \bfY {\mathbf{Y}}

\def\Qbb {\mathbb{Q}}
\def\Pbb {\mathbb{P}}

\def\Pmeasure{\mathscr{P}}

\def\msi{\mathsf{I}}

\def\msx{\mathsf{X}}

\def\rset{\mathbb{R}}

\def\nset{\mathbb{N}}

\def\rmd{\mathrm{d}}

\def\rmf{\mathrm{f}}

\def\rmC{\mathrm{C}}

\def\rmb{\mathrm{b}}

\def\KL{\mathrm{KL}}
\def\ent{\mathrm{H}}

\newcommand{\defEns}[1]{\left\lbrace #1 \right\rbrace }

\newcommand{\PP}{\mathbb{P}}
\newcommand{\QQ}{\mathbb{Q}}
\newcommand{\QQr}{\mathbb{Q}^{\mathrm{ref}}}
\newcommand{\QQsg}{\mathbb{Q}^{\sigma}}
\newcommand{\QQVsg}{\mathbb{Q}^{V,\sigma}}

\newcommand\probaMarkovTilde[2][2=]
{\ifthenelse{\equal{#2}{}}{{\widetilde{\mathbb{P}}_{#1}}}{\widetilde{\mathbb{P}}_{#1}\left[ #2\right]}}

\newcommand{\PE}{\mathbb{E}} 

\newcommand{\plusinfty}{+\infty}
\def\ie{\textit{i.e.}}

\def\eqsp{\;}

\newcommand{\ocint}[1]{\left(#1\right]}

\newcommand{\ccint}[1]{\left[#1\right]}

\newcommand\sequence[3][2=,3=]
{\ifthenelse{\equal{#3}{}}{\ensuremath{\{ #1_{#2}\}}}{\ensuremath{\{ #1_{#2}, \eqsp #2 \in #3 \}}}}
\newcommand\sequenceD[3][2=,3=]
{\ifthenelse{\equal{#3}{}}{\ensuremath{\{ #1_{#2}\}}}{\ensuremath{( #1)_{ #2 \in #3} }}}

\newcommand\sequenceDouble[4][3=,4=]
{\ifthenelse{\equal{#3}{}}{\ensuremath{\{ (#1_{#3},#2_{#3}) \}}}{\ensuremath{\{  (#1_{#3},#2_{#3}), \eqsp #3 \in #4 \}}}}

\def\eg{e.g.}

\def\Idd{\mathrm{I}_d}

\newcommand{\ensembleLigne}[2]{\{#1\,:\eqsp #2\}}

\def\path{\operatorname{path}}

\newcommand{\beq}{\begin{equation}}
\newcommand{\eeq}{\end{equation}}

\def\Leb{\mathrm{Leb}}

\def\setFunConT{\mathbf{C}_T}

\def\densityGaussian{\mathrm{N}}
\def\gauss{\mathrm{N}}
\newcommand{\norm}[1]{\left\lVert#1\right\rVert}

\definecolor{pearDark}{HTML}{2980B9}

\def\projM{\operatorname{proj}_{\mathcal{M}_T}}
\def\projR{\operatorname{proj}_{\mathcal{R}(\Qbb^{\text{ref}})}}
\def\projMdisc{\operatorname{proj}_{\mathcal{M}_N}}
\def\projRdisc{\operatorname{proj}_{\mathcal{R}(q^{\sigma,g})}}
\def\adminControlT{\mathbf{A}_T}
\def\Qbb {\mathbb{Q}}
\def\Qbbsg {\mathbb{Q}^{\sigma}}
\def\QbbVsg {\mathbb{Q}^{V,\sigma}}
\def\Pbb {\mathbb{P}}
\def\Mbb {\mathbb{M}}

\DeclareMathOperator*{\argmin}{arg\,min}

\definecolor{byzantine}{rgb}{0.74, 0.2, 0.64}
\definecolor{darkgreen}{RGB}{125, 255, 125}
\definecolor{lightgreen}{RGB}{230, 230, 230}
\definecolor{mediumgreen}{RGB}{200, 255, 200}

\def\bfu{\mathbf{u}}

\def\law{\mathrm{Law}}

\crefname{theorem}{theorem}{Theorems}
\Crefname{Theorem}{Theorem}{Theorems}
\newtheorem*{lemma_nonumber*}{Lemma}
\newaliascnt{lemma}{theorem}
\newtheorem{lemma}[lemma]{Lemma}
\aliascntresetthe{lemma}
\crefname{lemma}{lemma}{lemmas}
\Crefname{Lemma}{Lemma}{Lemmas}
\newaliascnt{corollary}{theorem}
\newtheorem{corollary}[corollary]{Corollary}
\aliascntresetthe{corollary}
\crefname{corollary}{corollary}{corollaries}
\Crefname{Corollary}{Corollary}{Corollaries}
\newaliascnt{proposition}{theorem}
\newtheorem{proposition}[proposition]{Proposition}
\aliascntresetthe{proposition}
\crefname{proposition}{proposition}{propositions}
\Crefname{Proposition}{Proposition}{Propositions}
\newaliascnt{definition}{theorem}

\aliascntresetthe{definition}
\crefname{definition}{definition}{definitions}
\Crefname{Definition}{Definition}{Definitions}
\newaliascnt{remark}{theorem}

\aliascntresetthe{remark}
\crefname{remark}{remark}{remarks}
\Crefname{Remark}{Remark}{Remarks}
\crefname{figure}{figure}{figures}
\Crefname{Figure}{Figure}{Figures}
\Crefname{assumption}{\textbf{A}\hspace{-3pt}}{\textbf{A}\hspace{-3pt}}
\crefname{assumption}{\textbf{A}}{\textbf{A}}
\crefformat{assumption}{{\textbf{A}}#2#1#3}
\crefname{assumption}{assumption}{assumptions}
\Crefname{Assumption}{Assumption}{Assumptions}

\AtBeginEnvironment{appendix}{\crefalias{subsubsection}{appendix}}
\AtBeginEnvironment{appendix}{\crefalias{subsection}{appendix}}
\AtBeginEnvironment{appendix}{\crefalias{section}{appendix}}
\AtBeginEnvironment{appendix}{\crefalias{chapter}{appendix}}

\title{Twisted Schrödinger Bridge Matching}

\author{%
  Maxence Noble$^{*,1}$
  \quad
  Marie Scheid$^{*,1}$
  \quad
  Yazid Janati$^{2, 3}$
  \quad 
  Eric Moulines$^{3, 4}$
  \quad
  Alain Durmus$^{1}$\\[0.3em]
  $^1$ CMAP, Ecole polytechnique \quad $^2$ Institute of Foundation Models \quad $^3$ MBZUAI \quad $^4$ EPITA, LRE
}

\begin{document}

\maketitle

\begin{abstract}
Over the past few years, diffusion-based Schrödinger bridge models have been proposed to approximate optimal transport dynamics between two prescribed boundary distributions, with successful applications to generative modeling. More precisely, these methods aim to estimate a path measure whose initial and terminal marginals match the two boundary distributions, while minimizing the Kullback–Leibler divergence with respect to a reference Markov process.
In this work, we consider the generalized Schrödinger bridge problem, in which the reference process is a twisted Brownian motion, that is, a Feynman–Kac transform of a Brownian motion induced by a time-dependent differentiable potential.  Building on the Iterative Markovian Fitting (IMF) paradigm, and in particular on its special case Diffusion Schrödinger Bridge Matching (DSBM), which corresponds to the zero potential case, we introduce {Twisted Schrödinger Bridge Matching} (TSBM), a diffusion-based method designed to handle both continuous- and discrete-time potentials.
Unlike previous approaches, TSBM provides a rigorous extension of the IMF scheme to the generalized Schrödinger bridge problem. This derivation leads to a new bridge-matching loss that depends explicitly on the gradient of the potential and recovers the DSBM objective when the potential vanishes,  yielding improved performance. We further introduce trajectory-based variance-reduction techniques that substantially stabilize optimization and may be useful beyond the present setting. Finally, we empirically demonstrate the benefits of TSBM for trajectory inference across increasingly high-dimensional settings, including crowd navigation and single-cell data. Code available at \url{https://github.com/maxencenoble/twisted-sb-matching}.
\end{abstract}

\blfootnote{* Equal contribution. Corresponding author: \texttt{maxence.noble@gmail.com}}

\etocdepthtag.toc{main}

\vspace{-0.4cm}
\section{Introduction}

The problem of learning a transport map between two probability distributions, a source $\mu$ and a target $\nu$, is central in modern machine learning, with prominent applications in generative modeling. Recent \emph{bridge matching} methods \citep{albergo2025stochastic} address this problem dynamically, by learning a Markov process on a time interval $[0,T]$ that transports $\mu$ to $\nu$. In continuous time, this typically amounts to learning the drift of a diffusion process. The construction starts from an interpolation process $(\bfX_t)_{t\in\ccint{0,T}}$ whose endpoints follow a coupling $\pi_{0,T}$ of $(\mu,\nu)$, while its conditional law given $(\bfX_0,\bfX_T)$ is fixed and tractable, often chosen as a Brownian bridge. Under suitable conditions, the induced density path admits a Markov projection, and learning its drift reduces to a mean-square regression problem. This perspective provides a unified view of several diffusion and flow-based generative models \citep{song2020score,lipman2022flow}.

Bridge matching also recovers optimal-transport dynamics for appropriate choices of the endpoint coupling. In particular, if $\pi_{0,T}$ is the entropy-regularized optimal transport (EOT) plan between $\mu$ and $\nu$, and if the conditional interpolation is a Brownian bridge, then the Markov projection solves the dynamical counterpart of EOT; see, e.g., \cite{DeBortoli2021diffusion}. Equivalently, it solves the canonical Schrödinger bridge (SB) problem\footnote{Throughout this paper, we may also refer to this instance as the classic or standard SB problem.} \citep{schrodinger1932theorie} with Brownian motion as reference path measure.

In general, the SB problem provides a variational framework for trajectory inference. Given a reference path measure $\QQr$, it seeks a path measure $\Pbb$ minimizing $\KL(\Pbb\mid\QQr)$ subject to the endpoint constraints $\Pbb_0=\mu$ and $\Pbb_T=\nu$. Under appropriate conditions, the solution is the unique path measure satisfying these marginal constraints, being Markov, and belonging to the reciprocal class of $\QQr$, meaning that it has the same conditional law as $\QQr$ given its endpoints \citep{leonard2014survey}. This characterization has motivated fixed-point algorithms based on KL projections onto subsets of these constraints. In particular, the Iterative Markovian Fitting (IMF) scheme alternates reciprocal and Markovian projections, and its bridge-matching implementation in the canonical setting is Diffusion Schrödinger Bridge Matching (DSBM) \citep{shi2023diffusion}.

However, the canonical SB formulation can be too restrictive for trajectory inference problems where additional information is available at intermediate times. For instance, sparse snapshot observations may provide useful guidance on the dynamics without being rich enough to define prescribed intermediate marginals, see, e.g., \cite{tamir2023transport}, thereby limiting the applicability of multi-marginal SB formulations \citep{chen2023deepmomentum,noble2023tree}. This motivates generalized SB problems with state costs. In this direction, \cite{liu2023generalized} consider the stochastic optimal control formulation of the generalized SB problem, augmenting the standard kinetic energy on the drift with a time-dependent potential that encodes intermediate information. They propose Generalized Schrödinger Bridge Matching (GSBM), an iterative bridge-matching algorithm in the spirit of DSBM. Despite strong empirical performance, the connection between its updates and the IMF scheme, which should in principle extend to general SB-type problems, remains unclear. This work fills this gap.

\paragraph{Contributions.}
In this paper, we focus on the generalized SB problem introduced by \cite{liu2023generalized} and make the following contributions:
\begin{itemize}[wide, labelwidth=!, labelindent=0pt]
    \item We reformulate this generalized SB problem through the lens of the \emph{Twisted Schrödinger Bridge} (TSB), where the reference process is a twisted Brownian motion, i.e., a Feynman--Kac transform of Brownian motion. This formulation covers both continuous- and discrete-time state costs. We then extend the IMF scheme to each setting, obtaining update rules that differ from GSBM. In particular, we cast the reciprocal projection as a tractable variational problem with respect to the twisted Brownian bridge, and derive a Markovian projection loss that explicitly involves the state-cost gradient. Notably, unlike GSBM, this loss exactly recovers standard diffusion-model objectives in the zero-potential case: the DSBM loss in continuous time and the DDPM loss \cite{Ho2020denoising} in discrete time.

    \item Building on this formulation, we introduce \emph{Twisted Schrödinger Bridge Matching} (TSBM), a methodology for solving the TSB problem from samples of the initial and terminal marginals. A key component is a learnable control-variate scheme that reduces variance and improves the estimation of the Markovian projection. We further adapt the efficient DSBM implementation from \citep{bortoli2024schrodinger} to obtain a bidirectional version of TSBM that stabilizes the training procedure.

    \item We empirically validate TSBM on crowd-navigation and single-cell modeling tasks across several dimensions. Our results show that TSBM improves the estimation of both reciprocal and Markovian projections compared with GSBM, while achieving comparable or better performance on the SB-based objective.
\end{itemize}

\paragraph{Notation.} For readability, we collect the notation used throughout the paper in \Cref{tsbm:app:notation}.

In the rest of the paper, we fix a scaling hyperparameter $\sigma>0$ and denote by $\Qbbsg$ the distribution of the process  $(\sigma \bfB_t)_{t\in [0,T]}$, where $(\bfB_t)_{t\geq 0}$ is a standard Brownian motion.


\section{From Generalized to Twisted Schrödinger Bridge} 
\label{tsbm:sec:background}

\paragraph{Schrödinger Bridge problem.}
In its most general formulation, the Schrödinger Bridge (SB) problem consists in finding a path measure $\Pbb^\star \in \Pmeasure(\setFunConT)$ solving the constrained variational problem
\begin{align}
  \label{tsbm:eq:classic_sb}
  \argmin \ensembleLigne{\KL(\Pbb\|\Qbb^{\text{ref}})}{\Pbb \in \Pmeasure(\setFunConT), \Pbb_0 = \mu, \ \Pbb_T = \nu} \eqsp .
\end{align}
Here, $\Qbb^{\text{ref}} \in \Pmeasure(\setFunConT)$ denotes a reference path measure.
From its original formulation \citep{schrodinger1932theorie} to modern developments \citep{DeBortoli2021diffusion}, the SB problem \eqref{tsbm:eq:classic_sb} has primarily been studied with Brownian motion as the reference measure.
In this setting, the conditional measure $\Qbb^{\text{ref}}_{|0,T}$ reduces to the Brownian bridge, which is tractable and can be easily sampled from.

Moreover, \cite{leonard2014survey} showed in this case that the path measure formulation \eqref{tsbm:eq:classic_sb} is equivalent to the following affine Stochastic Optimal Control (SOC) problem
\begin{align} \label{tsbm:eq:standard_soc}
  &  \textstyle\argmin \ensembleLigne{\int_{0}^T \frac{1}{2}\PE [\|u_t(\bfX^u_t)\|^2]\rmd t}{u\in \adminControlT(\nu) \eqsp , \quad (\bfX_t^u)_{t\in\ccint{0,T}} \text{ solves \eqref{eq:contr_diff} }} \eqsp , \\
  \label{eq:contr_diff}
  & \rmd \bfX_t^{u} = u_t(\bfX_t^{u})\rmd t +  \sigma\rmd\bfB_t\eqsp, \bfX_0^{u} \sim \mu \eqsp,
\end{align}
where $\adminControlT(\nu)= \{u \in \adminControlT: \bfX_T^{u} \sim \nu \}$ is the subset of admissible control functions $u\in \adminControlT$ (see details in \Cref{tsbm:app:proofs}) enforcing the terminal marginal constraint.
Then, the optimal path measure $\Pbb^\star$ is induced by $\bfX^{u^\star}$, where $u^\star$ solves \eqref{tsbm:eq:standard_soc} \citep[Proposition 2.10]{leonard2014survey}.
In the limit $\sigma \to 0$, the entropy regularization in EOT vanishes, yielding classical quadratic OT \citep{peyre2019computational}. Correspondingly, \eqref{tsbm:eq:standard_soc} converges to the Benamou--Brenier formulation \citep{benamou2000computational}, its dynamical counterpart. We therefore refer to $\sigma$ as the (entropic) regularization hyperparameter.

\paragraph{Twisted instance : continuous.}Let $V\in \rmC^{0,1}([0,T]\times \rset^d, \rset)$ be an arbitrary time-dependent potential. We study the instance of \eqref{tsbm:eq:classic_sb} with reference measure $\Qbb^{\text{ref}}=\QbbVsg$, where $\QbbVsg$ is absolutely continuous with respect to $\Qbbsg$, with Radon--Nikodym derivative
\begin{align}\label{tsbm:eq:q_v}
    \textstyle \rmd\QbbVsg( x_{[0:T]})=\frac{1}{Z_V} \exp\left(-\int_{0}^T \frac{V_t(x_t)}{\sigma^2} \rmd t\right) \rmd\Qbbsg( x_{[0:T]}) \eqsp ,
\end{align}
with normalizing constant $Z_V=\int \exp\left(-\int_{0}^T \frac{V_t(x_t)}{\sigma^2} \rmd t\right) \rmd\Qbbsg( x_{[0:T]})$. Such transforms of Brownian motion are referred to under various names in the probability literature, including Feynman--Kac penalization \citep{roynette2009penalising}, general $h$-transform \citep{leonard2011stochastic}, exponential tilting \citep{chetrite2015nonequilibrium}, and exponential reweighting \citep{nusken2021solving}. In this work, we say that $\QbbVsg$ is a \emph{twisted} version of $\Qbbsg$, and we adopt the same terminology for the associated bridges, \ie, $\QbbVsg_{|0,T}$ is a twisted version of $\Qbbsg_{|0,T}$. We refer to the resulting SB problem as the Twisted SB (TSB) problem. This form of path distribution naturally appears in filtering theory as the posterior distribution of a signal given some continuous-time observations \citep{del2001stability,fujisaki1972stochastic,kallianpur1968estimation}.

Assuming  $Z_V<\infty$, one can show that $\QbbVsg$ is associated with a Stochastic Differential Equation (SDE) that can be written as an affine-controlled Brownian motion, although the corresponding drift is generally intractable for arbitrary $V$, see \Cref{tsbm:app:twisted}. Therefore, using Girsanov's formula, solving the TSB problem is equivalent to solving the affine SOC problem introduced by \cite{liu2023generalized}, namely
\begin{align}
  \label{tsbm:eq:extended_soc}
  &  \textstyle\argmin \ensembleLigne{\int_{0}^T \PE[\frac{1}{2}\|u_t(\bfX^u_t)\|^2 + V_t(\bfX^u_t)] \rmd t }{u\in \adminControlT(\nu) \eqsp , \, (\bfX_t^u)_{t\in\ccint{0,T}} \text{ solves \eqref{eq:contr_diff} }} \eqsp.
\end{align}
As in the canonical setting, the optimal TSB path measure $\Pbb^\star$ is induced by $\bfX^{u^\star}$, where $u^\star$ solves \eqref{tsbm:eq:extended_soc}. From a physical standpoint, \eqref{tsbm:eq:extended_soc} can be interpreted as a trajectory inference problem with marginal constraints at endpoints, where $V$ represents a penalty to be minimized along the dynamics in addition to the quadratic control cost. In particular, the canonical setting is recovered when $V\equiv 0$, and the formulation can be solved in closed form when $V$ is quadratic \citep{teter2024schr}.

\paragraph{Twisted instance : discrete.}
In this paper, we also consider discrete-time state costs. Based on $K$ potentials $\{V_k\}_{k=1}^K$, each in $ \rmC^{1}(\rset^d, \rset)$, we may define $\QbbVsg$ as
\begin{align}\label{tsbm:eq:q_v_discret}
    \textstyle \rmd\QbbVsg( x_{[0:T]})=\frac{1}{Z_V} \exp\left(-\frac{1}{\sigma^2}\sum_{k=1}^K V_k(x_{t_k})\right) \rmd\Qbbsg( x_{[0:T]}) \eqsp ,
\end{align}
where $Z_V= \int \exp\left(-\frac{1}{\sigma^2}\sum_{k=1}^K V_k(x_{t_k})\right) \rmd\Qbbsg( x_{[0:T]})$,  with the assumption $Z_V< \infty$. As in the continuous-time case, this class of path measures is well motivated by filtering theory \cite{takeuchi1981nonlinear,del2001monte}. Analogously, the optimal TSB path measure $\Pbb^\star$ solution of \eqref{tsbm:eq:classic_sb} for $\QQr$ given by \eqref{tsbm:eq:q_v_discret} is induced by $\bfX^{u^\star}$, where $u^\star$ solves the SOC problem with the potentials $\{V_k\}_{k=1}^K$ as intermediate state cost functions, defined by
\begin{align}
  \label{tsbm:eq:extended_soc_discrete}
  \textstyle\argmin \ensembleLigne{\int_{0}^T \PE[\frac{1}{2}\|u_t(\bfX^u_t)\|^2]\rmd t + \sum_{k=1}^K \PE[V_k(\bfX^u_{t_k})] }{u\in \adminControlT(\nu), (\bfX_t^u)_{t\in\ccint{0,T}} \text{ solves \eqref{eq:contr_diff} }}.
\end{align}

In this setting, one may also replace the Brownian path measure by its discrete-time projection, i.e., its finite-dimensional joint distribution, leading to a fully discrete-time reformulation of the SB problem \eqref{tsbm:eq:classic_sb} under twisting penalization. This framework has received limited attention in SB-related works, but was recently studied by \cite{gushchin2024adversarial}, together with a diffusion-model-based implementation, \emph{in the canonical case only}. We extend this viewpoint to twisted reference measures in \Cref{tsbm:app:full_discrete}, where the fully discrete-time framework is deferred due to space constraints.

\paragraph{Iterative Markovian Fitting \& Application to the canonical SB problem.}

Under mild assumptions on $\Qbb^{\text{ref}}$, $\mu$, and $\nu$, \cite{leonard2014survey} established that $\Pbb^\star$ is the unique Markov path measure satisfying $\Pbb^\star = \Pbb^\star_{0,T}\Qbb^{\text{ref}}_{|0,T}$ together with the marginal constraints $\Pbb^\star_0=\mu$ and $\Pbb^\star_T=\nu$. This characterization naturally motivates the Iterative Markovian Fitting (IMF) scheme \citep{shi2023diffusion}, a fixed-point procedure inspired by iterative Bregman projections \citep{bregman1967relaxation}, which we now briefly recall.

Let $\Pbb \in \Pmeasure(\setFunConT)$. The \emph{reciprocal projection} of $\Pbb$ with respect to $\Qbb^{\text{ref}}$, denoted $\projR(\Pbb) \in \Pmeasure(\setFunConT)$, is defined as
\begin{align}
    \argmin \ensembleLigne{\KL(\Pbb \| \Pi)}{\Pi \in \mathcal{R}(\Qbb^{\text{ref}})} \eqsp ,
\end{align}
where $\mathcal{R}(\Qbb^{\text{ref}})$ denotes the reciprocal class of $\Qbb^{\text{ref}}$, \ie, the set of path measures of the form $\Pi = \Pi_{0,T}\Qbb^{\text{ref}}_{|0,T}$. Importantly, this projection admits a closed-form expression via the KL chain rule \citep{Leonard2014some},
\begin{align}\label{tsbm:eq:proj_r}
\projR(\Pbb)= \Pbb_{0,T}\Qbb^{\text{ref}}_{|0,T} \eqsp .
\end{align}

On the other hand, the \emph{Markovian projection} of $\Pbb$, denoted $\projM(\Pbb) \in \Pmeasure(\setFunConT)$, is defined as 
\begin{align}
    \label{tsbm:eq:markov_projection}
    \argmin \ensembleLigne{\KL(\Pbb \| \Mbb)}{\Mbb \in \mathcal{M}_T} \eqsp ,
\end{align}
where $\mathcal{M}_T$ is the class of Markov diffusion path measures, namely those induced by affine-control diffusion processes solving the SDE \eqref{eq:contr_diff} for some admissible control $u$. In the canonical setting, it can be shown 
(see e.g., \cite{silveri2025exponential}) that, for any 
$\Pi \in \mathcal{R}(\Qbbsg)$, the Markovian projection $\projM(\Pi)$ preserves the marginals of $\Pi$ and is associated with the forward dynamics
$\rmd \bfXf_t = \mathrm{b}^{\mathrm{f}}_t(\bfXf_t)\rmd t + \sigma \rmd \bfB_t$,
with initial condition $\bfXf_0\sim\mu$, where
$\mathrm{b}^{\mathrm{f}}_t(x_t)
= \sigma^2
\PE_{\Pi_{T|t}}
[
\nabla_{\bfX_t}\log \QQsg_{T|t}(\bfX_T|\bfX_t)
\mid \bfX_t=x_t
]$.
Moreover, $\projM(\Pi)$ is the time reversal of the path measure induced by
$\rmd \bfXb_t = \mathrm{b}^{\mathrm{b}}_{T-t}(\bfXb_t)\rmd t + \sigma \rmd {\bfB}_t$,
with $\bfXb_0\sim\nu$, where $\mathrm{b}^{\mathrm{b}}_t(x_t)
= \sigma^2
\PE_{\Pi_{0|t}}
[
\nabla_{\bfX_t}\log \QQsg_{t|0}(\bfX_t|\bfX_0)
\mid \bfX_t=x_t
]$. Since the transition kernels of $\QQsg$ are available in closed form, the drifts can be further simplified as 
\begin{align}\label{tsbm:eq:drifts_canonical}
    \mathrm{b}^{\mathrm{f}}_t(x_t)= 
(\PE_{\Pi_{T|t}}[\bfX_T
\mid \bfX_t=x_t]-x_t)/(T-t) \eqsp, \eqsp \mathrm{b}^{\mathrm{b}}_t(x_t)= 
(\PE_{\Pi_{0|t}}[\bfX_0
\mid \bfX_t=x_t]-x_t)/t \eqsp .
\end{align}

The ideal IMF scheme alternates between these two KL projections, thereby enforcing both the Markov and reciprocal constraints satisfied by $\Pbb^\star$. Starting from an arbitrary $\Pbb^{-1} \in \Pmeasure(\setFunConT)$ with correct marginals $\Pbb^{-1}_0=\mu$ and $\Pbb^{-1}_T=\nu$, one defines recursively for $n\in \nset$
\begin{align}
    \Pbb^{2n} = \projR(\Pbb^{2n-1}) \eqsp , \quad \Pbb^{2n+1} = \projM(\Pbb^{2n}) \eqsp .
\end{align}
By construction, the marginal constraints are preserved at each iteration. Recently, \cite{silveri2025exponential} provided evidence of exponential convergence of IMF sequence $(\Pbb^{n})_{n\in \nset}$ toward $\Pbb^\star$ under some assumptions.

The main computational bottleneck in implementing IMF in the standard case $\QQr = \QQsg$, is the Markovian projection due to the fact that the controls of the associated path measures are intractable.
To address this problem, \cite{shi2023diffusion,peluchetti2023diffusion} independently cast IMF as a machine-learning procedure for approximating these controls; following \cite{shi2023diffusion}, we refer to this approach as DSBM. Given a path measure $\Pi \in \mathcal{R}(\Qbbsg)$, DSBM exploits the conditional-expectation form of the forward and backward controls $\mathrm{b}^{\mathrm{f}}$ and $\mathrm{b}^{\mathrm{b}}$ that define $\projM(\Pi)$, see \eqref{tsbm:eq:drifts_canonical}, to estimate them through the regression problems
\begin{align}\label{tsbm:eq:markov_proj_DSBM}
&    \argmin\ensembleLigne{\textstyle\int_0^T \PE_{\Pi_{t,T}}[\| (\bfX_T -\bfX_t)/(T-t)- v_t^{\theta^\rmf}(\bfX_t)\|^2]\rmd t}{\theta^\rmf \in \Theta} \eqsp , \\
\label{tsbm:eq:markov_proj_DSBM_b}
 &   \argmin\ensembleLigne{\textstyle\int_0^T \PE_{\Pi_{0,t}}[\|(\bfX_0 -\bfX_t)/t- v_t^{\theta^\rmb}(\bfX_t)\|^2]\rmd t}{\theta^\rmb \in \Theta} \eqsp,
\end{align}
where $\{(v_t^{\theta^\rmf})_{t\in\ccint{0,T}} \,: \, \theta^\rmf\}$  and $\{(v_t^{\theta^\rmb})_{t\in\ccint{0,T}} \,: \, \theta^\rmb\}$ are two families of parameterized drift functions. 

The DSBM objectives \eqref{tsbm:eq:markov_proj_DSBM} and \eqref{tsbm:eq:markov_proj_DSBM_b} are fully tractable. Indeed, since $\Pi \in \mathcal{R}(\Qbbsg)$, samples from $\Pi_{t,T}$ or $\Pi_{0,t}$ can be generated from samples of $\Pi_{0,T}$ through Brownian bridges.
After fitting parameters $\theta^\rmf_\star$ and $\theta^\rmb_\star$, the Markovian projection of $\Pi$ is approximated either by the path measure induced by
$\rmd \bfXfs_t= v^{\theta^\rmf_\star}_t(\bfXfs_t)\rmd t + \sigma \rmd \bfB_t$, with $\bfXfs_0 \sim \mu$, or, in the backward direction, by the time reversal of the path measure induced by
$\rmd \bfXbs_t= v^{\theta^{\rmb}_\star}_{T-t}(\bfXbs_t)\rmd t + \sigma \rmd {\bfB}_t$, with $\bfXbs_0 \sim \nu$.
Thus, in the canonical setting, IMF reduces to an iterative bridge-matching procedure with Brownian bridge as stochastic interpolant \citep{albergo2025stochastic}. In practice, to mitigate the bias caused by imperfect marginal matching, \cite{shi2023diffusion} alternate between the forward loss \eqref{tsbm:eq:markov_proj_DSBM} and the backward loss \eqref{tsbm:eq:markov_proj_DSBM_b} when performing Markovian projection steps. This requires two neural networks, one for each drift direction.


\section{Twisted Schrödinger Bridge Matching} 
\label{tsbm:sec:continuous}

A natural approach to solving the SB problem \eqref{tsbm:eq:classic_sb} with $\QQr=\QQVsg$---equivalently, \eqref{tsbm:eq:extended_soc} in the continuous-time case or \eqref{tsbm:eq:extended_soc_discrete} in the discrete-time case---is to derive the IMF scheme associated with the TSB formulation, thereby extending the algorithm originally introduced for the canonical SB problem. We therefore describe how the reciprocal and Markovian projections can be implemented in this setting. This practical realization of IMF is what we call \emph{Twisted Schrödinger Bridge Matching} (TSBM); its updates differ from those of GSBM, as discussed in \Cref{sec:related}. The proofs of the propositions stated in this section are deferred to \Cref{tsbm:app:proofs}. Pseudo-codes are given in \Cref{tsbm:app:pseudo}.

\subsection{Computation of the reciprocal projection}\label{tsbm:subsec:reciprocal_cont}
Let $\Pbb \in \mathcal{M}_T$, and assume that we have access to samples from $\Pbb_{0,T}$\footnote{In practice, this can be achieved by simulating the SDE associated with $\Pbb$ in either the forward or backward direction, as in standard bridge matching approaches.}. In principle, computing the reciprocal projection of $\Pbb$ presents no difficulty, since it is explicit and given by \eqref{tsbm:eq:proj_r}. However, a key difference arises compared to the canonical setting. Although we show in \Cref{tsbm:app:twisted} that the twisted bridge $\QbbVsg_{|0,T}$ is Markov and corresponds to a diffusion process, its drift is not tractable in general. This prevents efficient sampling from $\QbbVsg_{|0,T}$, in contrast to the Brownian bridge $\Qbbsg_{|0,T}$. To address this limitation, we approximate the twisted bridge $\QbbVsg_{|0,T}$ by solving a reverse-KL problem 
\begin{align} \label{tsbm:eq:variational_cont}
    \textstyle\argmin_{\psi \in \Psi}
    \PE_{\Pbb_{0,T}}
    [
    \KL(\Qbb^{\psi}_{|0,T}\|\QbbVsg_{|0,T})
    ] \eqsp .
\end{align}

Here, $\{\Qbb^{\psi}_{|0,T}\}_{\psi \in \Psi}$ denotes a parametric family of Markov bridges obtained from affine-controlled Brownian motions, with drift parameterized by $\psi$. This choice mirrors the structure of $\QbbVsg_{|0,T}$, which itself admits an affine-controlled Brownian representation; see \Cref{tsbm:app:twisted} for details. Crucially, the resulting variational problem admits a tractable objective that can be implemented in practice. We state the continuous-time result in the forward direction below, for $\QbbVsg$ defined by \eqref{tsbm:eq:q_v}, and defer the backward counterpart to \Cref{tsbm:app:reciprocal}.
\vspace{0.2cm}

\begin{proposition}[Forward reciprocal projection -- continuous time]\label{tsbm:prop:reciprical_cont}
Assume that the variational family $\{\Qbb^{\psi}_{|0,T}\}_{\psi \in \Psi}$ in \eqref{tsbm:eq:variational_cont} describes the pinned diffusion processes associated, for almost surely any $(x_0,x_T)\sim \Pbb_{0,T}$, to forward-time parametric SDEs of the form
\begin{align*}
\rmd \bfX^{0,T}_t &=  v_t^{\rmf,\psi} (\bfX^{0,T}_t|x_0, x_T) \rmd t + \sigma \rmd \bfB_t\eqsp , \eqsp \bfX^{0,T}_0=x_0, \eqsp \bfX^{0,T}_T=x_T \eqsp ,
\end{align*}
with $v^{\rmf,\psi}\in \adminControlT$. Define $\Pi^{\psi}= \Pbb_{0,T}\Qbb^{\psi}_{|0,T}$. Then, the minimizer of \eqref{tsbm:eq:variational_cont}, denoted by $\psi_\star$, is also the minimizer of the following loss
\begin{align}
  \mathcal{L}_{\mathcal{R}}^{\rmf}(\psi) &= \textstyle\int_{0}^T\PE_{\Pi^{\psi}_{0,t,T}}\left[ \frac{1}{2}\|(\bfX_T -\bfX_t)/(T-t)-v_t^{\rmf,\psi}(\bfX_t|\bfX_0, \bfX_T)\|^2 + V_t(\bfX_t)\right]\rmd t \eqsp \label{tsbm:eq:objective_reciprocal_f} .
\end{align}
\end{proposition}

In particular, \eqref{tsbm:eq:objective_reciprocal_f} can be interpreted as a conditional SOC problem.
The variational learning problem is unchanged with discrete-time potential, leading to the following forward loss (the backward loss being also given in \Cref{tsbm:app:reciprocal})
\begin{align}
  \mathcal{L}_{\mathcal{R}}^{\rmf}(\psi) &= \textstyle\int_{0}^T\PE_{\Pi^{\psi}_{0,t,T}}\left[ \frac{1}{2}\|(\bfX_T -\bfX_t)/(T-t)-v_t^{\rmf,\psi}(\bfX_t|\bfX_0, \bfX_T)\|^2\right]\rmd t + \sum_{k=1}^K \PE_{\Pi^{\psi}_{t_k}}[V_k(\bfX_{t_k})] .\label{tsbm:eq:objective_reciprocal_f_d}
\end{align}

\vspace{-0.3cm}
\paragraph{Practical implementation.}
In practice, the variational bridge $\Qbb^\psi_{|0,T}$ should meet three requirements: (1) it should be easy to sample from its marginals, (2) it should allow reparameterized sampling for efficient backpropagation, and (3) it should admit tractable expressions for the forward drift $v^{\rmf,\psi}$. A natural choice satisfying these properties is a stochastic interpolant \citep{albergo2025stochastic}. In this case, $\Qbb^\psi_{t|0,T}$ is Gaussian, which enables efficient sampling and reparameterization, while the associated forward and backward drifts are available in closed form. To increase expressivity, one may also use mixtures of stochastic interpolants, as proposed by \cite{du2024doob}. To reduce the cost of this variational stage, we follow the spirit of \cite{liu2023generalized} and use a spline parameterization of the interpolants rather than neural networks. This allows us to reuse stored control-point trajectories when sampling from the coupling $\Pi_{0,T}$ while enjoying lightweight optimization. For more details, we refer the reader to \Cref{tsbm:app:var_reciprocal}.

\subsection{Computation of the Markovian projection}
\Cref{tsbm:prop:tsb_markovian_loss} characterizes the Markovian projection step of the IMF scheme for the continuous-time TSB problem in both the forward and backward directions. In particular, it shows that the corresponding drift terms explicitly depend on the state-cost gradient $\nabla V$.

\begin{proposition}
  \label{tsbm:prop:tsb_markovian_loss}
  For any $\Pi \in \mathcal{R}(\QbbVsg)$, the path measure solution of \eqref{tsbm:eq:markov_projection} for $\PP = \Pi$, is given by the distribution associated to the forward dynamics
  $\rmd \bfXf_t=  \mathrm{b}^{\mathrm{f},V}_t(\bfXf_t) \rmd t  + \sigma \rmd \bfB_t, \bfXf_0\sim \mu$ and, equivalently, to the backward dynamics $\rmd \bfXb_t= \mathrm{b}^{\mathrm{b},V}_{T-t}(\bfXb_t) \rmd t  + \sigma \rmd \bfB_t, \bfXb_0\sim \nu$, where
  \vspace{-0.1cm}
  \begin{align}
    \label{tsbm:eq:prop:tsb_markovian_loss_forward_1}
 &\textstyle\mathrm{b}^{\mathrm{f},V}_t(x_t) = \PE_{\Pi_{|t}}[ (\bfX_T-\bfX_t)/(T-t) -(T-t)^{-1}\int_t^T(T-s)\nabla V_{s}(\bfX_s)|\bfX_t=x_t] \eqsp,\\
        \label{tsbm:eq:prop:tsb_markovian_loss_forward_2}
 &\textstyle  \mathrm{b}^{\mathrm{b},V}_t(x_t) = \PE_{\Pi_{|t}}[(\bfX_0-\bfX_t)/t -t^{-1}\int_0^ts\nabla V_{s}(\bfX_s)|\bfX_t=x_t] \eqsp.
  \end{align}
 Moreover, for any $t\in [0,T]$, we have $\mathrm{Law}(\bfXf_t)=\mathrm{Law}(\bfXb_{T-t})=\Pi_t$. 
\end{proposition}

Informally, the forward and backward optimal controls can be viewed as state-cost perturbations of the canonical controls in \eqref{tsbm:eq:drifts_canonical}, which are recovered when the potential vanishes. These perturbations are driven by $-\nabla V$ evaluated along the corresponding continuation of the trajectory: over future times $s>t$ in the forward direction, and over past times $s<t$ in the backward direction. Note that the optimal controls in \eqref{tsbm:eq:prop:tsb_markovian_loss_forward_1} and \eqref{tsbm:eq:prop:tsb_markovian_loss_forward_2} involve state-cost integrals, which would be costly to estimate directly. However, we show in \Cref{tsbm:app:markov_continu} that they admit tractable regression characterizations. This motivates the following Markovian projection losses for any $\Pi \in \mathcal{R}(\QbbVsg)$, defined for two parameterized control families $\{(v^{\theta^{\rmf}})_{t\in [0,T]}:\theta^{\rmf}\in \Theta\}$ and $\{(v^{\theta^{\rmb}})_{t\in [0,T]}:\theta^{\rmb}\in \Theta\}$:
\vspace{-0.1cm}
\begin{align}
     \textstyle\mathcal{L}_{\mathcal{M}_T}^{\rmf}(\theta^{\rmf}) &=\int_0^T (T-t)^{-1}\int_t^T \PE_{\Pi_{t,s,T}}[\|\mathbf{b}^{\mathrm{f}, V}_{t,s}(\bfX_t,\bfX_s,\bfX_T)  -v_t^{\theta^{\rmf}}(\bfX_t)\|^2] \eqsp \rmd s\rmd t \eqsp , \label{tsb:eq:obj_f}\\[-0.1em]
     & \qquad\text{with } \mathbf{b}^{\mathrm{f},V}_{t,s}(x_t,x_s,x_T) = (x_T-x_t)/(T-t)-(T-s)\nabla V_{s}(x_s) \eqsp ,\notag\\[0.5em]
     \textstyle\mathcal{L}_{\mathcal{M}_T}^{\rmb}(\theta^{\rmb})&=\int_0^T t^{-1}\int_0^t \PE_{\Pi_{0,s,t}}[\|\mathbf{b}^{\mathrm{b},V}_{t,s}(\bfX_0,\bfX_s,\bfX_t)  -v_t^{\theta^{\rmb}}(\bfX_t)\|^2] \eqsp \rmd s\rmd t \eqsp ,\label{tsb:eq:obj_b}\\[-0.1em]
      & \qquad\text{with } \mathbf{b}^{\mathrm{b},V}_{t,s}(x_0,x_s,x_t) = (x_0-x_t)/t-s\nabla V_{s}(x_s)\notag \eqsp .
\end{align}
Notably, setting $V\equiv 0$ in the TSBM bridge-matching objectives \eqref{tsb:eq:obj_f} and \eqref{tsb:eq:obj_b} \emph{exactly} recovers the DSBM objectives \eqref{tsbm:eq:markov_proj_DSBM} and \eqref{tsbm:eq:markov_proj_DSBM_b}. Since our objectives do not depend explicitly on $\sigma$, their regression targets remain unchanged in the deterministic limit $\sigma=0$, thereby yielding in principle a twisted analogue of Rectified Flow \citep{liu2022rectified}.

\paragraph{Variance reduction via learnable control variates.}
In practice, the regression targets $\mathbf{b}^{\mathrm{f},V}$ and $\mathbf{b}^{\mathrm{b},V}$ used in the training objectives \eqref{tsb:eq:obj_f} and \eqref{tsb:eq:obj_b} can exhibit high variance, due to noisy evaluations of $\nabla V(X_s)$ along the intermediate states $(\bfX_s)_{s\in [t,T]}$ in the forward case and $(\bfX_s)_{s\in [0,t]}$ in the backward case. To address this issue, we introduce \emph{control variates} into these regression targets, reducing their variance without introducing bias. This construction relies on the following result.
\begin{proposition}\label{tsbm:prop:cv_main}
  For any $\Pi \in \mathcal{R}(\QbbVsg)$ and $\alpha : \ccint{0,T}^2 \to \rset$, $\rmC^1$ in its second variable, it holds that $\PE_{\Pi_{|t}}[H^{\alpha,\rmf}_{t}|\bfX_t] = 0$ and $\PE_{\Pi_{|t}}[H^{\alpha,\rmb}_{t}|\bfX_t] = 0$ where
  \vspace{-0.1cm}
  \begin{align*}
    H^{\alpha,\rmf}_{t} &= \alpha(t,t)[\bfX_T-\bfX_t] + \int_{t}^T\defEns{ \frac{\partial_s \tilde{\alpha}^{\rmf}(t,s)}{T-s}[\bfX_T-\bfX_s]  +(T-s) \nabla V_s(\bfX_s) \int_{t}^s \frac{\alpha(t,u)}{T-u} \rmd u } \rmd s  \\
H^{\alpha,\rmb}_{t} & = \alpha(t,t)[\bfX_0-\bfX_t]- \int_{0}^t\defEns{\frac{\partial_s \tilde{\alpha}^{\rmb}(t,s)}{s}[\bfX_0-\bfX_s]  +  s \nabla V_s(\bfX_s) \int_{s}^t \frac{\alpha(t,u)}{u}\rmd u} \rmd s\eqsp,
  \end{align*}
  where $\tilde{\alpha}^{\rmf}(t,s)=(T-s)\alpha(t,s)$ and $\tilde{\alpha}^{\rmb}(t,s)=s\alpha(t,s)$.
\end{proposition}

We can therefore construct \emph{strictly equivalent} representations of the optimal drifts
\eqref{tsbm:eq:prop:tsb_markovian_loss_forward_1}--\eqref{tsbm:eq:prop:tsb_markovian_loss_forward_2}
from \Cref{tsbm:prop:tsb_markovian_loss}, given for any functions $\alpha^\rmf$ and $\alpha^\rmb$ satisfying the assumptions of \Cref{tsbm:prop:cv_main} by 
\begin{align}\notag
    \mathrm{b}^{\alpha^\rmf,V}_t(x_t)
    =
    \mathrm{b}^{\rmf,V}_t(x_t)
    -
    \PE_{\Pi_{|t}}
    [
    H^{\alpha^\rmf,\rmf}_{t}
    \mid \bfX_t=x_t
    ] \eqsp, \eqsp
    \mathrm{b}^{\alpha^\rmb,V}_t(x_t)
    =
    \mathrm{b}^{\rmb,V}_t(x_t)
    -
    \PE_{\Pi_{|t}}
    [
    H^{\alpha^\rmb,\rmb}_{t}
    \mid \bfX_t=x_t
    ].
\end{align}
These $\alpha$-dependent drifts still admit conditional-expectation representations under $\Pi_{|t}$, again involving both the dynamical states and the gradient of the state cost. They induce alternative Markovian projection losses,
$\mathcal{L}_{\mathcal{M}_T}^{\rmf}(\theta^{\rmf},\alpha^\rmf)$ and
$\mathcal{L}_{\mathcal{M}_T}^{\rmb}(\theta^{\rmb},\alpha^\rmb)$,
which are strictly equivalent to \eqref{tsb:eq:obj_f} and \eqref{tsb:eq:obj_b}: although the objectives differ, they share the same optimal drift. In particular, setting $\alpha^\rmf=\alpha^\rmb=0$, \ie, using no control variate, exactly recovers the base objectives \eqref{tsb:eq:obj_f} and \eqref{tsb:eq:obj_b}. 

We present the forward objective below (the backward objective is given in \Cref{tsbm:app:markov_continu}):
\begin{align}
     \mathcal{L}_{\mathcal{M}_T}^{\rmf}(\theta^{\rmf}, \alpha^\rmf) &=\int_0^T (T-t)^{-1}\int_t^T \PE_{\Pi_{t,s,T}}[\| \mathbf{b}^{\alpha^\mathrm{f},V}_{t,s}(\bfX_t,\bfX_s,\bfX_T) -v_t^{\theta^{\rmf}}(\bfX_t)\|^2] \eqsp \rmd s\rmd t \notag \eqsp ,\\
     \textstyle \text{with } \mathbf{b}^{\alpha^\mathrm{f},V}_{t,s}(x_t,x_s,x_T)&\textstyle = (x_T-x_t)/(T-t)-\alpha^{\rmf}(t,t)[x_T-x_t] -\frac{\partial_s \tilde{\alpha}^{\rmf}(t,s)}{T-s}(T-t)[x_T-x_s]\notag\\
      & \textstyle \quad-\left\{1+(T-t)\int_{t}^s \frac{\alpha^{\rmf}(t,u)}{T-u} \rmd u\right\}(T-s) \nabla V_s(x_s)\notag \eqsp .
\end{align}

Importantly, these $\alpha$-based objectives admit \emph{a bias--variance decomposition} into the original TSBM loss, which depends only on $\theta$, and an additional term depending only on $\alpha$, corresponding exactly to the variance of the $\alpha$-based regression target; see \Cref{tsbm:app:markov_continu} for details. Motivated by this decomposition, we parameterize $\alpha^\rmf$ and $\alpha^\rmb$ with lightweight neural networks and learn them jointly with the drifts $v^{\theta^{\rmf}}$ and $v^{\theta^{\rmb}}$, by minimizing the objectives with respect to $(\theta^{\rmf},\alpha^\rmf)$ or $(\theta^{\rmb},\alpha^\rmb)$. This preserves the optimal drift while stabilizing optimization. The exact parameterization is detailed in \Cref{tsbm:app:cv_param}. A related learnable control-variate strategy was recently proposed by \cite{blessing2026bridge}, but only for terminal potentials $V_T$. Our result is more general: it defines a bi-level control variate for arbitrary pairs $(t,s)$, and may therefore be useful beyond the Twisted SB setting considered here.

\paragraph{Discrete-time instance.}
For the discrete-time setting, we derive an analogue of \Cref{tsbm:prop:tsb_markovian_loss,tsbm:prop:cv_main} and the corresponding tractable losses for approximating the intractable controls; we defer the details to \Cref{tsbm:app:markovian_proofs_d}. Remarkably, we show in \Cref{tsbm:app:full_discrete} that, in a fully discrete-time formulation of the SB problem, \emph{the resulting Markovian projection loss is a state-cost perturbation of the DDPM loss} \citep{Ho2020denoising}, paralleling the connection with DSBM in continuous time.

\paragraph{Practical implementation.}
Given $\Pi\in \mathcal{R}(\QQVsg)$, the joint laws $\Pi_{t,s,T}$ and $\Pi_{0,s,t}$ appearing in the expectations of the objectives cannot be sampled directly, unlike in DSBM. This is precisely because the twisted bridge $\QQVsg_{s,t|0,T}$ is itself not directly simulable, as discussed in \Cref{tsbm:subsec:reciprocal_cont}. For implementation, we therefore replace $\Pi$ with its variational approximation $\Pi^{\psi^\star}=\Pi_{0,T}\Qbb^{\psi^\star}_{|0,T}$, where $\psi^\star$ is obtained at the preceding reciprocal projection step by solving \eqref{tsbm:eq:variational_cont} with coupling $\Pi_{0,T}$ using parametric stochastic interpolants. A naive implementation would sample from $\Qbb^{\psi^\star}_{s,t|0,T}$ by integrating the corresponding pinned SDEs, depending on the direction of the Markovian projection, from time $t$ to $s$, hence introducing a non-negligible overhead compared with the original DSBM implementation. We avoid this bottleneck by exploiting the spline parameterization, which allows these SDE integrations to be computed almost instantaneously; see \Cref{tsbm:app:var_reciprocal} for details.


\section{Related work}\label{sec:related}

In this section, we provide a detailed comparison between TSBM and Generalized Schrödinger Bridge Matching (GSBM), introduced by \cite{liu2023generalized}, which is the closest existing method to ours. We focus on both the underlying framework and the practical implementation for solving the generalized SB problem in the form of \eqref{tsbm:eq:extended_soc}. Other related works, which are less directly connected to TSBM than GSBM, are discussed in \Cref{tsbm:app:related_work}. At a high level, GSBM and TSBM share a similar spirit: both are iterative bridge-matching methods that alternate, in the forward and backward directions, between two steps: approximating the generally intractable twisted bridge, and learning the corresponding Markovian projection. However, unlike TSBM, GSBM does not explicitly establish the connection with the IMF scheme. This leads to different update rules, which we detail below.

\paragraph{Reciprocal projection.}
Similarly to our approach, \cite{liu2023generalized} consider a variational approximation of the twisted bridge based on stochastic interpolants, which is then used in their bridge-matching loss. However, the related loss derived in \citep[Proposition 2]{liu2023generalized} differs from our reciprocal objective \eqref{tsbm:eq:objective_reciprocal_f}. Instead, their objective is obtained by turning the SOC objective \eqref{tsbm:eq:standard_soc} into a conditional version
\begin{align}
  \mathcal{L}_{\mathcal{R}}^f(\psi) 
  &= \textstyle\int_{0}^T
  \PE_{\Pi^{\psi}_{0,t,T}}
  [
  \frac{1}{2}\|v_t^{\rmf,\psi}(\bfX_t|\bfX_0, \bfX_T)\|^2 
  + V_t(\bfX_t)
  ]\rmd t \eqsp ,
  \label{tsbm:eq:objective_reciprocal_f_gsbm}
\end{align}
where $\Pi^{\psi}_{0,t,T}= \Qbb^{\psi}_{t|0,T}\Pbb_{0,T}$. The key difference between \eqref{tsbm:eq:objective_reciprocal_f} and \eqref{tsbm:eq:objective_reciprocal_f_gsbm} is the absence, in the latter, of the Brownian bridge term in the quadratic control cost used as regression target. As a result, the two objectives lead to different variational solutions under the same stochastic interpolant parameterization. Our analysis suggests that this discrepancy may stem from an issue in the derivation of \cite{liu2023generalized}, which is already apparent in the zero-potential case. We empirically support this claim in \Cref{tsbm:sec:xps}, where our objective yields more accurate learning in this simplest setting.

\paragraph{Markovian projection.}
The main discrepancy between TSBM and GSBM lies in the Markovian projection loss. We detail it here in the forward direction. The GSBM bridge-matching loss stated in \cite[Equation 5]{liu2023generalized} is defined as
\begin{align}
    \argmin \ensembleLigne{\textstyle
    \int_0^T 
    \PE_{\Pi^{\psi_\star}_{0,t,T}}
    [
    \|
    v_t^{\rmf,\psi_\star}(\bfX_t|\bfX_0, \bfX_T)
    -v_t^{\theta^{\rmf}}(\bfX_t)
    \|^2
    ]\rmd t}{\theta^{\rmf} \in \Theta} \eqsp ,
    \label{tsbm:eq:gsb_loss_markov_f}  
\end{align}
where $v^{\rmf,\psi_\star}$ solves the variational problem \eqref{tsbm:eq:objective_reciprocal_f_gsbm} and is used to approximate the preceding reciprocal projection step. As in our practical Markovian objective, the expectation is taken with respect to $\Pi^{\psi_\star}$, obtained by combining the available coupling with the variational bridge rather than with the intractable twisted bridge. However, the regression target is fundamentally different. In GSBM, the loss \eqref{tsbm:eq:gsb_loss_markov_f} regresses onto the conditional vector field of the variational interpolant itself. If the variational approximation were exact, this objective would be valid. In the presence of a non-zero variational gap, however, this regression target inherits the corresponding approximation error, leading to an increasing bias as this gap grows. By contrast, our bridge-matching loss \eqref{tsb:eq:obj_f} avoids this source of bias, since its regression target is derived from the \emph{exact} conditional velocity field.

\vspace{-0.1cm}
\section{Numerical experiments}\label{tsbm:sec:xps}
\vspace{-0.1cm}

We consider two complementary trajectory inference tasks within the continuous-time Twisted SB framework\footnote{Toy experiments in the discrete-time setting, absent from the experimental study of \cite{liu2023generalized}, are reported in \Cref{tsbm:app:toy_discrete}.}: (i) \textbf{crowd navigation}, a standard benchmark for generalized SB problems with repulsive state costs \citep{liu2022deep,liu2023generalized}; and (ii) \textbf{single-cell} inference from sparse observations, using an attractive state cost, which is original to our work. We compare TSBM with GSBM, its closest competitor, using the same computational backbone and, whenever possible, the same main hyperparameters. Our experiments cover several dimensions to assess the scaling behavior of both methods. We evaluate performance along two complementary axes: \textbf{optimality}, measured by the SB cost from the SOC formulation \eqref{tsbm:eq:extended_soc}, and \textbf{feasibility}, measured by satisfaction of the terminal marginal constraint.

\textbf{Classical unidirectional procedure and bidirectional extension.}
By default, we run 5 outer iterations of GSBM and TSBM, each consisting of one forward followed by one backward Markovian projection. Following \cite{bortoli2024schrodinger}, we also consider a \emph{bidirectional} extension of both algorithms. Instead of alternating full forward and backward updates from the start, this variant proceeds in two phases.  First, in a pretraining phase, the forward and backward Markovian projection losses, along with their spline parameters, are optimized jointly under the independent coupling $\mu \otimes \nu$. Second, in a fine-tuning phase, forward and backward reciprocal/Markovian projections alternate after fewer training epochs than in the classical procedure, set to 10\% in our experiments, which is expected to yield smoother convergence. For a fair comparison, unidirectional and bidirectional variants are trained for the same total number of epochs. Full experimental details are provided in \Cref{tsbm:app:xps}.
\vspace{-0.1cm}
\subsection{Accuracy of the variational reciprocal projection}

We first assess how the discrepancy between the TSBM and GSBM reciprocal objectives, \eqref{tsbm:eq:objective_reciprocal_f} and \eqref{tsbm:eq:objective_reciprocal_f_gsbm}, affects the approximation of the twisted bridge. Since it is generally intractable, its approximation error cannot usually be measured exactly. We therefore consider the zero-potential case $V\equiv 0$, where the twisted bridge reduces to the tractable Brownian bridge. We set $T=1$, and evaluate both methods across dimension $d\in\{2,10,50,100\}$ and regularization $\sigma\in\{1,2,5\}$. In each setting, the reciprocal projection is performed with respect to $\pi_{0,1}=\mu\otimes \nu$, where $\mu=\densityGaussian(-10\cdot \mathbf{1}_d, 0.2 \Idd)$ and $\nu=\densityGaussian(10\cdot \mathbf{1}_d, 0.2 \Idd)$, with $\mathbf{1}_d$ being the $d$-dimensional vector with all entries equal to one. We use the same spline parameterization  for TSBM and GSBM, and report in \Cref{tsbm:table:reciprocal} the forward KL-based metric $\mathbb{E}_{\pi_{0,1}}[\KL(\QQsg_{t|0,1}\|\QQ^{\psi^\star}_{t|0,1})]$
averaged over $200$ uniformly sampled values of $t\in[0,1]$.

\begin{figure}[h!]
  \centering
  \begin{minipage}{0.65\textwidth}
    \centering
    \vspace{-0.16cm}
    \captionof{table}{Accuracy of Brownian bridge estimation using the TSBM and GSBM reciprocal losses for increasing dimension $d$ and regularization $\sigma$. Results are averaged over $8$ metric evaluations. Bold indicates the best result in each setting.}
    \label{tsbm:table:reciprocal}

    \resizebox{\linewidth}{!}{%
  \begin{tabular}{l|cc|cc|cc}
    \toprule
     Regularization & \multicolumn{2}{c}{$\sigma=1$} & \multicolumn{2}{c}{$\sigma=2$} &\multicolumn{2}{c}{$\sigma=5$} \\
     \cmidrule(lr){1-1} \cmidrule(lr){2-3} \cmidrule(lr){4-5} \cmidrule(lr){6-7}
    Method & GSBM & TSBM     & GSBM & TSBM     & GSBM & TSBM \\
    \midrule
    $d=2$ & $0.09$  &  $\bf0.02$ & $0.05$ & $\bf0.02$  & $0.01$ &   $\bf0.0008$  \\
    $d=10$ & $0.17$  & $\bf0.10$ & $0.11$ & $\bf0.08$  & $0.03$ &    $\bf0.004$   \\
    $d=50$  & $0.58$  & $\bf0.50$  & $0.41$ & $\bf0.38$  & $0.14$ &  $\bf0.02$ \\
    $d=100$ &  $1.09$ & $\bf1.00$ & $\bf0.77$ & $\bf0.77$  & $0.27$ &  $\bf0.04$  \\
    \bottomrule
  \end{tabular}
   }
  \end{minipage}
  \hfill
  \begin{minipage}{0.34\textwidth}
    \centering
    \includegraphics[width=\linewidth]{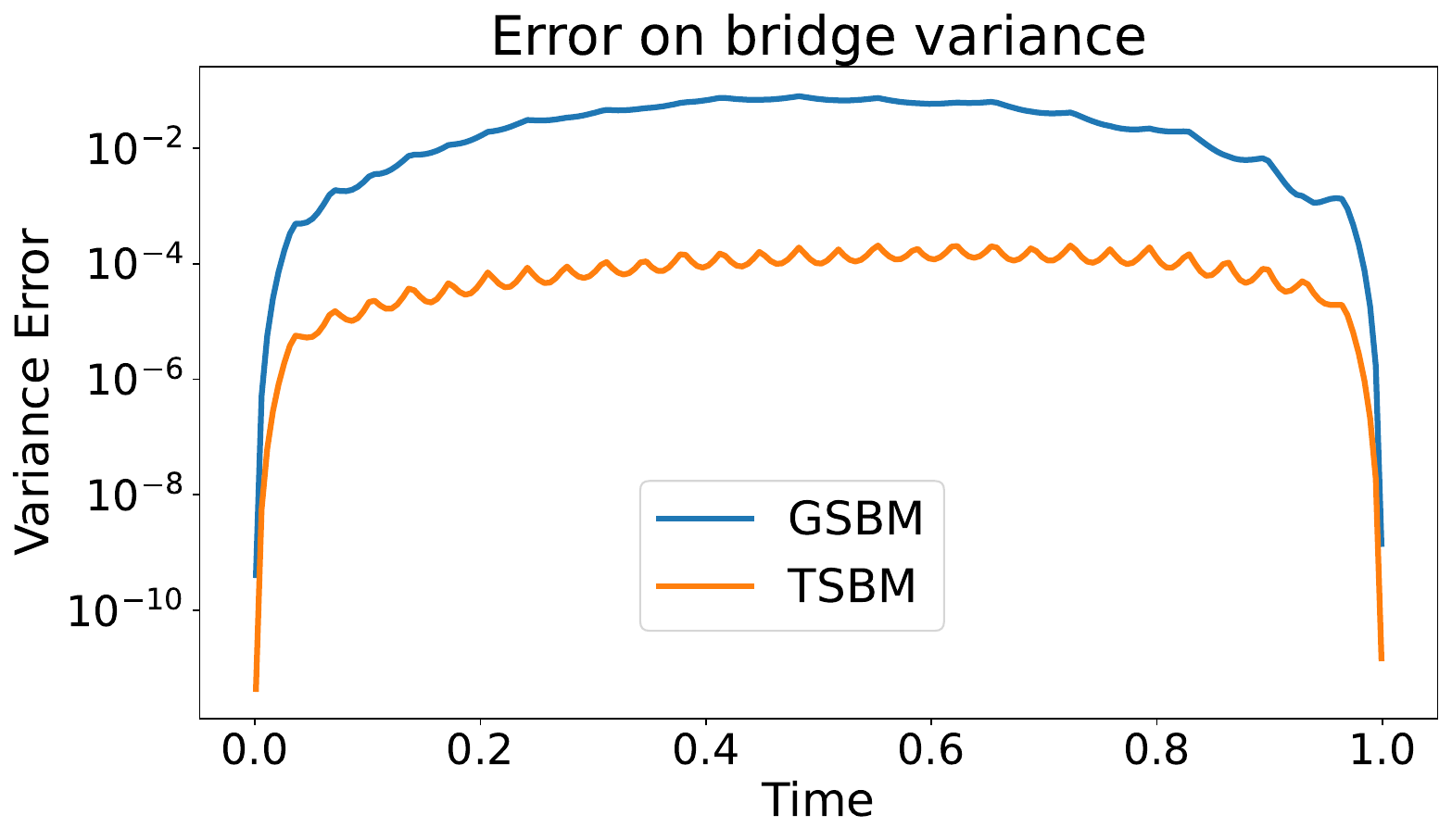}
    \vspace{-0.6cm}
    \caption{GSBM and TSBM learning error of the Brownian bridge variance over time ($d=2$, $\sigma=2$).}
    \label{tsbm:fig:reciprocal}
  \end{minipage}
\vspace{-1.4cm}
\end{figure}

\newpage

\begin{figure}[h!]
\vspace{-0.3cm}
  \centering
\begin{minipage}{0.49\textwidth}
    \centering
    \includegraphics[width=\linewidth]{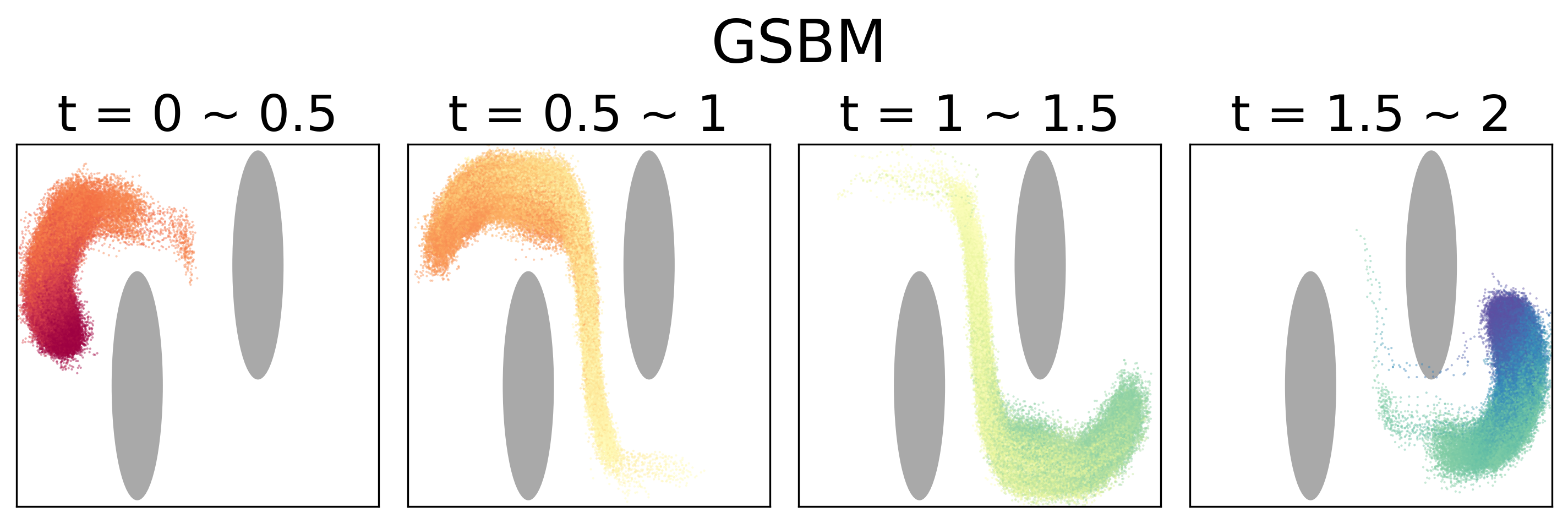}
    \vfill
    \includegraphics[width=\linewidth]{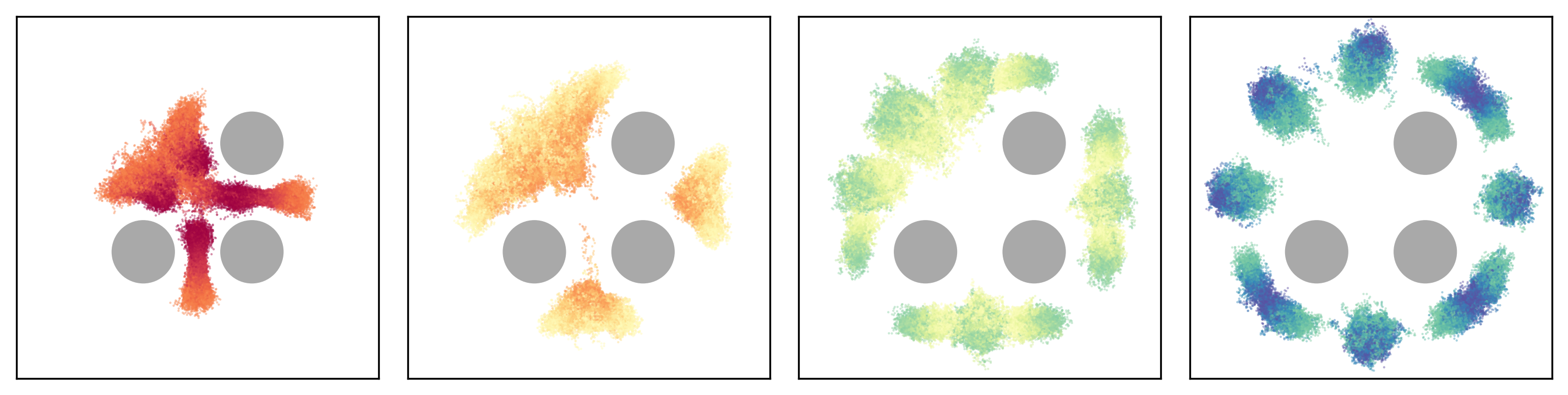}
    \vfill
    \includegraphics[width=\linewidth]{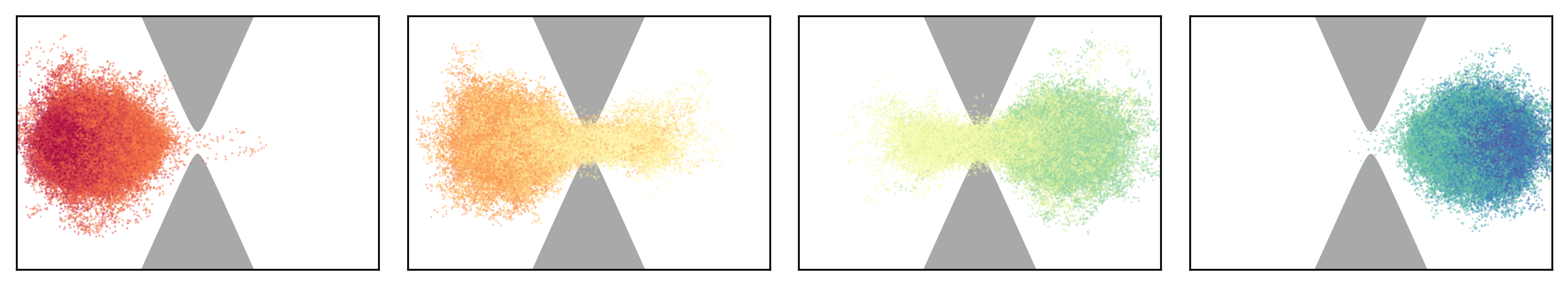}
  \end{minipage}
  \hfill
\begin{minipage}{0.49\textwidth}
    \centering
    \includegraphics[width=\linewidth]{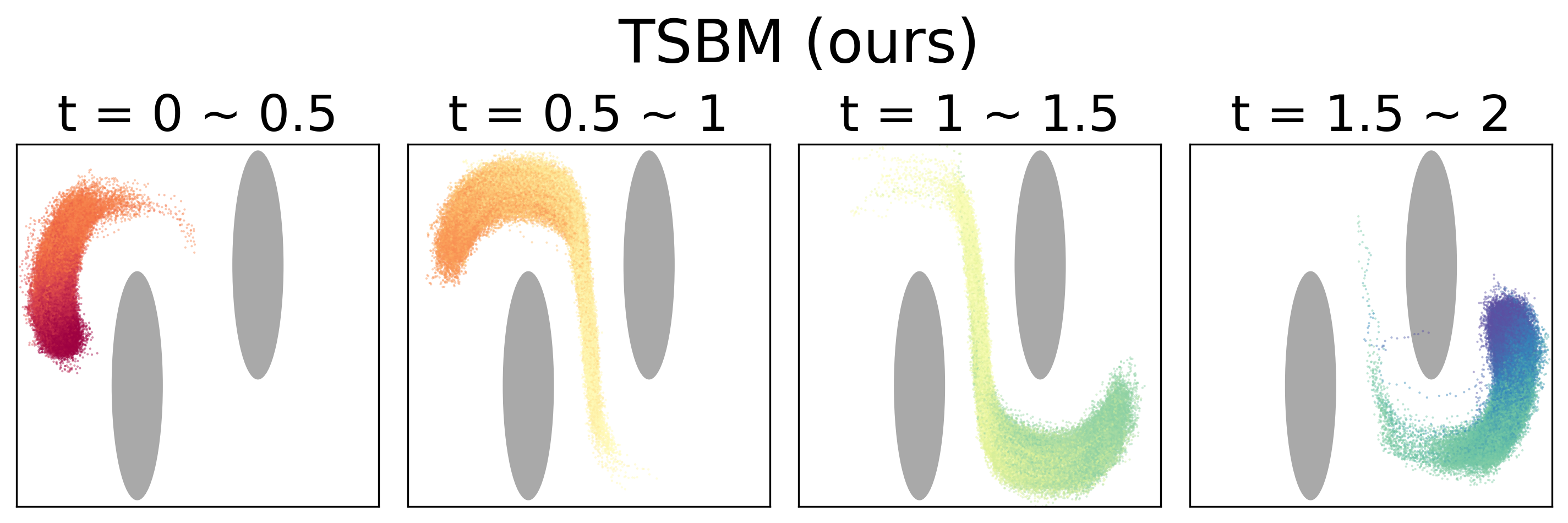}
    \vfill
    \includegraphics[width=\linewidth]{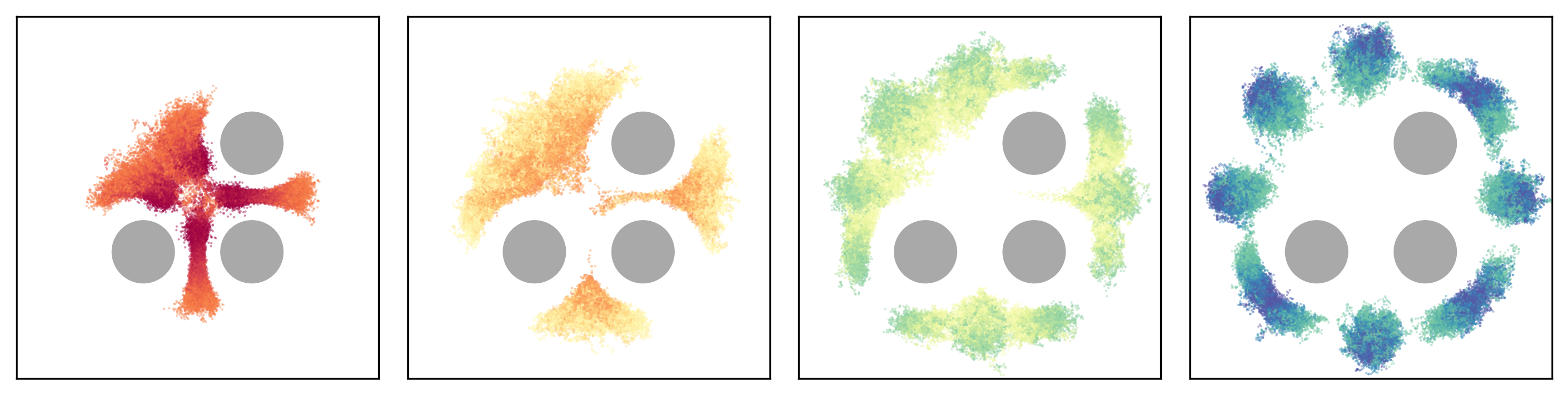}
    \vfill
    \includegraphics[width=\linewidth]{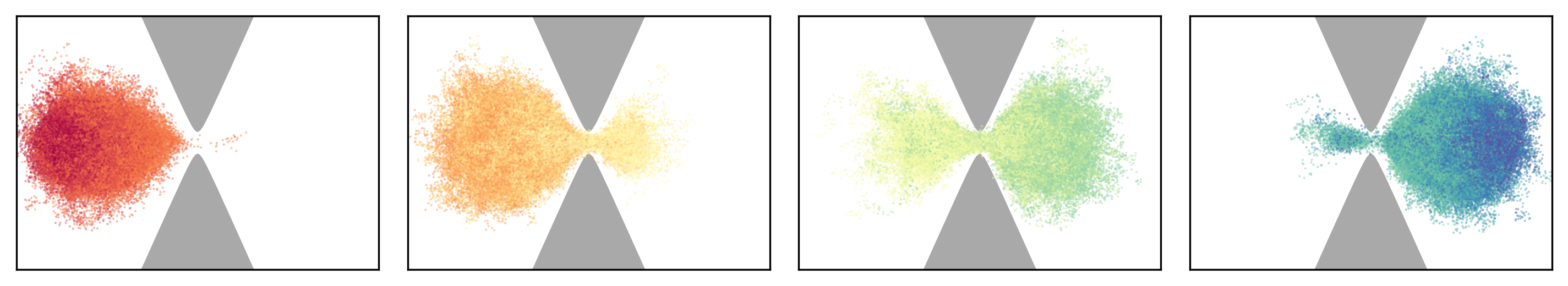}
  \end{minipage}  
  \vspace{-0.2cm}
  \caption{Learned TSB dynamics for Stunnel (top), Gmm (middle), Vneck (bottom) with $d=2$.}
  \label{tsbm:fig:crowd_state_main}
\vspace{-0.2cm}
\end{figure}

TSBM consistently yields a more accurate approximation than GSBM, especially with high $\sigma$, supporting our variational formulation. \Cref{tsbm:fig:reciprocal} further illustrates this performance gap by showing the evolution over time of the error on the learned bridge variance in this simple setting.   
\vspace{-0.1cm}
\subsection{Crowd navigation} 
Following \cite{liu2023generalized}, we validate TSBM on crowd-navigation tasks (`Stunnel', `Gmm', `Vneck'), modeled as Twisted SB problems with SOC formulation \eqref{tsbm:eq:extended_soc}, where the state cost encodes physical constraints; see \Cref{tsbm:app:crowd} for details on these settings. The resulting trajectories are shown in \Cref{tsbm:fig:crowd_state_main} for $d=2$ and $\sigma=2$, where TSBM appears to better capture the obstacle geometry. Complementing these qualitative results, we compare GSBM and TSBM across dimensions in \Cref{tsbm:tab:crowd_main}, again with $\sigma=2$. The results reveal a consistent optimality--feasibility trade-off: GSBM better matches the terminal marginal, whereas TSBM achieves lower state-cost transport. This suggests that using the state-cost gradient enforces the objective constraint more directly, but also introduces additional variance (even with the control-variate strategy), which can hinder satisfaction of the marginal constraint. 

\begin{table}[h!]
\vspace{-0.4cm}
    \caption{Results of crowd navigation experiments for Stunnel and Vneck,  with increasing $d$. Results are averaged over $8$ metric evaluations. Bold indicates the best result in each setting.}
    \label{tsbm:tab:crowd_main}
  \resizebox{\textwidth}{!}{%
  \begin{tabular}{l|cc|cc|cc|cc}
    \toprule
     State cost & \multicolumn{4}{c}{Stunnel} & \multicolumn{4}{c}{Vneck}  \\
     \cmidrule(lr){1-1} \cmidrule(lr){2-5} \cmidrule(lr){6-9}
    Metric & \multicolumn{2}{c}{Feasibility ($\downarrow$)} & \multicolumn{2}{c}{Optimality ($\downarrow$)} & \multicolumn{2}{c}{Feasibility ($\downarrow$)} & \multicolumn{2}{c}{Optimality ($\downarrow$)}    \\
    \cmidrule(lr){1-1} \cmidrule(lr){2-3} \cmidrule(lr){4-5} \cmidrule(lr){6-7} \cmidrule(lr){8-9}
    Method & GSBM & TSBM     & GSBM & TSBM      & GSBM & TSBM & GSBM & TSBM \\
    \midrule
    $d=2$ & $\bf 0.006{\scriptstyle \pm 0.001}$  &  $0.056{\scriptstyle \pm 0.007}$ & $978{\scriptstyle\pm 1.0}$ & $\bf956{\scriptstyle\pm 1.0}$  & $\bf 0.008{\scriptstyle \pm 0.004}$ &   $0.16{\scriptstyle \pm 0.02}$ & $69.2{\scriptstyle\pm 0.5}$ & $\bf53.1{\scriptstyle\pm 0.4}$ \\
    $d=10$ & $\bf1.02 {\scriptstyle\pm 0.01}$  & $1.14{\scriptstyle\pm 0.01}$ & $762 {\scriptstyle\pm 0.9}$ & $\bf748{\scriptstyle\pm 0.8}$  & $\bf 1.02 {\scriptstyle\pm 0.01}$ &    $ 1.50{\scriptstyle\pm 0.03}$ & $254 {\scriptstyle\pm 1.0}$ & $\bf252 {\scriptstyle\pm 0.4}$  \\
    $d=50$  & $17.49 {\scriptstyle\pm0.08}$  & $\bf17.07 {\scriptstyle\pm0.06}$  & $3592 {\scriptstyle\pm2.9}$ & $\bf3570 {\scriptstyle\pm2.8}$  & $\bf16.4 {\scriptstyle\pm0.2}$ &  $21.2 {\scriptstyle\pm3.1}$ & $1388 {\scriptstyle\pm2.5}$ &  $\bf1381 {\scriptstyle\pm2.8}$ \\
    \bottomrule
  \end{tabular}
  }
\vspace{-0.2cm}
\end{table}

\paragraph{Ablation studies.}
To further investigate the differences between GSBM and TSBM, we report additional ablation experiments. Complementary visualizations are provided in \Cref{tsbm:app:crowd}.
\begin{enumerate}[label=(\Alph*), labelindent=0pt, wide]
    \vspace{-0.15cm}
    \item {\textbf{Unidirectional vs. bidirectional.}}
    In \Cref{tsbm:fig:bidir_vs_undir_stunnel}, we compare the training dynamics of the metrics for the two variants of GSBM and TSBM on the 2D-Stunnel task. The results highlight the practical benefit of the bidirectional formulation, which stabilizes convergence, particularly for TSBM.

    \vspace{-0.15cm}
    \item {\textbf{Lower-entropy setting.}}
    We also compare GSBM and TSBM in 2D settings with $\sigma=1$, considering only the pretraining phase of the bidirectional variant for both methods. As expected, the resulting trajectories, shown in \Cref{tsbm:app:crowd}, are straighter than in \Cref{tsbm:fig:crowd_state_main}. Moreover, the numerical results in \Cref{tsbm:tab:crowd_main_ablation} (left column) exhibit trends consistent with those observed for $\sigma=2$.

    \vspace{-0.15cm}
    \item {\textbf{Effect of the control variate.}}
    We assess the impact of the control variate (CV) strategy for TSBM in 2D by comparing the default version, which uses CV, with the no-CV variant. For this ablation, we run only the first outer iteration of the unidirectional setting. Although the generated trajectories, shown in \Cref{tsbm:app:crowd}, appear visually similar, the metrics in \Cref{tsbm:tab:crowd_main_ablation} (right column) generally improve when CV is used, confirming its practical benefit.
\end{enumerate}

\begin{figure}[h!]
    \vspace{-0.3cm}
    \centering
    \includegraphics[width=\linewidth]{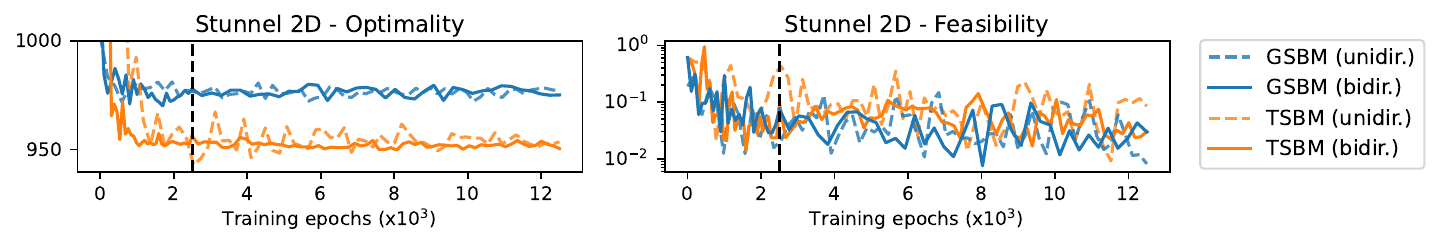}
    \vspace{-0.7cm}
    \caption{Evolution of metrics for TSBM and GSBM on 2D-Stunnel with $\sigma=2$. The vertical dotted line marks the end of bidirectional pretraining or, equivalently, the first unidirectional outer iteration.}
    \label{tsbm:fig:bidir_vs_undir_stunnel}
    \vspace{-2cm}
\end{figure}

\newpage

\begin{table}[h!]

\vspace{-0.4cm}
    \caption{Results of crowd navigation ablation studies for Stunnel, Gmm and Vneck with $d=2$. Results are averaged over $8$ metric evaluations. Bold indicates the best result in each setting.}
    \label{tsbm:tab:crowd_main_ablation}
  \resizebox{\textwidth}{!}{%
  \begin{tabular}{l|cc|cc|cc|cc}
    \toprule
     Setting & \multicolumn{4}{c}{Bidirectional ($\sigma=1$, only pretraining)} & \multicolumn{4}{c}{Unidirectional ($\sigma=2$, only 1 outer iteration)}  \\
     \cmidrule(lr){1-1} \cmidrule(lr){2-5} \cmidrule(lr){6-9}
    Metric & \multicolumn{2}{c}{Feasibility ($\downarrow$)} & \multicolumn{2}{c}{Optimality ($\downarrow$)} & \multicolumn{2}{c}{Feasibility ($\downarrow$)} & \multicolumn{2}{c}{Optimality ($\downarrow$)}    \\
    \cmidrule(lr){1-1} \cmidrule(lr){2-3} \cmidrule(lr){4-5} \cmidrule(lr){6-7} \cmidrule(lr){8-9}
    Method & GSBM & TSBM     & GSBM & TSBM      & TSBM (no CV) & TSBM & TSBM (no CV) & TSBM \\
    \midrule
    2D-Stunnel & $ 0.09{\scriptstyle \pm 0.02}$  &  $\bf 0.03{\scriptstyle \pm 0.003}$ & $958{\scriptstyle\pm 0.9}$ & $\bf 949{\scriptstyle\pm 0.8}$  & $\bf 0.18{\scriptstyle \pm 0.02}$ &   $0.25{\scriptstyle \pm 0.02}$ & $964{\scriptstyle\pm 1.0}$ & $\bf 957{\scriptstyle\pm 1.0}$ \\
    2D-Gmm & $\bf 4.92 {\scriptstyle\pm 0.83}$  & $7.68{\scriptstyle\pm 0.63}$ & $\bf 61.8 {\scriptstyle\pm 3.2}$ & $\bf 62.3{\scriptstyle\pm 0.3}$  & $ 11.5 {\scriptstyle\pm 1.3}$ &    $ \bf 4.6{\scriptstyle\pm 0.8}$ & $68.8 {\scriptstyle\pm 0.3}$ & $\bf 65.5 {\scriptstyle\pm 0.5}$  \\
    2D-Vneck  & $\bf 0.03 {\scriptstyle\pm0.004}$  & $0.09 {\scriptstyle\pm0.01}$  & $64.9 {\scriptstyle\pm0.2}$ & $\bf 48.0 {\scriptstyle\pm0.2}$  & $0.03 {\scriptstyle\pm0.002}$ &  $\bf 0.02 {\scriptstyle\pm0.005}$ & $\bf 60.2 {\scriptstyle\pm0.3}$ &  $\bf 60.5 {\scriptstyle\pm0.3}$ \\
    \bottomrule
  \end{tabular}
  }
\end{table}
\vspace{-0.3cm}
\subsection{Singe-cell inference under sparse observations}
We further evaluate the \underline{bidirectional} variants of TSBM and GSBM on single-cell data \cite{Moon2019} in dimension $d\in\{2,5,50\}$. We consider a \emph{sparse-observation regime} by intentionally retaining only $25\%$ of the intermediate observations, thereby precluding the training of a multi-marginal model. Based on this data, we design a continuous-time state cost inspired by \cite[Eq.~12]{tong2020trajectorynet}. Given snapshot times $\{t_k\}_{k=1}^K$, with $t_1=0$ and $t_K=T$, we define $V_t$ as a time-linear interpolation between two static state costs, which attract particles toward the available observations closest to the two snapshot times bracketing $t\in [0,T]$ (the cost being set to zero on boundary times). These static costs depend on a hyperparameter $\beta>0$ controlling the locality of the interaction: small values of $\beta$ yield nearly uniform weights over all observations, whereas large values emphasize the nearest ones. We refer to \Cref{tsbm:app:single_cell} for the detailed expression. We emphasize that this attractive cost is not specific to single-cell data and can be applied more broadly to sparse-observation settings. As shown in \Cref{tsbm:tab:single_cell}, we observe the same optimality--feasibility trade-off as in crowd navigation, with a more pronounced gap in higher dimensions, while \Cref{tsbm:fig:crowd_main} illustrates how sparse observations guide TSBM trajectories.

\begin{table}[h!]
\vspace{-0.4cm}
\caption{Results of single-cell experiments with $\beta\in \{20,100\}$, $\sigma=0.1$,  with increasing $d$. Results are averaged over $8$ metric evaluations. Bold indicates the best result in each setting.}
\label{tsbm:tab:single_cell}
  \resizebox{\textwidth}{!}{%
  \begin{tabular}{l|cc|cc|cc|cc}
    \toprule
     Hyperparameter & \multicolumn{4}{c}{$\beta=20$} & \multicolumn{4}{c}{$\beta=100$}  \\
     \cmidrule(lr){1-1} \cmidrule(lr){2-5} \cmidrule(lr){6-9}
    Metric & \multicolumn{2}{c}{Feasibility ($\downarrow$)} & \multicolumn{2}{c}{Optimality ($\downarrow$)} & \multicolumn{2}{c}{Feasibility ($\downarrow$)} & \multicolumn{2}{c}{Optimality ($\downarrow$)}    \\
    \cmidrule(lr){1-1} \cmidrule(lr){2-3} \cmidrule(lr){4-5} \cmidrule(lr){6-7} \cmidrule(lr){8-9}
    Method & GSBM & TSBM     & GSBM & TSBM      & GSBM & TSBM & GSBM & TSBM \\
    \midrule
    $d=2$ & $\bf 0.05{\scriptstyle \pm 0.005}$  &  $\bf0.05{\scriptstyle \pm 0.003}$ & $\bf7.23{\scriptstyle\pm 0.01}$ & $7.25{\scriptstyle\pm 0.01}$  & $\bf 0.011{\scriptstyle \pm 0.002}$ &  $\bf 0.011{\scriptstyle \pm 0.001}$ & $\bf6.79{\scriptstyle\pm 0.01}$ & $6.82{\scriptstyle\pm 0.01}$ \\
    $d=5$ & $\bf3.05 {\scriptstyle\pm 0.01}$  & $3.06{\scriptstyle\pm 0.02}$ & $31.05 {\scriptstyle\pm 0.01}$ & $\bf29.93{\scriptstyle\pm 0.01}$  & $ 2.96 {\scriptstyle\pm 0.02}$ &    $ \bf 2.81{\scriptstyle\pm 0.02}$ & $29.68 {\scriptstyle\pm 0.02}$ & $\bf29.3 {\scriptstyle\pm 0.02}$  \\
    $d=50$  & $\bf 38.79{\scriptstyle\pm0.04}$  & $46.31 {\scriptstyle\pm0.06}$  & $31.75 {\scriptstyle\pm0.01}$ & $\bf 22.53 {\scriptstyle\pm0.01}$  & $\bf39.99 {\scriptstyle\pm0.03}$ &  $41.92 {\scriptstyle\pm0.04}$ & $30.76 {\scriptstyle\pm0.02}$ &  $\bf26.9 {\scriptstyle\pm0.01}$ \\
    \bottomrule
  \end{tabular}
  }
\end{table}

\begin{figure}[h!]
\vspace{-0.5cm}
    \centering
    \includegraphics[width=\linewidth]{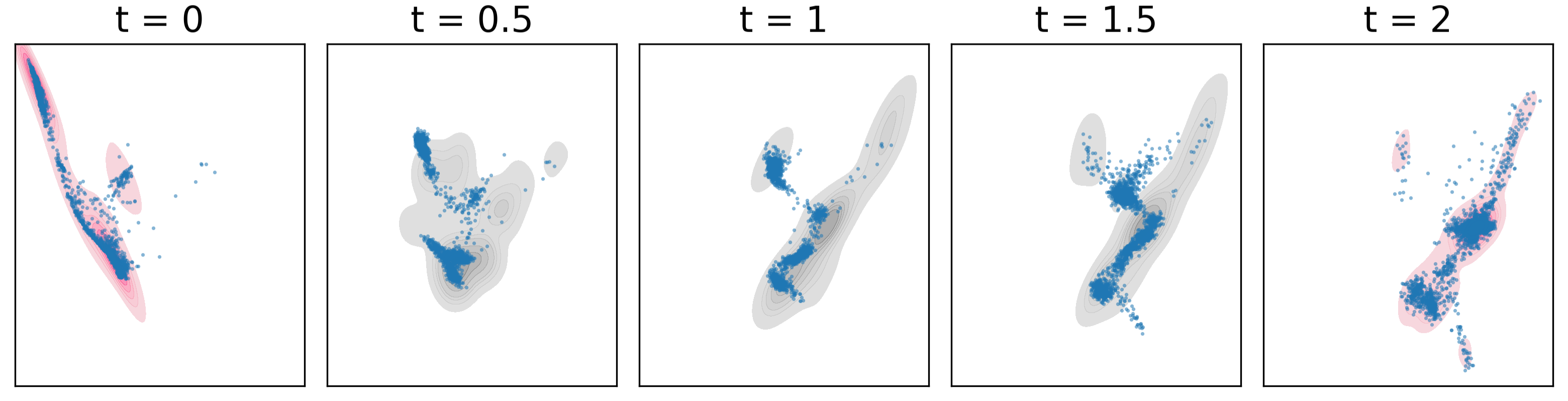}
    \vspace{-0.6cm}
    \caption{Snapshots of learned TSBM dynamics for 2D single-cell ($\beta=100$, $\sigma=0.1$) with respect to available observations (endpoint \emph{marginal} samples in pink, intermediary \emph{sparse} samples in grey).}
    \label{tsbm:fig:crowd_main}
    \vspace{-0.4cm}
\end{figure}

\section{Conclusion \& Perspectives}

In this work, we revisit the trajectory-inference problem with time-dependent state costs and endpoint marginal constraints studied by \cite{liu2023generalized}, through the lens of \emph{Twisted Schrödinger Bridge} (TSB). We extend the canonical SB setting by replacing Brownian motion with a twisted Brownian reference measure, obtained by reweighting trajectories according to the state cost. Building on this, we derive \emph{Twisted Schrödinger Bridge Matching} (TSBM), a bridge-matching implementation of Iterative Markovian Fitting (IMF) for continuous- and discrete-time state costs. Compared with GSBM \citep{liu2023generalized}, TSBM introduces Markovian projection losses that leverage state-cost information to reduce bias, exactly recover DSBM \citep{shi2023diffusion} in the zero-potential case, support adaptive control variates and largely benefit from their bidirectional extension. We validate TSBM on crowd-navigation and single-cell tasks, where it matches or improves upon GSBM. Several limitations remain. TSBM requires differentiable state costs with informative gradients on the relevant domain; otherwise, the state-cost contribution may be underestimated in the TSBM loss. Moreover, its state-cost-informed objective is more computationally demanding than GSBM due to potential-gradient evaluations. Regarding the computational cost of iterative inference and learning, combining TSBM and GSBM with recent fast SDE inference techniques \citep{pan2026maps,mccallum2026strong} is a promising direction toward more lightweight SB procedures.
\vspace{-0.5cm}

\newpage
\begin{ack}
We warmly thank Stefano Peluchetti and Valentin De Bortoli for insightful discussions at the inception of this project. This work was performed using
HPC resources from GENCI–IDRIS (AD011014860R1, AD011014860R2). This work received government funding administered by the National Research Agency (ANR) under the France 2030 program "Hi! PARIS", grant number ANR-23-IACL-0005.
MD thanks the Chaire Deep Finance and Statistics, a joint research program between Qube R$\&$T, Ecole polytechnique and Fondation de l'Ecole polytechnique, for financial support.
AD received funding from the Fondation de l’École polytechnique as part of its “Servir la science” campaign.
AD is supported by the France 2030 program with the reference ANR-25-PEIA-0001 (THEOREM project). 
AD is funded by the European Union (ERC-2022-SYG-OCEAN-101071601). Views and opinions expressed are however those of the author only and do not necessarily reflect those of the European Union or the European Research Council Executive Agency. Neither the European Union nor the granting authority can be held responsible for them. AD is supported by Hi! Paris and Agence Nationale de la Recherche (Grant 11-LABX-0047). 
\end{ack}

\bibliographystyle{plainnat}
\bibliography{main}

@article{albergo2025stochastic,
  author = {Albergo, Michael and Boffi, Nicholas M and Vanden-Eijnden, Eric},
  journal = {Journal of Machine Learning Research},
  number = {209},
  pages = {1--80},
  title = {Stochastic interpolants: A unifying framework for flows and diffusions},
  volume = {26},
  year = {2025}
}

@article{benamou2000computational,
  author = {Benamou, Jean-David and Brenier, Yann},
  journal = {Numerische Mathematik},
  number = {3},
  pages = {375--393},
  publisher = {Springer-Verlag Berlin/Heidelberg},
  title = {A computational fluid mechanics solution to the Monge-Kantorovich mass transfer problem},
  volume = {84},
  year = {2000}
}

@inproceedings{bortoli2024schrodinger,
  author = {Valentin De Bortoli and Iryna Korshunova and Andriy Mnih and Arnaud Doucet},
  booktitle = {Advances in Neural Information Processing Systems},
  title = {Schrodinger Bridge Flow for Unpaired Data Translation},
  year = {2024}
}

@article{bregman1967relaxation,
  author = {Bregman, Lev M},
  journal = {USSR computational mathematics and mathematical physics},
  number = {3},
  pages = {200--217},
  publisher = {Elsevier},
  title = {The relaxation method of finding the common point of convex sets and its application to the solution of problems in convex programming},
  volume = {7},
  year = {1967}
}

@article{chen2021likelihood,
  author = {Chen, Tianrong and Liu, Guan-Horng and Theodorou, Evangelos A},
  journal = {International Conference on Learning Representations},
  title = {Likelihood training of {S}chr\"odinger bridge using forward-backward {SDE}s theory},
  year = {2022}
}

@article{chen2023deepmomentum,
  author = {Chen, Tianrong and Liu, Guan-Horng and Tao, Molei and Theodorou, Evangelos},
  journal = {Advances in Neural Information Processing Systems},
  pages = {57058--57086},
  title = {Deep momentum multi-marginal Schr{\"o}dinger bridge},
  year = {2023}
}

@inproceedings{chetrite2015nonequilibrium,
  author = {Chetrite, Rapha{\"e}l and Touchette, Hugo},
  booktitle = {Annales Henri Poincar{\'e}},
  organization = {Springer},
  pages = {2005--2057},
  title = {Nonequilibrium Markov processes conditioned on large deviations},
  volume = {16},
  year = {2015}
}

@article{DeBortoli2021diffusion,
  author = {Valentin De Bortoli and James Thornton and Jeremy Heng and Arnaud Doucet},
  journal = {Advances in Neural Information Processing Systems},
  pages = {17695--17709},
  title = {Diffusion Schr{\"o}dinger bridge with applications to score-based generative modeling},
  year = {2021}
}

@article{domingo2023stochastic,
  author = {Domingo i Enrich, Carles and Han, Jiequn and Amos, Brandon and Bruna, Joan and Chen, Ricky TQ},
  journal = {Advances in Neural Information Processing Systems},
  pages = {112459--112504},
  title = {Stochastic optimal control matching},
  year = {2024}
}

@article{du2024doob,
  author = {Du, Yuanqi and Plainer, Michael and Brekelmans, Rob and Duan, Chenru and Noe, Frank and Gomes, Carla P and Aspuru-Guzik, Alan and Neklyudov, Kirill},
  journal = {Advances in Neural Information Processing Systems},
  pages = {65791--65822},
  title = {Doob's Lagrangian: A Sample-Efficient Variational Approach to Transition Path Sampling},
  year = {2024}
}

@inproceedings{feydy2019interpolating,
  author = {Feydy, Jean and S{\'e}journ{\'e}, Thibault and Vialard, Fran{\c{c}}ois-Xavier and Amari, Shun-ichi and Trouve, Alain and Peyr{\'e}, Gabriel},
  booktitle = {International Conference on Artificial Intelligence and Statistics},
  pages = {2681--2690},
  title = {Interpolating between Optimal Transport and MMD using Sinkhorn Divergences},
  year = {2019}
}

@article{gushchin2024adversarial,
  author = {Gushchin, Nikita and Selikhanovych, Daniil and Kholkin, Sergei and Burnaev, Evgeny and Korotin, Aleksandr},
  journal = {Advances in Neural Information Processing Systems},
  pages = {89612--89651},
  title = {Adversarial schr{\"o}dinger bridge matching},
  year = {2024}
}

@article{heng2020controlled,
  author = {Heng, Jeremy and Bishop, Adrian N and Deligiannidis, George and Doucet, Arnaud},
  journal = {The Annals of Statistics},
  number = {5},
  pages = {2904--2929},
  publisher = {JSTOR},
  title = {Controlled sequential monte carlo},
  volume = {48},
  year = {2020}
}

@article{Ho2020denoising,
  author = {Ho, Jonathan and Jain, Ajay and Abbeel, Pieter},
  journal = {Advances in Neural Information Processing Systems},
  pages = {6840--6851},
  title = {Denoising diffusion probabilistic models},
  year = {2020}
}

@inproceedings{howard2025schrdinger,
  author = {Samuel Howard and Peter Potaptchik and George Deligiannidis},
  booktitle = {Advances in Neural Information Processing Systems},
  title = {Schr\"odinger Bridge Matching for Tree-Structured Costs and Entropic Wasserstein Barycentres},
  year = {2025}
}

@article{Kingma2021variational,
  author = {Kingma, Diederik and Salimans, Tim and Poole, Ben and Ho, Jonathan},
  journal = {Advances in Neural Information Processing Systems},
  pages = {21696--21707},
  title = {Variational diffusion models},
  year = {2021}
}

@article{leonard2011stochastic,
  author = {L{\'e}onard, Christian},
  journal = {arXiv preprint arXiv:1102.3172},
  title = {Stochastic derivatives and generalized h-transforms of Markov processes},
  year = {2011}
}

@article{Leonard2014some,
  author = {L{\'e}onard, Christian},
  journal = {S{\'e}minaire de Probabilit{\'e}s XLVI},
  pages = {207--230},
  publisher = {Springer},
  title = {Some properties of path measures},
  year = {2014}
}

@article{leonard2014survey,
  author = {L{\'e}onard, Christian},
  journal = {Discrete \& Continuous Dynamical Systems-A},
  number = {4},
  pages = {1533--1574},
  title = {A survey of the {S}chr{\"o}dinger problem and some of its connections with optimal transport},
  volume = {34},
  year = {2014}
}

@article{lipman2022flow,
  author = {Lipman, Yaron and Chen, Ricky TQ and Ben-Hamu, Heli and Nickel, Maximilian and Le, Matt},
  journal = {International Conference on Learning Representations},
  title = {Flow matching for generative modeling},
  year = {2023}
}

@article{liu2022deep,
  author = {Liu, Guan-Horng and Chen, Tianrong and So, Oswin and Theodorou, Evangelos},
  journal = {Advances in Neural Information Processing Systems},
  pages = {9374--9388},
  title = {Deep generalized schr{\"o}dinger bridge},
  year = {2022}
}

@article{liu2022rectified,
  author = {Liu, Qiang},
  journal = {arXiv preprint arXiv:2209.14577},
  title = {Rectified flow: A marginal preserving approach to optimal transport},
  year = {2022}
}

@inproceedings{liu2023generalized,
  author = {Guan-Horng Liu and Yaron Lipman and Maximilian Nickel and Brian Karrer and Evangelos Theodorou and Ricky T. Q. Chen},
  booktitle = {International Conference on Learning Representations},
  title = {Generalized Schr\"odinger Bridge Matching},
  year = {2024}
}

@inproceedings{neklyudov2023action,
  author = {Neklyudov, Kirill and Brekelmans, Rob and Severo, Daniel and Makhzani, Alireza},
  booktitle = {International Conference on Machine Learning},
  organization = {PMLR},
  pages = {25858--25889},
  title = {Action matching: Learning stochastic dynamics from samples},
  year = {2023}
}

@inproceedings{neklyudov2023computational,
  author = {Neklyudov, Kirill and Brekelmans, Rob and Tong, Alexander and Atanackovic, Lazar and Liu, Qiang and Makhzani, Alireza},
  booktitle = {International Conference on Machine Learning},
  pages = {37461--37485},
  publisher = {PMLR},
  title = {A Computational Framework for Solving {W}asserstein Lagrangian Flows},
  year = {2024}
}

@article{nusken2021solving,
  author = {N{\"u}sken, Nikolas and Richter, Lorenz},
  journal = {Partial differential equations and applications},
  number = {4},
  pages = {48},
  publisher = {Springer},
  title = {Solving high-dimensional Hamilton--Jacobi--Bellman {PDEs} using neural networks: perspectives from the theory of controlled diffusions and measures on path space},
  volume = {2},
  year = {2021}
}

@article{peluchetti2023diffusion,
  author = {Peluchetti, Stefano},
  journal = {Journal of Machine Learning Research},
  number = {374},
  pages = {1--51},
  title = {Diffusion bridge mixture transports, Schr{\"o}dinger bridge problems and generative modeling},
  volume = {24},
  year = {2023}
}

@article{peyre2019computational,
  author = {Peyr{\'e}, Gabriel and Cuturi, Marco and others},
  journal = {Foundations and Trends{\textregistered} in Machine Learning},
  number = {5-6},
  pages = {355--607},
  publisher = {Now Publishers, Inc.},
  title = {Computational optimal transport: With applications to data science},
  volume = {11},
  year = {2019}
}

@book{rogers2000diffusions,
  author = {Rogers, L Chris G and Williams, David},
  publisher = {Cambridge university press},
  title = {Diffusions, Markov processes, and martingales: It{\^o} calculus},
  volume = {2},
  year = {2000}
}

@book{roynette2009penalising,
  author = {Roynette, Bernard and Yor, Marc},
  publisher = {Springer Science \& Business Media},
  title = {Penalising brownian paths},
  volume = {1969},
  year = {2009}
}

@inproceedings{sabour2025alignflowscalingcontinuoustime,
  author = {Amirmojtaba Sabour and Sanja Fidler and Karsten Kreis},
  booktitle = {Advances in Neural Information Processing Systems},
  title = {Align Your Flow: Scaling Continuous-Time Flow Map Distillation},
  year = {2025}
}

@article{schrodinger1932theorie,
  author = {Schr{\"o}dinger, Erwin},
  journal = {Annales de l'Institut Henri Poincar{\'e}},
  number = {4},
  pages = {269--310},
  title = {Sur la th{\'e}orie relativiste de l'{\'e}lectron et l'interpr{\'e}tation de la m{\'e}canique quantique},
  volume = {2},
  year = {1932}
}

@article{shi2023diffusion,
  author = {Shi, Yuyang and De Bortoli, Valentin and Campbell, Andrew and Doucet, Arnaud},
  journal = {Advances in Neural Information Processing Systems},
  title = {Diffusion schr{\"o}dinger bridge matching},
  year = {2023}
}

@inproceedings{silveri2025exponential,
  author = {Silveri, Marta Gentiloni and Conforti, Giovanni and Durmus, Alain},
  booktitle = {Advances in Neural Information Processing Systems},
  title = {Exponential Convergence Guarantees for Iterative Markovian Fitting},
  year = {2025}
}

@inproceedings{song2020score,
  author = {Yang Song and Jascha Sohl-Dickstein and Diederik P Kingma
and Abhishek Kumar and Stefano Ermon and Ben Poole},
  booktitle = {International Conference on Learning
Representations},
  title = {Score-Based Generative Modeling through Stochastic
Differential Equations},
  year = {2021}
}

@article{tamir2023transport,
  author = {Tamir, Ella and Trapp, Martin and Solin, Arno},
  journal = {Transactions on Machine Learning Research},
  title = {Transport with Support: Data-Conditional Diffusion Bridges},
  year = {2023}
}

@article{teter2024schr,
  author = {Teter, Alexis MH and Wang, Wenqing and Halder, Abhishek},
  journal = {arXiv preprint arXiv:2406.00503},
  title = {Schr{\"o}dinger Bridge with Quadratic State Cost is Exactly Solvable},
  year = {2024}
}

@inproceedings{tong2020trajectorynet,
  author = {Tong, Alexander and Huang, Jessie and Wolf, Guy and Van Dijk, David and Krishnaswamy, Smita},
  booktitle = {International Conference on Machine Learning},
  organization = {PMLR},
  pages = {9526--9536},
  title = {Trajectorynet: A dynamic optimal transport network for modeling cellular dynamics},
  year = {2020}
}

@inproceedings{blessing2026bridge,
  title={Bridge Matching Sampler: Scalable Sampling via Generalized Fixed-Point Diffusion Matching},
  author={Blessing, Denis and Richter, Lorenz and Berner, Julius and Malitskiy, Egor and Neumann, Gerhard},
  booktitle = {International Conference on Machine Learning},
  year={2026}
}

@article{noble2023tree,
  title={Tree-based diffusion {S}chr{\"o}dinger bridge with applications to {W}asserstein barycenters},
  author={Noble, Maxence and De Bortoli, Valentin and Doucet, Arnaud and Durmus, Alain},
  journal={Advances in Neural Information Processing Systems},
  volume={36},
  pages={55193--55236},
  year={2023}
}

@inproceedings{
pooladian2024neural,
title={Neural Optimal Transport with Lagrangian Costs},
author={Aram-Alexandre Pooladian and Carles Domingo-Enrich and Ricky T. Q. Chen and Brandon Amos},
organization = {PMLR},
booktitle={Uncertainty in Artificial Intelligence},
year={2024},
}

@inproceedings{
kapusniak2405metric,
title={Metric Flow Matching for Smooth Interpolations on the Data Manifold},
author={Kacper Kapusniak and Peter Potaptchik and Teodora Reu and Leo Zhang and Alexander Tong and Michael M. Bronstein and Joey Bose and Francesco Di Giovanni},
booktitle={Advances in Neural Information
Processing Systems},
year={2024},
}

@inproceedings{
chen2023flow,
title={Flow Matching on General Geometries},
author={Ricky T. Q. Chen and Yaron Lipman},
booktitle={International Conference on Learning Representations},
year={2024},
}

@inproceedings{del2001stability,
  title={On the stability of interacting processes with applications to filtering and genetic algorithms},
  author={Del Moral, Pierre and Guionnet, Alice},
  booktitle={Annales de l'Institut Henri Poincar{\'e} (B) Probability and Statistics},
  volume={37},
  number={2},
  pages={155--194},
  year={2001},
  organization={Elsevier}
}

@article{fujisaki1972stochastic,
  title={Stochastic differential equations for the non linear filtering problem},
  author={Fujisaki, M and Kallianpur, G and Kunita, H},
  journal={Osaka J. Math},
  volume={9},
  pages={19--40},
  year={1972}
}

@article{kallianpur1968estimation,
  title={Estimation of stochastic systems: Arbitrary system process with additive white noise observation errors},
  author={Kallianpur, G and Striebel, C.},
  journal={The Annals of Mathematical Statistics},
  volume={39},
  number={3},
  pages={785--801},
  year={1968},
  publisher={JSTOR}
}

@article{takeuchi1981nonlinear,
  title={Nonlinear filtering formulas for discrete-time observations},
  author={Takeuchi, Yoshiki and Akashi, Hajime},
  journal={SIAM Journal on Control and Optimization},
  volume={19},
  number={2},
  pages={244--261},
  year={1981},
  publisher={SIAM}
}

@article{del2001monte,
  title={The Monte-Carlo method for filtering with discrete-time observations},
  author={Del Moral, Pierre and Jacod, Jean and Protter, Philip},
  journal={Probability Theory and Related Fields},
  volume={120},
  number={3},
  pages={346--368},
  year={2001},
  publisher={Springer Science and Business Media LLC}
}

@article{Moon2019,
   author = {Kevin R. Moon and David van Dijk and Zheng Wang and Scott Gigante and Daniel B. Burkhardt and William S. Chen and Kristina Yim and Antonia van den Elzen and Matthew J. Hirn and Ronald R. Coifman and Natalia B. Ivanova and Guy Wolf and Smita Krishnaswamy},
   institution = {bioRxiv},
   journal = {bioRxiv},
   pages = {120378},
   publisher = {Cold Spring Harbor Laboratory},
   title = {Visualizing Structure and Transitions for Biological Data Exploration},
   year = {2019}
}

@article{pan2026maps,
  title={Itô maps for any-step SDEs},
  author={Pan, Zhengkai and Potaptchik, Peter and Yao, Wenxi and Albergo, Michael S and Pidstrigach, Jakiw},
  journal={arXiv preprint arXiv:2606.11156},
  year={2026}
}

@article{mccallum2026strong,
  title={Strong Stochastic Flow Maps},
  author={McCallum, Sam and Blasingame, Zander W and Herschell, Timothy and Rindtorff, Niklas and Tong, Alexander and Foster, James},
  journal={arXiv preprint arXiv:2606.01086},
  year={2026}
}

@inproceedings{lu2025simplifying,
  title={Simplifying, stabilizing and scaling continuous-time consistency models},
  author={Lu, Cheng and Song, Yang},
  booktitle={International Conference on Learning Representations},
  year={2025}
}

@article{peluchetti2023non,
  title={Non-denoising forward-time diffusions},
  author={Peluchetti, Stefano},
  journal={arXiv preprint arXiv:2312.14589},
  year={2023}
}

\newpage
\appendix

\etocdepthtag.toc{appendix}

\section*{Organization of the supplementary}

The appendix is organized as follows. \Cref{tsbm:app:notation} summarizes the notation used throughout the paper. \Cref{tsbm:app:related_work} discusses related work on generalized SB problems and alternative formulations of the standard SB problem. \Cref{tsbm:app:proofs} collects theoretical results that support the Twisted SB framework introduced in \Cref{tsbm:sec:background}, and provides the proofs of the results stated in \Cref{tsbm:sec:continuous}. \Cref{tsbm:app:full_discrete} presents a fully discrete-time reformulation of the Twisted SB problem with discrete-time potentials, with corresponding methodology coined \emph{Discrete-TSBM} (D-TSBM). Extending the canonical framework of \cite{gushchin2024adversarial} to twisted reference measures, we show in particular that the corresponding Markovian projection loss reduces to the widely used DDPM loss \citep{Ho2020denoising} in the zero-potential case, further highlighting the connection between TSBM and standard diffusion-model objectives. For clarity, the proofs of the results introduced in \Cref{tsbm:app:full_discrete} are deferred to \Cref{tsbm:app:discrete_proofs}. \Cref{tsbm:app:practical} provides implementation details for TSBM, focusing on two main aspects: the variational learning framework used to approximate the twisted bridge in the reciprocal projection step, see \Cref{tsbm:subsec:reciprocal_cont}, and the control variate parameterization for the TSBM Markovian losses. Finally, \Cref{tsbm:app:xps} gives additional details on the experimental setup and reports further results.

\begingroup
\etocsettagdepth{main}{none}
\etocsettagdepth{appendix}{subsection}
\etocsettocstyle{}{}
\tableofcontents
\endgroup

\newpage

\section{Notation} \label{tsbm:app:notation}

We denote by $\rmC^{0,1}([0,T]\times \rset^d,\rset)$ the space of time-dependent potentials on $[0,T]\times\rset^d$ that are continuous in time and continuously differentiable in space. Similarly, $\rmC^1(\rset^d,\rset)$ denotes the space of continuously differentiable potentials on $\rset^d$.

For any measurable space $(\msx,\mathcal{X})$, we write $\Pmeasure(\msx)$ for the set of probability measures on $(\msx,\mathcal{X})$. Unless stated otherwise, $\mathcal{X}$ is the Borel $\sigma$-algebra on $\msx$. We denote by $\setFunConT=\rmC(\ccint{0,T},\rset^d)$ the space of continuous paths from $\ccint{0,T}$ to $\rset^d$; accordingly, elements of $\Pmeasure(\setFunConT)$ are continuous-time stochastic processes, or path measures, on $[0,T]$. Given a path measure $\Pbb\in\Pmeasure(\setFunConT)$ and a time index set $\msi\subset\ccint{0,T}$, either discrete or continuous, we denote by $\Pbb_{\msi}$ the law of $\bfX_{\msi}$ when $\bfX\sim\Pbb$. In particular, $\Pbb_t$ is the marginal of $\Pbb$ at time $t\in[0,T]$, while $\Pbb_{0,T}$ denotes the joint law of the initial and terminal endpoints. We write $\Pbb_{|0,T}(\cdot|x_0,x_T)$ for the conditional path measure of $\bfX\sim\Pbb$ given $\bfX_0=x_0$ and $\bfX_T=x_T$.

We will frequently use mixtures of endpoint-conditioned bridges. Given a coupling $\pi_{0,T}\in\Pmeasure(\rset^d\times\rset^d)$ and a path measure $\Qbb\in\Pmeasure(\setFunConT)$, we denote by $\Pbb=\pi_{0,T}\Qbb_{|0,T}$ the path measure defined by $\Pbb(\cdot)=\int_{\rset^d\times\rset^d}\Qbb_{|0,T}(\cdot|x_0,x_T)\,\rmd\pi_{0,T}(x_0,x_T)$. We also denote by $\mathcal{M}_T\subset\Pmeasure(\setFunConT)$ the set of Markov continuous-time path measures. For discrete-time processes, and for any $N\geq0$, we use the shorthand $\Pmeasure^{(N+2)}=\Pmeasure((\rset^d)^{N+2})$, which corresponds to probability measures over $(N+2)$ time steps, namely the two endpoints together with $N$ intermediate time steps. When $\msx=(\rset^d)^{N+2}$, for any $x\in\msx$ and any indices $m,n\in\{0,\ldots,N+1\}$ with $m\leq n$, we write $x_{m:n}=(x_m,x_{m+1},\ldots,x_n)$.

Let $\Leb$ denote the Lebesgue measure. For any nonnegative measurable function $f:\msx\to\rset_+$ such that $\int_{\msx} f\,\rmd\Leb<\infty$, we define $\ent(f)=-\int_{\msx} f\log f\,\rmd\Leb\in\ocint{-\infty,\plusinfty}$. For any probability measure $\mu\in\Pmeasure(\msx)$, its entropy is defined by $\ent(\mu)=\ent(\rmd\mu/\rmd\Leb)$ if $\mu\ll\Leb$, and $\ent(\mu)=+\infty$ otherwise. For two arbitrary measures $\mu$ and $\nu$ on $(\msx,\mathcal{X})$, we define the Kullback--Leibler divergence by $\KL(\mu\|\nu)=\int_{\msx}\log(\rmd\mu/\rmd\nu)\,\rmd\mu-\int_{\msx}\rmd\mu+\int_{\msx}\rmd\nu$ if $\mu\ll\nu$, and set $\KL(\mu\|\nu)=+\infty$ otherwise. Finally, $\densityGaussian(\mathbf{m},\Sigma)$ denotes the Gaussian distribution with mean $\mathbf{m}$ and covariance matrix $\Sigma$.

\section{Related works} \label{tsbm:app:related_work}

\paragraph{Potential-based approaches for solving the generalized SB problem.}

Beyond bridge-matching methods, a closely related line of work is the action-matching (AM) framework \citep{neklyudov2023action}. AM is a diffusion-based generative modeling approach designed to learn marginal-preserving stochastic or deterministic dynamics from observed intermediate samples. Its main idea is to parameterize the target velocity field as the gradient of a time-dependent potential $(\phi_t)_{t\in[0,T]}$, approximated by a neural network $(\phi_t^\theta)_{t\in[0,T]}$. Unlike bridge matching, which \emph{explicitly} regresses onto a conditional drift---namely, the drift of the dynamics induced by the underlying bridge path measure---the AM objective is \emph{implicit}: it only requires evaluating the neural time derivative $\partial_t\phi_t^\theta$, the gradient $\nabla\phi_t^\theta$, and, in the stochastic case, the Laplacian $\Delta\phi_t^\theta$.

More recently, \cite{neklyudov2023computational} extended action matching to a range of SB-like problems, including standard, unbalanced, and generalized Schrödinger bridge formulations close to \eqref{tsbm:eq:standard_soc}. Their approach casts the problem as a min--max optimization. As in TSBM, training alternates between two stages: first learning a neural interpolant, and then learning dynamics that preserve the marginals induced by this interpolant, while generating an updated coupling between endpoints. This is achieved by optimizing a joint objective based on a variant of the AM loss. In contrast to TSBM, their method does not require explicitly updating and storing the endpoint coupling; in particular, one may keep the independent coupling throughout training, reducing the memory footprint and avoiding intermediate simulations such as those used in our procedure. However, a key computational limitation is that evaluating the AM loss requires neural-network derivatives: first-order derivatives in the deterministic case, and second-order derivatives when $\sigma>0$, as in the stochastic setting considered in this paper. This restricts its practical scalability to relatively low-dimensional settings. By contrast, the TSBM loss is derivative-free with respect to the neural drift parameterization, although it may require derivatives of the prescribed state cost.

Another related contribution is Lagrangian Optimal Transport \citep{pooladian2024neural}, which proposes a learning approach for generalized SB-type problems in the zero-noise regime, i.e., $\sigma=0$. Rather than learning the deterministic dynamics themselves, this method directly estimates the underlying optimal transport map. It also relies on a min--max formulation, alternating between the construction of a suitable interpolant and the learning of Kantorovich potentials, parameterized by input-convex neural networks, from which the transport map is recovered. The interpolant is spline-based, similarly to the lightweight parameterization we use in our implementation. However, this approach is restricted to the deterministic setting, whereas TSBM targets stochastic dynamics. Moreover, as in the AM framework, it involves input derivatives of neural networks, which makes it most suitable for low-dimensional problems.

\paragraph{Diffusion-based approaches for trajectory inference with prior data.}

In recent years, several works have extended diffusion-based methods beyond unconditional generation between source and target distributions, by incorporating additional prior information into the learned dynamics. One direction is the multi-marginal extension of the canonical SB problem, which decomposes the trajectory into several SB sub-problems over a single time interval. This requires enough samples at the intermediate times of interest to model these intermediate marginals in the same way as the endpoint distributions. Related approaches include \cite{chen2023deepmomentum}, building on the computational SB framework of \cite{chen2021likelihood}; \cite{noble2023tree}, building on \cite{DeBortoli2021diffusion}; and \cite{howard2025schrdinger}, building on DSBM \citep{shi2023diffusion}. Intermediate samples can also be used outside the SB framework. For example, \cite{kapusniak2405metric} propose Metric Flow Matching (MFM), an instance of Riemannian Flow Matching \citep{chen2023flow}, in which the geodesic interpolant is learned from a Riemannian metric constructed using the available samples. Although this approach achieves strong empirical performance, its connection with generalized SB problems remains unclear. In the sparse-observation regime, where too few samples are available to reliably define intermediate marginals, \cite{tamir2023transport} propose a method based on Bayesian filtering and optimal control, built on the diffusion backbone of \cite{DeBortoli2021diffusion}, to generate constrained stochastic processes. Here again, the precise relation to generalized SB formulations is not explicit.

\section{Theoretical results on TSBM} \label{tsbm:app:proofs}
Throughout this section, we work on the canonical measurable space $(\rset^d, \mathcal{B}(\rset^d))$, where $\mathcal{B}(\rset^d)$ denotes the Borel $\sigma$-field of $\rset^d$. On this space, we consider a $d$-dimensional Brownian motion $(\bfB_t)_{t \geq 0}$.
We define $\adminControlT$, the set of admissible control functions, by
\begin{align*}
  \adminControlT = \{ u \in \mathrm{C}^1([0,T]\times \rset^d, \rset^d): \exists C>0: \forall (t,x)\in [0,T]\times \rset^d, \norm{u_t(x)}\leq C (1 + \norm{x}) \}\eqsp.
\end{align*}

We first recall that for any $(x_0, x_T)\in \rset^d\times \rset^d$, by simple application of Doob's $h$-transform \citep{rogers2000diffusions}, the scaled Brownian bridge $\Qbbsg_{|0,T}(\cdot|x_0, x_T)$
is a pinned diffusion process equivalently associated to the forward and backward SDEs
\begin{align}
    \rmd \bfX_t^{0,T}&= \bfu_t^{\rmf}(\bfX_t^{0,T}) \rmd t+ \sigma \rmd \bfB_t,\eqsp \bfX^{0,T}_0 = x_0,\eqsp\bfX^{0,T}_T = x_T\eqsp, \eqsp \bfu_t^{\rmf}(x)=\sigma^2\nabla_x \log \Qbbsg_{T|t}(x_T|x) \eqsp ,\label{tsbm:eq:brownian_bridge_f}\\
    \rmd \bfY_t^{0,T}&= \bfu_{T-t}^{\rmb}(\bfY_t^{0,T}) \rmd t+ \sigma \rmd \bfB_t,\eqsp \bfY^{0,T}_0 = x_T,\bfY^{0,T}_T = x_0, \bfu_t^{\rmb}(x)=\sigma^2\nabla_x \log \Qbbsg_{t|0}(x|x_0) .\label{tsbm:eq:brownian_bridge_b}
\end{align}
These drifts simplify as $\bfu_t^{\rmf}(x)=(x_T-x_t)/(T-t)$ and $\bfu_t^{\rmb}(x)=(x_0-x_t)/t$.

\subsection{Results on twisted measures}\label{tsbm:app:twisted}

We begin by recalling the following result from \citep[Proposition 4.2]{leonard2011stochastic} concerning the twisted Brownian motion $\QbbVsg$ defined in \eqref{tsbm:eq:q_v}.

\begin{proposition} \label{tsbm:app:prop_leonard} Let $V\in \rmC^{0,1}([0,T]\times \rset^d, \rset)$ be such that $Z_V< \infty$. Then it holds that
(i) $\QbbVsg$ is Markov, (ii) for any $0\leq s \leq t \leq T$, we have
$$
\textstyle\rmd \QbbVsg_{s:t}(x_{[s:t]})\propto \exp\left(-\int_{s}^t \frac{V_u(x_u)}{\sigma^2} \rmd u\right) \rmd \Qbbsg_{s:t}(x_{[s:t]}) \eqsp .
$$
\end{proposition}
As a direct consequence of \Cref{tsbm:app:prop_leonard}, we obtain the following unnormalized Radon–Nikodym densities:
\begin{align*}
    &\textstyle\rmd \QbbVsg_{t:T|t}(x_{(t:T]}|x_t)\propto \exp\left(-\int_{t}^T \frac{V_u(x_u)}{\sigma^2} \rmd u\right) \rmd \Qbbsg_{t:T|t}(x_{(t:T]}|x_t) \eqsp , \\
    &\textstyle\rmd \QbbVsg_{0:t|t}(x_{[0:t)}|x_t)\propto \exp\left(-\int_{0}^t \frac{V_u(x_u)}{\sigma^2} \rmd u\right) \rmd \Qbbsg_{0:t|t}(x_{[0:t)}|x_t) \eqsp , \\
    &\textstyle\rmd \QbbVsg_{|0,T}(x_{(0:T)}|x_0, x_T)\propto \exp\left(-\int_{0}^T \frac{V_u(x_u)}{\sigma^2} \rmd u\right) \rmd \Qbbsg_{|0,T}(x_{(0:T)}|x_0, x_T) \eqsp .
\end{align*}

These identities suggest that the twisted measure $\QbbVsg$ and its bridge $\QbbVsg_{|0,T}$ should admit stochastic differential representations. More precisely, we now show that they are associated with affine-controlled Brownian motions. To this end, we first establish a general result from which the desired statements follow as corollaries, see \Cref{tsb:corollary:brownian,tsbm:cor:proj_pinned_BM}.

\begin{proposition}\label{tsbm:prop:reference_measure} Consider $\Qbbsg\in \mathcal{M}_T$, associated to an SDE $\rmd \bfX_t = f_t(\bfX_t)\rmd t + \sigma \rmd \bfB_t, \eqsp \bfX_0\sim \Qbbsg_0$, where $f\in \mathrm{C}^1([0,T] \times \rset^d, \rset^d)$ satisfies a linear growth condition, i.e., there exists $C>0$ such that for any $(t,x)\in [0,T]\times \rset^d$, we have $\norm{f_t(x)}\leq C(1 + \norm{x})$.
Let $V\in \rmC^{0,1}([0,T]\times \rset^d, \rset)$ be such that for any $t\in [0,T]$ and any $\Qbbsg_t$-almost surely $x\in \rset^d$, we have
\begin{align*}
    Z_V(t,x)& \textstyle{= \PE_{\Qbbsg_{|t}}\left[\exp\left(-\int_{t}^T \frac{V_s(\bfX_s)}{\sigma^2}\rmd s\right)|\bfX_t=x\right]} < \infty \eqsp .
\end{align*}

Then, the path measure $\QbbVsg$, defined in \eqref{tsbm:eq:q_v}, is Markov, associated with the affine-controlled SDE
\begin{align} \label{tsbm:eq:q_v_sde_general}
    \rmd \bfX^\bfu_t =\{f_t(\bfX^\bfu_t) +\sigma \bfu_t^V(\bfX^\bfu_t)\}\rmd t +  \sigma\rmd\bfB_t , \eqsp  \bfX^\bfu_0\sim \QbbVsg_0 \eqsp \text{where} \eqsp \bfu_t^V(x)= \sigma \nabla_x \log Z_V(t,x) .
\end{align}

Moreover, for any arbitrary continuously differentiable function $m_t: [t,T] \to \rset$ such that $m_t(t)=1$, the control $\bfu_t^V$ admits a closed-form solution depending on the potential $V$ via the \emph{path-wise reparameterization trick} \citep{domingo2023stochastic} as
\begin{align}
     \bfu_t^V(x)&=\frac{\sigma}{Z_V(t,x)}\PE_{\Qbbsg_{|t}}\Bigg[\left\{-\int_{t}^T m_t(s)\frac{\nabla V_s(\bfX_s)}{\sigma^2}\rmd s + \frac{1}{\sigma}\int_{t}^T \left(m_t(s) \nabla f_s(\bfX_s) -\partial_s m_t(s) \Idd \right) \rmd \bfB_s\right\}\notag\\
     & \quad \times \exp\left(-\int_{t}^T \frac{V_s(\bfX_s)}{\sigma^2} \rmd s\right)  \Bigg|\bfX_t=x\Bigg] \eqsp .\label{tsbm:eq:path_reparam}
\end{align}

\end{proposition}

\begin{proof}
Let $x_0\sim \QbbVsg_0$ and consider the following minimization problem over path measures:
\begin{align} \label{tsbm:eq:KL_Q_V}
    \argmin \ensembleLigne{\KL(\Mbb_{|0}(\cdot|x_0)\|\QbbVsg_{|0}(\cdot|x_0))}{\Mbb \in \mathcal{M}_{T,c}} \eqsp, 
\end{align}
where $\mathcal{M}_{T,c}$ is the set of Markov measures that are associated with controlled SDEs of the form
$$
\rmd \bfX^\bfu_t = f_t(\bfX^\bfu_t) \rmd t + \bfu_t(\bfX^\bfu_t)\rmd t + \sigma \rmd \bfB_t \eqsp, \eqsp \bfX^u_0=x_0 \eqsp ,
$$
where $\bfu\in \adminControlT$. Let $\Mbb^V$ denote the solution to \eqref{tsbm:eq:KL_Q_V}.

By Girsanov's formula, we have for any $\Mbb\in  \mathcal{M}_{T,c}$, the following KL decomposition
\begin{align*}
    \KL(\Mbb_{|0}(\cdot|x_0)\|\QbbVsg_{|0}(\cdot|x_0))= J(\bfu;x_0,0) + \log Z_V \eqsp ,
\end{align*}
where
\begin{align*}
    J(\bfu;x,t)= \frac{1}{\sigma^2}\PE\left[\int_t^T \left\{\frac{1}{2}\norm{\bfu_t(\bfX^\bfu_t)}^2 + V_t(\bfX^\bfu_t)\right\}\rmd t \mid \bfX^\bfu_t=x\right] \eqsp .
\end{align*}
Hence, solving \eqref{tsbm:eq:KL_Q_V} is equivalent to solving the following optimization problem
\begin{align}\label{tsbm:eq:soc_J}
    \argmin \ensembleLigne{J(\bfu;x_0,0)}{\bfu \in \adminControlT} \eqsp .
\end{align}
In fact, this is a SOC problem associated to the cost functional $(u,x,t) \mapsto J(u;x,t)$.
Combining \citep[Theorem 2.2]{nusken2021solving} and \citep[Remark 2.6]{nusken2021solving}, we deduce under the stated assumptions that the unique minimizer $\bfu^V$ satisfies
\begin{align*}
   \bfu_t^V(x)= \sigma \nabla_x \log Z_V(t,x) \quad \text{and} \quad J(\bfu^V; x_0,0)=-\log Z_V \eqsp .
\end{align*}
In particular, for any $x_0$, $\QbbVsg_0$-almost surely,
$$
\KL(\Mbb^V_{|0}(\cdot|x_0)\|\QbbVsg_{|0}(\cdot|x_0))=0 \eqsp ,
$$
which proves the result \eqref{tsbm:eq:q_v_sde_general}. The representation \eqref{tsbm:eq:path_reparam} follows directly from \citep[Proposition 1]{domingo2023stochastic}.
\end{proof}
Notably, \Cref{tsbm:prop:reference_measure} extends \citep[Theorem 5.4]{leonard2011stochastic}, where $\Qbbsg$ is assumed to be stationary, an assumption that does not hold for Brownian motion.

\begin{corollary}[Application to the Brownian motion]\label{tsb:corollary:brownian} If $\Qbbsg$ is the distribution of the scaled Brownian motion $(\sigma \bfB_t)_{t\in[0,T]}$, then $\QbbVsg$ is associated with
\begin{align}
    \rmd \bfX^\bfu_t= \bfu_t^V(\bfX_t) \rmd t + \sigma \rmd \bfB_t \eqsp , \eqsp  \bfX^\bfu_0\sim \QbbVsg_0 \eqsp ,
\end{align}
where
\begin{align}\label{tsbm:eq:drift_brownian_v}
    \textstyle \bfu_t^V(x)= \PE_{\QbbVsg_{|t}}\left[-\int_{t}^T \nabla V_s(\bfX_s)\rmd s|\bfX_t=x\right] \eqsp .
\end{align}
\end{corollary}
\begin{proof} This is an immediate application of \Cref{tsbm:prop:reference_measure} in the case where $f\equiv 0$ and $m_t\equiv 1$ for any $t\in [0,T]$.

\end{proof}

\begin{corollary}[Application to the \emph{pinned} Brownian motion] \label{tsbm:cor:proj_pinned_BM}Let $(x_0,x_T)\in \rset^d \times \rset^d$. If $\Qbbsg_{|0,T}(\cdot|x_0,x_T)$ is the distribution of the scaled Brownian motion $(\sigma \bfB_t)_{t\in [0,T]}$ \emph{pinned} down at initial and terminal points $(x_0,x_T)$, then $\QbbVsg_{|0,T}(\cdot|x_0,x_T)$ is equivalently associated with
the forward and backward SDEs
\begin{align}
    \rmd \bfX^\bfu_t&= \frac{x_T-\bfX^\bfu_t}{T-t}  + \bfu_t^{\rmf,V}(\bfX^\bfu_t|x_T) \rmd t + \sigma \rmd \bfB_t \eqsp, \bfX^\bfu_0=x_0 \eqsp ,\label{tsbm:eq:pinned_v_f}\\
    \rmd \bfY^\bfu_t&= \frac{x_0-\bfY^\bfu_t}{T-t}  + \bfu_{T-t}^{\rmb,V}(\bfY^\bfu_t|x_0) \rmd t + \sigma \rmd \bfB_t\eqsp, \bfY^\bfu_0=x_T \eqsp , \label{tsbm:eq:pinned_v_b}
\end{align}
where
\begin{align*}
    \textstyle \bfu_t^{\rmf,V}(x|x_T)&= \PE_{\QbbVsg_{|t,T}}\left[-\int_{t}^T \frac{T-s}{T-t}\nabla V_s(\bfX_s)\rmd s|\bfX_t=x, \bfX_T=x_T\right] \eqsp ,\\
    \textstyle \bfu_t^{\rmb,V}(x|x_0)&= \PE_{\QbbVsg_{|0,t}}\left[-\int_{0}^t \frac{s}{t}\nabla V_s(\bfX_s)\rmd s|\bfX_0=x_0,\bfX_t=x\right] \eqsp .
\end{align*}
\end{corollary}
\begin{proof}

Obtaining the forward SDE results from the application of \Cref{tsbm:prop:reference_measure} when considering the pinned Brownian bridge (in forward time direction) as base process $\Qbb^\sigma$, with
drift given in \eqref{tsbm:eq:brownian_bridge_f}, and by choosing $m_t(s)=\frac{T-s}{T-t}$, which simplifies the stochastic integral term in \eqref{tsbm:eq:path_reparam}.
To obtain the backward SDE, we apply \Cref{tsbm:prop:reference_measure} to the backward pinned Brownian bridge $(\Qbbsg)^R$, whose SDE is exactly given in \eqref{tsbm:eq:brownian_bridge_b},
considering as twisted path measure $\tilde{\Qbb}^{\sigma,V}(y_{[0:T]})\propto \exp(-\int_{0}^T V_{T-s}(y_s) \rmd s) (\Qbbsg)^R(y_{[0:T]})$ and taking once again $m_t(s)=\frac{T-s}{T-t}$.
\end{proof}

\subsection{Results on the reciprocal projection}\label{tsbm:app:reciprocal}

Below, we restate the result of \Cref{tsbm:prop:reciprical_cont}, additionally including the backward direction. 
\begin{proposition}[Reciprocal projection -- continuous time]\label{tsbm:prop:reciprical_cont_app}
Assume that the variational family $\{\Qbb^{\psi}_{|0,T}\}_{\psi \in \Psi}$ in \eqref{tsbm:eq:variational_cont} describes the pinned diffusion processes associated, for almost surely any $(x_0,x_T)\sim \Pbb_{0,T}$, to forward-time and backward-time parametric SDEs of the form
\begin{align}
\rmd \bfX^{0,T}_t &=  v_t^{\rmf,\psi} (\bfX^{0,T}_t|x_0, x_T) \rmd t + \sigma \rmd \bfB_t\eqsp , \eqsp \bfX^{0,T}_0=x_0, \eqsp \bfX^{0,T}_T=x_T \eqsp , \label{tsbm:eq:var_path_f}\\
\rmd \bfY^{0,T}_t &=  v_{T-t}^{\rmb,\psi} (\bfY^{0,T}_t|x_0, x_T) \rmd t + \sigma \rmd \bfB_t\eqsp , \eqsp \bfY^{0,T}_0=x_T, \eqsp \bfY^{0,T}_T=x_0 \eqsp , \label{tsbm:eq:var_path_b}
\end{align}
with $v^{\rmf,\psi},v^{\rmb,\psi} \in \adminControlT$. Then, the minimizer of \eqref{tsbm:eq:variational_cont} for $\QbbVsg$ defined by \eqref{tsbm:eq:q_v}, denoted by $\psi_\star$, is equivalently the minimizer of the two following losses
\begin{align}
  \mathcal{L}_{\mathcal{R}}^{\rmf}(\psi) &= \textstyle\int_{0}^T\PE_{\Pi^{\psi}_{0,t,T}}\left[ \frac{1}{2}\|(\bfX_T -\bfX_t)/(T-t)-v_t^{\rmf,\psi}(\bfX_t|\bfX_0, \bfX_T)\|^2 + V_t(\bfX_t)\right]\rmd t \eqsp ,\label{tsbm:eq:objective_reciprocal_f_app}\\
  \mathcal{L}_{\mathcal{R}}^{\rmb}(\psi) &= \textstyle\int_{0}^T\PE_{\Pi^{\psi}_{0,t,T}}\left[ \frac{1}{2}\|(\bfX_0 -\bfX_t)/t-v_t^{\rmb,\psi}(\bfX_t|\bfX_0, \bfX_T)\|^2 + V_t(\bfX_t)\right]\rmd t \eqsp ,\label{tsbm:eq:objective_reciprocal_b_app}
\end{align}
where $\Pi^{\psi}= \Pbb_{0,T}\Qbb^{\psi}_{|0,T}$.
\end{proposition}

\begin{proof}[Proof of \Cref{tsbm:prop:reciprical_cont_app}]Consider such variational family. We first prove the result in the forward direction. For any $\psi \in \Psi$, and any $(x_0,x_T)\in \rset^d\times \rset^d$, we have, $\Pbb_{0,T}$-almost surely,
\begin{align*}
&\KL(\Qbb^{\psi}_{|0,T}(\cdot|x_0,x_T)\|\QbbVsg_{|0,T}(\cdot|x_0,x_T))\\
&\textstyle = \KL(\Qbb^{\psi}_{|0,T}(\cdot|x_0,x_T)\|\Qbbsg_{|0,T}(\cdot|x_0,x_T)) + \int_0^T \frac{1}{\sigma^2}\PE_{\Qbb^{\psi}_{t|0,T}}[V_t(\bfX_t)|\bfX_0=x_0, \bfX_T=x_T]\rmd t + \mathbf{c}(x_0,x_T) \\
&\textstyle= \frac{1}{\sigma^2}\int_{0}^T\PE_{\Qbb^{\psi}_{t|0,T}}\left[ \frac{1}{2}\|v_t^{\rmf, \psi}(\bfX_t|\bfX_0, \bfX_T)- \sigma^2 \nabla \log \Qbbsg_{T|t}(\bfX_T|\bfX_t)\|^2 + V_t(\bfX_t)|\bfX_0=x_0, \bfX_T=x_T\right]\rmd t \\
& \quad + \mathbf{c}(x_0,x_T)
\end{align*}
where $\mathbf{c}(x_0,x_T)$ is a constant independent of $\psi$ (namely, the normalizing constant of $\QbbVsg_{|0,T}(\cdot|x_0,x_T)$), and the second equality comes from Girsanov's theorem application exploiting \eqref{tsbm:eq:brownian_bridge_f}.
Therefore, by taking the expectation with respect to $\Pbb_{0,T}$, we obtain the forward objective identity
$$
\PE_{\Pbb_{0,T}}[\KL(\Qbb^{\psi}_{|0,T}\|\QbbVsg_{|0,T})] = \frac{1}{\sigma^2}\mathcal{L}_{\mathcal{R}}^\rmf(\psi) + \PE_{\Pbb_{0,T}}[\mathbf{c}(\bfX_0,\bfX_T)] \eqsp .
$$
Similarly, by exploiting \eqref{tsbm:eq:brownian_bridge_b}, we can obtain the backward objective identity
$$
\PE_{\Pbb_{0,T}}[\KL(\Qbb^{\psi}_{|0,T}\|\QbbVsg_{|0,T})] = \frac{1}{\sigma^2}\mathcal{L}_{\mathcal{R}}^\rmb(\psi) + \PE_{\Pbb_{0,T}}[\mathbf{c}(\bfX_0,\bfX_T)] \eqsp .
$$
The result of the proposition immediately follows these observations.
\end{proof}

We obtain a similar result for the discrete time case, which can be proved in the exact same way.

\begin{proposition}[Reciprocal projection -- discrete time]\label{tsbm:app:prop_reciprocal_discrete}
Assume that the variational family $\{\Qbb^{\psi}_{|0,T}\}_{\psi \in \Psi}$ in \eqref{tsbm:eq:variational_cont} describes the pinned diffusion processes associated, for almost surely any $(x_0,x_T)\sim \Pbb_{0,T}$, to forward-time and backward-time parametric SDEs of the form
\begin{align*}
\rmd \bfX^{0,T}_t &=  v_t^{\rmf,\psi} (\bfX^{0,T}_t|x_0, x_T) \rmd t + \sigma \rmd \bfB_t\eqsp , \eqsp \bfX^{0,T}_0=x_0, \eqsp \bfX^{0,T}_T=x_T \eqsp, \\
\rmd \bfY^{0,T}_t &=  v_{T-t}^{\rmb,\psi} (\bfY^{0,T}_t|x_0, x_T) \rmd t + \sigma \rmd \bfB_t\eqsp , \eqsp \bfY^{0,T}_0=x_T, \eqsp \bfY^{0,T}_T=x_0 \eqsp,
\end{align*}
with $v^{\rmf,\psi},v^{\rmb,\psi} \in \adminControlT$. Then, the minimizer of \eqref{tsbm:eq:variational_cont} for $\QbbVsg$ defined by \eqref{tsbm:eq:q_v_discret}, denoted by $\psi_\star$, is equivalently the minimizer of the two following losses
\begin{align*}
  \mathcal{L}_{\mathcal{R}}^{\rmf}(\psi) &= \textstyle\int_{0}^T\PE_{\Pi^{\psi}_{0,t,T}}\left[ \frac{1}{2}\|(\bfX_T -\bfX_t)/(T-t)-v_t^{\rmf,\psi}(\bfX_t|\bfX_0, \bfX_T)\|^2\right]\rmd t + \sum_{k=1}^K \PE_{\Pi^{\psi}_{t_k}}[V_k(\bfX_{t_k})] \eqsp ,\\
  \mathcal{L}_{\mathcal{R}}^{\rmb}(\psi) &= \textstyle\int_{0}^T\PE_{\Pi^{\psi}_{0,t,T}}\left[ \frac{1}{2}\|(\bfX_0 -\bfX_t)/t-v_t^{\rmb,\psi}(\bfX_t|\bfX_0, \bfX_T)\|^2 \right]\rmd t + \sum_{k=1}^K \PE_{\Pi^{\psi}_{t_k}}[V_k(\bfX_{t_k})]  \eqsp ,
\end{align*}
where $\Pi^{\psi}= \Pbb_{0,T}\Qbb^{\psi}_{|0,T}$.
\end{proposition}

\subsection{Results on the Markovian projection (continuous time)} \label{tsbm:app:markov_continu}

\begin{proposition}[Extended version of \Cref{tsbm:prop:tsb_markovian_loss}] \label{tsbm:prop:markov_theory} Let $\Pi \in \mathcal{R}(\QbbVsg)$. Then,
$\projM(\Pi)$ is equivalently induced by the diffusion processes
$(\bfXf_t)_{t\in [0,T]}$ and $(\bfXb_{T-t})_{t\in [0,T]}$, which are time-reversed of each other and respectively defined by the SDEs
\begin{align}
    \rmd \bfXf_t &= \mathrm{b}_t^{\rmf, V}(\bfXf_t)\rmd t + \sigma\rmd\bfB_t\eqsp, \quad \bfXf_0 \sim \Pi_0 \eqsp ,\label{tsbm:eq:v_f_star_sde}\\
    \rmd \bfXb_t &= \mathrm{b}_{T-t}^{\rmb, V}(\bfXb_t)\rmd t + \sigma\rmd\bfB_t\eqsp, \quad \bfXb_0 \sim \Pi_T \eqsp ,
\end{align}
where
\begin{align}
     \mathrm{b}_t^{\rmf, V}(x_t)&=\textstyle\mathbb{E}_{\Pi_{|t}}[(\bfX_T-\bfX_t)/(T-t)-\int_t^T \frac{T-s}{T-t}\nabla V_s(\bfX_s) \rmd s|\bfX_t=x_t ] \eqsp ,\label{tsbm:eq:v_f_star}\\
    \mathrm{b}_t^{\rmb, V}(x_t)&=\textstyle\mathbb{E}_{\Pi_{|t}}[(\bfX_0-\bfX_t)/t-\int_0^t \frac{s}{t}\nabla V_s(\bfX_s) \rmd s|\bfX_t=x_t ] \eqsp .\label{tsbm:eq:v_b_star}
\end{align}
In other words, the optimal drifts $\mathrm{b}^{\rmf, V}$ and $\mathrm{b}^{\rmb, V}$ write as conditional expectations of the forward and backward drifts driving the pinned process $\QbbVsg_{|0,T}$, given in \Cref{tsbm:cor:proj_pinned_BM}, as follows  
\begin{align*}
    \mathrm{b}_t^{\rmf, V}(x_t)&= \textstyle\mathbb{E}_{\Pi_{T|t}}\left[ (\bfX_T-\bfX_t)/(T-t)  + \bfu_t^{\rmf,V}(\bfX_t|\bfX_T)|\bfX_t=x_t \right] \eqsp , && \text{derived from \eqref{tsbm:eq:pinned_v_f}}\eqsp , \\
    \mathrm{b}_t^{\rmb, V}(x_t)&= \textstyle\mathbb{E}_{\Pi_{0|t}}\left[(\bfX_0-\bfX_t)/t  + \bfu_t^{\rmb,V}(\bfX_t|\bfX_0)|\bfX_t=x_t \right]  \eqsp ,&& \text{derived from \eqref{tsbm:eq:pinned_v_b}}\eqsp .
\end{align*}

 Therefore, for any $t\in [0,T]$, we have $\mathrm{Law}(\bfXf_t)=\mathrm{Law}(\bfXb_{T-t})=\Pi_t$. Additionally, $\mathrm{b}^{\rmf, V}$ and $\mathrm{b}^{\rmb, V}$ can be approximated by $v^{\theta^{\rmf}_\star}$ and $v^{\theta^{\rmb}_\star}$, where $\theta^{\rmf}_\star$ and $\theta^{\rmb}_\star$ are respectively the minimizers of forward objective $\mathcal{L}^\rmf_{\mathcal{M}_T}$, given in \eqref{tsb:eq:obj_f},
and backward objective $\mathcal{L}_{\mathcal{M}_T}^\rmb$, given in \eqref{tsb:eq:obj_b}.
\end{proposition}

\begin{proof}
    The structure of the proof closely follows that of \cite[Proposition 2]{shi2023diffusion} and relies on analogous technical assumptions, which we omit here for readability. We detail below the argument in the forward setting; the backward case follows by entirely similar arguments.

Since $\Pi \in \mathcal{R}(\QbbVsg)$, we have $\Pi_{|0} = \phi_{T|0} \QbbVsg_{|0}$, where $\phi_{T|0}=\frac{\rmd \Pi_{T|0}}{\rmd \QbbVsg_{T|0}}$.
In particular, for any almost surely $x_0\sim \Pi_0$, $\Pi_{|0}(\cdot|x_0)$ is obtained from $\QbbVsg_{|0}(\cdot|x_0)$ through a Doob's $h$-transform, and is therefore associated with the SDE
\begin{align*}
    \rmd \bfX^\bfu_t &= \bfu_t^V(\bfX^\bfu_t)\rmd t + \sigma^2 \nabla \log \phi_{t|0}(\bfX^\bfu_t|\bfX^\bfu_0)\rmd t + \sigma\rmd\bfB_t \eqsp, \eqsp \bfX^\bfu_0=x_0 \eqsp ,
\end{align*}
where $\bfu^V$ denotes the drift of the SDE associated with $\QbbVsg$, see \eqref{tsbm:eq:drift_brownian_v}, and the potential $\phi_{t|0}$ is given by
\begin{align*}
    \phi_{t|0}(x_t|x_0)= \mathbb{E}_{\QbbVsg_{T|t}}[\phi_{T|0}(\bfX_T|x_0)\mid\bfX_t = x_t] \eqsp .
\end{align*}
Observing that $\phi_{t|0}(x_t|x_0)= \frac{\rmd \Pi_{t|0}}{\rmd \QbbVsg_{t|0}}(x_t|x_0)$, we may apply Tweedie’s formula to obtain
\begin{align*}
    \nabla \log \phi_{t|0}(x_t|x_0)=\mathbb{E}_{\Pi_{T|0,t}}[\nabla_{\bfX_t} \log \QbbVsg_{T|t}(\bfX_T|\bfX_t)\mid\bfX_t=x_t, \bfX_0=x_0] \eqsp .
\end{align*}

We now compute the score $\nabla \log \QbbVsg_{T|t}$. By definition,
\begin{align*}
    & \rmd \QbbVsg_{T|t}(x_T|x_t) = \frac{1}{Z_t(x_t)}\overbrace{\int \exp\left(-\int_t^T \frac{V_s(x_s)}{\sigma^2} \rmd s\right) \rmd \Qbbsg(x_{(t,T)}|x_t, x_T)}^{F(x_t, x_T)} \rmd \Qbbsg_{T|t}(x_T|x_t) \\
    &\text{where } Z_t(x_t)= \int \exp\left(-\int_t^T \frac{V_s(x'_s)}{\sigma^2} \rmd s\right) \rmd \Qbbsg(x'_{(t,T]}|x_t) \eqsp .
\end{align*}
Therefore,
\begin{align*}
    \nabla_{x_t} \log \QbbVsg_{T|t}(x_T|x_t)= -\nabla \log Z_t(x_t) + \nabla_{x_t} \log \Qbbsg_{T|t}(x_T|x_t) + \nabla_{x_t} \log F(x_t, x_T) \eqsp .
\end{align*}

To compute $\nabla_{x_t} \log F(x_t, x_T)$, we use the Brownian bridge reparameterization inside $F(x_t,x_T)$. This yields
\begin{align*}
    F(x_t, x_T) = \int_0^{T-t}\exp\left(-\int_t^T \frac{1}{\sigma^2} V_s\left(\frac{s-t}{T-t}x_T + \frac{T-s}{T-t}x_t + \sigma \{\bfB_{s-t} - \frac{s-t}{T-t}\bfB_{T-t}\}\right)\right) \rmd \bfB_u \eqsp .
\end{align*}
Differentiating under the integral sign, we obtain
\begin{align*}
    \nabla_{x_t} \log F(x_t, x_T)
    = \frac{1}{\sigma^2}\mathbb{E}_{\QbbVsg_{|t,T}}\left[-\int_t^T \frac{T-s}{T-t} \nabla V_s (\bfX_s) \rmd s \mid \bfX_t=x_t, \bfX_T=x_T\right] \eqsp .
\end{align*}
Using the same reparameterization argument for $Z_t(x_t)$, we further obtain
\begin{align*}
    \nabla_{x_t}\log Z_t(x_t)
    = \frac{1}{\sigma^2}\mathbb{E}_{\QbbVsg_{|t}}\left[-\int_t^T \nabla V_s (\bfX_s) \rmd s \mid \bfX_t=x_t\right]
    = \frac{\bfu_t^V(x_t)}{\sigma^2} \eqsp .
\end{align*}

Combining the above identities, we conclude that $\Pi_{|0}(\cdot|x_0)$ is associated with the SDE
\begin{align*}
    \rmd \bfX_t
    &= \underbrace{\mathbb{E}_{\Pi_{|0,t}}\left[\sigma^2 \nabla_{x_t} \log \Qbbsg_{T|t}(\bfX_T|\bfX_t)
    - \int_t^T \frac{T-s}{T-t} \nabla V_s (\bfX_s) \rmd s \mid \bfX_t, \bfX_0 \right]}_{\tilde{\mathrm{b}}_t^{\rmf, V}(\bfX_t, \bfX_0)}\rmd t + \sigma\rmd\bfB_t  \eqsp , \eqsp \bfX_0=x_0 \eqsp .
\end{align*}

Let $b \in \adminControlT$, and denote by $\mathbb{M}^b$ the path measure associated with
$$
\rmd \bfX_t = b_t(\bfX_t)\rmd t + \sigma\rmd\bfB_t \eqsp,
\quad \bfX_0 \sim \Pi_0 \eqsp .
$$
Since $\mathbb{M}^b_0=\Pi_0$, the Markovian projection objective writes by Girsanov’s theorem as
\begin{align*}
    \KL(\Pi\|\mathbb{M})= \frac{1}{2\sigma^2}\int_0^T
    \mathbb{E}_{\Pi_{0,t}}\!\left[\norm{\tilde{\mathrm{b}}_t^{\rmf, V}(\bfX_t, \bfX_0) - b_t(\bfX_t)}^2\right] \rmd t \eqsp .
\end{align*}
This expression is minimized when $b_t=\mathrm{b}_t^{\rmf, V}$, defined by \eqref{tsbm:eq:v_f_star}, proving that $\projM(\Pi)$ is associated with \eqref{tsbm:eq:v_f_star_sde}. Then, the marginal-preserving property of the Markovian projection of $\Pi$ follows directly from two facts: the optimal drift is given by conditional expectations of the conditional drifts of $\QbbVsg_{\mid 0,T}$ under $\Pi_{0,T\mid t}$, and $\Pi$ satisfies the reciprocal property. Combining these observations, one can show, by applying \cite[Theorem 2]{peluchetti2023non}, that $\Pi$ and $\projM(\Pi)$ solve the same Fokker--Planck equation and therefore share the same time marginals.

To obtain a computable formulation of the Markovian projection step, we replace the original Markovian projection step \eqref{tsbm:eq:markov_projection} by the variational problem
\begin{align*}
    \argmin\ensembleLigne{\KL(\projM(\Pi)\|\mathbb{M}_{\theta^\rmf})}{\theta^\rmf \in \Theta} \eqsp ,
\end{align*}
where $\mathbb{M}_{\theta^\rmf}$ is associated with
$$
\rmd \bfX_t = v_t^{\theta^\rmf}(\bfX_t)\rmd t + \sigma\rmd\bfB_t \eqsp,
\quad \bfX_0 \sim \Pi_0 \eqsp .
$$
For any $\theta^\rmf \in \Theta$, by Girsanov's theorem, we have
\begin{align}\label{tsbm:eq:computation_loss_f}
    \KL(\projM(\Pi)\|\mathbb{M}_{\theta^\rmf})
    &= \frac{1}{2\sigma^2}\int_0^T \mathbb{E}_{\projM(\Pi)_{t}}\!\left[\norm{\mathrm{b}_t^{\rmf, V}(\bfX_t)- v_t^{\theta^\rmf}(\bfX_t)}^2\right] \rmd t \\
    &= \frac{1}{2\sigma^2}\int_0^T \mathbb{E}_{\Pi_{t}}\!\left[\norm{\mathrm{b}_t^{\rmf, V}(\bfX_t)- v_t^{\theta^\rmf}(\bfX_t)}^2\right] \rmd t \notag\\
    &= \frac{1}{2\sigma^2}\{\mathcal{L}_{\mathcal{M}_T}^\rmf(\theta^\rmf) - \mathrm{c}^\rmf \}\eqsp ,\notag
\end{align}
where
\begin{align}
    \textstyle\mathrm{c}^\rmf = \int_0^T (T-t)^{-1}\int_t^T \PE_{\Pi_{t}}\Big[&\PE_{\Pi_{s,T|t}}\left[\norm{ \mathbf{b}_{t,s}^{\rmf, V}(\bfX_t, \bfX_s, \bfX_T)}^2\right] \notag\\
    & \textstyle- \norm{(T-t)^{-1} \int_t^T \PE_{\Pi_{u,T|t}}[\mathbf{b}_{t,u}^{\rmf, V}(\bfX_t, \bfX_u, \bfX_T)] \rmd u }^2\Big] \eqsp \rmd s\rmd t \eqsp , \notag
\end{align}
with $\mathbf{b}_{t,s}^{\rmf, V}(x_t,x_s,x_T)=(x_T-x_t)/(T-t)-(T-s)\nabla V_s(x_s) $. Since $\mathrm{c}^{\rmf}$ does not depend on $\theta^{\rmf}$, minimizing 
$\mathcal{L}_{\mathcal{M}_T}^{\rmf}(\theta^{\rmf})$ with respect to $\theta^{\rmf}$ is equivalent to minimizing 
$\KL(\projM(\Pi)\|\mathbb{M}_{\theta^{\rmf}})$. This yields a variational approximation of the true Markovian projection.
\end{proof}

\paragraph{Derivations with learnable control variates.}
We first provide the proof of \Cref{tsbm:prop:cv_main}.

\begin{proof}[Proof of \Cref{tsbm:prop:cv_main}]

We only prove the forward identity, the backward one being obtained similarly by time reversal. Let $\alpha:[0,T]^2\to\rset$ be continuously differentiable with respect to its second variable.
We assume throughout that the quantities involved are sufficiently integrable to justify the use of It\^o's formula, conditional expectations, Fubini's theorem, and the martingale property of the stochastic integral. 
In particular, for every fixed \(t\in[0,T)\), we assume
\begin{align*}
    & \textstyle \int_t^T |\alpha(t,s)|^2 \rmd s < \infty \eqsp ,\\
& \textstyle\PE_{\Pi_{|t}}\left[\int_t^T |\alpha(t,s)|\,\|\bfu_s^{\rmf,V}(\bfX_s|\bfX_T)\| \rmd s
+ \int_t^T \int_s^T |\alpha(t,s)|\frac{T-r}{T-s}\|\nabla V_r(\bfX_r)\| \rmd r \rmd s
\,\middle|\, \bfX_t\right] < \infty\eqsp .
\end{align*}

Fix $t\in [0,T)$, and define the following notation
\begin{align*}
    \alpha_s=\alpha(t,s), \qquad \tilde{\alpha}_s=(T-s)\alpha_s, \qquad s\in[t,T] \eqsp .
\end{align*}

Since $\Pi\in\mathcal{R}(\QbbVsg)$, for almost surely any $(x_t, x_T)\sim \Pi_{t,T}$, we know from \Cref{tsbm:cor:proj_pinned_BM} that the pinned path measure $\Pi_{(t,T)|t,T}$ is Markov, induced by the pinned SDE defined on $[t,T]$ by
$$\rmd \bfX_s = \left\{ \frac{x_T-\bfX_s}{T-s} + \bfu_s^{\rmf,V}(\bfX_s|x_T) \right\}\rmd s + \sigma \rmd \bfB_s \eqsp , \eqsp \bfX_t= x_t \eqsp,\eqsp \bfX_s= x_T  \eqsp .$$
We introduce the auxiliary stochastic process
$$ G_s=\alpha_s(x_T-\bfX_s), \qquad s\in[t,T].$$
By It\^o's formula, we obtain
$$\rmd G_s = \partial_s\alpha_s(x_T-\bfX_s)\rmd s - \alpha_s \rmd \bfX_s, $$
hence
$$ \rmd G_s=  \left( \partial_s\alpha_s-\frac{\alpha_s}{T-s} \right)(x_T-\bfX_s)\rmd s -\alpha_s \bfu_s^{\rmf,V}(\bfX_s|x_T)\rmd s -\sigma \alpha_s \rmd \bfB_s\eqsp.$$
Since
$$ \frac{\partial_s\tilde{\alpha}_s}{T-s} = \partial_s\alpha_s-\frac{\alpha_s}{T-s}\eqsp, $$
we have
$$ \rmd G_s =  \frac{\partial_s\tilde{\alpha}_s}{T-s}(x_T-\bfX_s)\rmd s -\alpha_s \bfu_s^{\rmf,V}(\bfX_s|x_T)\rmd s -\sigma \alpha_s \rmd \bfB_s \eqsp. $$

Integrating from $t$ to $T$, and using $G_T=\alpha_T(x_T-x_T)=0$, we obtain
$$ \alpha_t(x_T-x_t) + \int_t^T \frac{\partial_s\tilde{\alpha}_s}{T-s}(x_T-\bfX_s)\rmd s = \int_t^T \alpha_s\bfu_s^{\rmf,V}(\bfX_s|x_T)\rmd s + \sigma\int_t^T \alpha_s\rmd \bfB_s \eqsp . $$

We next take the conditional expectation of the expressions above with respect to $\Pi_{(t,T]|t}$. In particular,  the stochastic term vanishes and we obtain
$$ \textstyle\PE_{\Pi_{|t}}\left[ \alpha_t(\bfX_T-\bfX_t)
+ \int_t^T \frac{\partial_s\tilde{\alpha}_s}{T-s}(\bfX_T-\bfX_s)\rmd s \;\middle|\; \bfX_t=x_t \right]
= \PE_{\Pi_{|t}}\left[ \int_t^T \alpha_s\bfu_s^{\rmf,V}(\bfX_s|\bfX_T)\rmd s \;\middle|\; \bfX_t=x_t \right].$$

Using the definition of \(\bfu_s^{\rmf,V}\), the tower property, and Fubini's theorem, the right-hand side becomes
\begin{align*}
&\PE_{\Pi_{|t}}\left[ \int_t^T \alpha_s\bfu_s^{\rmf,V}(\bfX_s|\bfX_T)\rmd s \;\middle|\; \bfX_t=x_t \right] \\
& = \int_t^T \alpha_s\, \PE_{\Pi_{|t}}\left[ \PE_{\Pi_{|s,T}}\left[\left.-\int_s^T \frac{T-r}{T-s}\nabla V_r(\bfX_r)\rmd r \,\right|\, \bfX_s,\bfX_T\right] \;\middle|\; \bfX_t=x_t \right]\rmd s \\
& = -\PE_{\Pi_{|t}}\left[ \int_t^T \alpha_s
\int_s^T \frac{T-r}{T-s}\nabla V_r(\bfX_r)\rmd r \,\rmd s
\;\middle|\; \bfX_t=x_t \right] \\
& = -\PE_{\Pi_{|t}}\left[ \int_t^T \left( \int_t^s \frac{\alpha_u}{T-u}\rmd u \right) (T-s)\nabla V_s(\bfX_s)\rmd s \;\middle|\; \bfX_t=x_t \right].
\end{align*}
Thus,
\begin{align*}
\PE_{\Pi_{|t}}\Bigg[ &\alpha(t,t)(\bfX_T-\bfX_t)
+\int_t^T \frac{\partial_s\tilde{\alpha}(t,s)}{T-s}(\bfX_T-\bfX_s)\rmd s \notag\\
&\qquad\qquad +\int_t^T \left(
\int_t^s \frac{\alpha(t,u)}{T-u}\rmd u \right)
(T-s)\nabla V_s(\bfX_s)\rmd s \; | \; \bfX_t \Bigg] = \PE_{\Pi_{|t}}[H^{\alpha,\rmf}_{t}|\bfX_t] =0.
\end{align*}
\end{proof}

\paragraph{Informal explanation of \Cref{tsbm:prop:cv_main}.}
At first sight, the quantities $H_t^{\alpha,\rmf}$ and $H_t^{\alpha,\rmb}$ introduced in \Cref{tsbm:prop:cv_main} may appear somewhat ad hoc. We explain here, in the forward case, why they arise naturally as \emph{score-based} quantities.

Recall from \eqref{tsbm:eq:v_f_star} that the drift of the forward Markovian projection, $\mathrm{b}_t^{\rmf,V}$, is written as a conditional expectation over the future segment $(X_s)_{s\in(t,T]} \sim \Pi_{(t,T]\mid t}$. Consider a time discretization $(s_k)_{k=0}^{K+1}$ of $[t,T]$, with $s_0=t$, $s_{K+1}=T$, and constant step size $\delta>0$. A discrete-time approximation of $\Pi_{(t,T]\mid t}$ is obtained through its finite-dimensional projection: 
\begin{align*}
    \textstyle \tilde{\Pi}_{s_1,\ldots,s_{K+1}\mid t}(x_{s_1:s_{K+1}}\mid x_t)
    =
    \int_{\rset^d }
    \QbbVsg_{s_1,\ldots,s_K\mid t,T}(x_{s_1:s_K}\mid x_t,x_T)
    \rmd \Pi_{T\mid t}(x_T\mid x_t) \eqsp .
\end{align*}
\newpage
By the Markov property of $\QbbVsg_{|0,T}$, this conditional path measure factorizes as
\begin{align*}
    \textstyle\QbbVsg_{s_1,\ldots,s_K\mid t,T}(x_{s_1:s_K}\mid x_t,x_T)
    =
    \prod_{k=1}^K
    \QbbVsg_{s_k\mid t,s_{k+1}}(x_{s_k}\mid x_t,x_{s_{k+1}}) \eqsp .
\end{align*}
Consequently, for each $k\in\{1,\ldots,K\}$, the usual zero-score identity gives
\begin{align*}
    \textstyle\int_{\rset^d }
    \nabla \log \QbbVsg_{s_k\mid t,s_{k+1}}(x_{s_k}\mid x_t,x_{s_{k+1}})
    \,
    \rmd \QbbVsg_{s_k\mid t,s_{k+1}}(x_{s_k}\mid x_t,x_{s_{k+1}})
    =
    0 \eqsp .
\end{align*}
Combining this identity with the definition of $\tilde{\Pi}$ yields
\begin{align*}
    \textstyle\PE_{\tilde{\Pi}_{\mid t}}
    \left[
        \sum_{k=1}^K
        \delta \sigma^2 \alpha(t,s_k)
        \nabla \log
        \QbbVsg_{s_k\mid t,s_{k+1}}
        (\bfX_{s_k}\mid \bfX_t,\bfX_{s_{k+1}})
        \,\middle|\, \bfX_t
    \right]
    =
    0 \eqsp .
\end{align*}
In the continuous-time limit $\delta\to 0$, this identity becomes precisely the forward identity in \Cref{tsbm:prop:cv_main}, thereby providing a score-based interpretation of the proposed quantities.

Relying on \Cref{tsbm:prop:cv_main}, we provide in \Cref{tsbm:prop:markov_theory_cv} strictly equivalent expressions for the optimal control drifts derived in \Cref{tsbm:prop:tsb_markovian_loss} based on forward and backward functions $\alpha^\rmf$ and $\alpha^\rmb$, formally acting as control variates.

\begin{proposition}[Markovian projection with control variates -- continuous time] \label{tsbm:prop:markov_theory_cv} Let $\Pi \in \mathcal{R}(\QbbVsg)$, $\alpha^\rmf$, $\alpha^\rmb: \ccint{0,T}^2 \to \rset$, continuously differentiable in their second variable. Define the velocity fields
\begin{align}
     \mathrm{b}^{\alpha^\rmf,V}_t(x_t)&=\textstyle\mathbb{E}_{\Pi_{|t}}\Big[
     \frac{\bfX_T-\bfX_t}{T-t}-\alpha^\rmf(t,t)[\bfX_T-\bfX_t] -\int_{t}^T \frac{\partial_s \tilde{\alpha}^\rmf(t,s)}{T-s}[\bfX_T-\bfX_s] \rmd s \label{tsbm:eq:v_f_star_cv}\\
     &\textstyle-\int_t^T \left\{\frac{1}{T-t}+\int_{t}^s \frac{\alpha^\rmf(t,u)}{T-u} \rmd u\right\}(T-s) \nabla V_s(\bfX_s) \rmd s|\bfX_t=x_t \Big] \eqsp ,\notag\\
     \mathrm{b}^{\alpha^\rmb,V}_t(x_t)&=\textstyle\mathbb{E}_{\Pi_{|t}}\Big[
     \frac{\bfX_0-\bfX_t}{t}-\alpha^\rmb(t,t)[\bfX_0-\bfX_t] +\int_{0}^t \frac{\partial_s \tilde{\alpha}^{\rmb}(t,s)}{s}[\bfX_0-\bfX_s] \rmd s \label{tsbm:eq:v_b_star_cv}\\
     &\textstyle-\int_0^t \left\{\frac{1}{t}-\int_{s}^t \frac{\alpha^{\rmb}(t,u)}{u} \rmd u\right\}s \nabla V_s(\bfX_s) \rmd s|\bfX_t=x_t \Big] \eqsp ,\notag
\end{align}
where $\tilde{\alpha}^{\rmb}(t,s)=s\alpha^\rmb(t,s)$ and  $\tilde{\alpha}^{\rmf}(t,s)=(T-s)\alpha^\rmf(t,s)$.  Then, we have the equality $\mathrm{b}^{\alpha^\rmf,V}=\mathrm{b}^{\rmf,V}$ and $\mathrm{b}^{\alpha^\rmb,V}=\mathrm{b}^{\rmb,V}$, where
the drifts $\mathrm{b}^{\rmf,V}$ and $\mathrm{b}^{\rmb,V}$ are respectively defined in \eqref{tsbm:eq:v_f_star} and \eqref{tsbm:eq:v_b_star}.
\end{proposition}

\begin{proof} This result derives from \Cref{tsbm:prop:cv_main} and linearity of the expectation, observing that
\begin{align}\notag
    \mathrm{b}^{\alpha^\rmf,V}_t(x_t)
    =
    \mathrm{b}^{\rmf,V}_t(x_t)
    -
    \PE_{\Pi_{|t}}
    [
    H^{\alpha^\rmf,\rmf}_{t}
    \mid \bfX_t=x_t
    ] \eqsp, \eqsp
    \mathrm{b}^{\alpha^\rmb,V}_t(x_t)
    =
    \mathrm{b}^{\rmb,V}_t(x_t)
    -
    \PE_{\Pi_{|t}}
    [
    H^{\alpha^\rmb,\rmb}_{t}
    \mid \bfX_t=x_t
    ].
\end{align}
\end{proof}

Since the obtained $\alpha$-based optimal drifts write as conditional expectations, we propose to learn them via parameterized families via a joint minimization loss that also includes the scalar coefficients $\alpha^\rmf$ and $\alpha^\rmb$. Interestingly, we show these joint losses follow a bias-variance decomposition, where the bias term is the original TSBM loss, while the variance term is minimized with respect to $\alpha^\rmf$ or $\alpha^\rmb$. This explains the interest of using learnable control variates in our setting.

\begin{proposition}[TSBM bridge matching losses with control variates -- continuous time]\label{tsbm:prop:tsb_markovian_loss_cv}
Let $\Pi\in \mathcal{R}(\QbbVsg)$. Consider the respective forward and backward objectives
\begin{align}
     &\mathcal{L}_{\mathcal{M}_T}^{\rmf}(\theta^{\rmf}, \alpha^\rmf) =\int_0^T (T-t)^{-1}\int_t^T \PE_{\Pi_{t,s,T}}\left[\norm{ \mathbf{b}_{t,s}^{\alpha^\rmf, V}(\bfX_t, \bfX_s, \bfX_T) -v_t^{\theta^{\rmf}}(\bfX_t)}^2\right] \eqsp \rmd s\rmd t \label{tsb:eq:obj_f_cv} \eqsp ,\\
     &\mathcal{L}_{\mathcal{M}_T}^{\rmb}(\theta^{\rmb}, \alpha^\rmb)=\int_0^T t^{-1}\int_0^t \PE_{\Pi_{0,s,t}}\left[\norm{\mathbf{b}_{t,s}^{\alpha^\rmb, V}(\bfX_0, \bfX_s, \bfX_t)-v_t^{\theta^{\rmb}}(\bfX_t)}^2\right] \eqsp \rmd s\rmd t
     \eqsp ,\label{tsb:eq:obj_b_cv}
\end{align}
where
\begin{align}
   \mathbf{b}_{t,s}^{\alpha^\rmf,V}(x_t, x_s, x_T)&= (x_T-x_t)/(T-t)-\alpha^{\rmf}(t,t)[x_T-x_t] -\frac{\partial_s \tilde{\alpha}^{ \rmf}(t,s)}{T-s}(T-t)[x_T-x_s]\label{tsb:eq:target_f_cv} \\
    & \quad-\left\{1+(T-t)\int_{t}^s \frac{\alpha^{ \rmf}(t,u)}{T-u} \rmd u\right\}(T-s) \nabla V_s(x_s)\notag \eqsp ,\\
    \mathbf{b}_{t,s}^{\alpha^\rmb, V}(x_0, x_s, x_t) &= (x_0-x_t)/t-\alpha^{ \rmb}(t,t)[x_0-x_t] +\frac{\partial_s \tilde{\alpha}^{\rmb}(t,s)}{s}t[x_0-x_s]\label{tsb:eq:target_b_cv} \\
    & \quad-\left\{1-t\int_{s}^t \frac{\alpha^{ \rmb}(t,u)}{u}\rmd u\right\}s \nabla V_s(x_s)\notag \eqsp ,
\end{align}
with $\tilde{\alpha}^{\rmf}(t,s)=(T-s)\alpha^{ \rmf}(t,s)$ and $\tilde{\alpha}^{ \rmb}(t,s)=s\alpha^{ \rmb}(t,s)$, where $\alpha^{\rmf}$ and $\alpha^{\rmb} \ccint{0,T}^2 \to \rset$ are continuously differentiable in their second variable.

Then, up to additive numerical constants independent of $\theta^{\rmf}$, $\theta^{\rmb}$, $ \alpha^\rmf$, $ \alpha^\rmb$, the following bias-variance loss decomposition holds
\begin{align*}
    \mathcal{L}_{\mathcal{M}_T}^{\rmf}(\theta^{\rmf}, \alpha^\rmf)&= \mathcal{L}_{\mathcal{M}_T}^{\rmf}(\theta^{\rmf}) + \mathcal{L}_{\text{CV}}^{\rmf}( \alpha^\rmf)\eqsp , \\
    \mathcal{L}_{\mathcal{M}_T}^{\rmb}(\theta^{\rmb}, \alpha^\rmb)&= \mathcal{L}_{\mathcal{M}_T}^{\rmb}(\theta^{\rmb}) + \mathcal{L}_{\text{CV}}^{\rmb}( \alpha^\rmb) \eqsp ,
\end{align*}
with
\begin{align*}
     \mathcal{L}_{\text{CV}}^{\rmf}( \alpha^\rmf)&= \textstyle\int_0^T (T-t)^{-1}\int_t^T \PE_{\Pi_{t}}\Big[\PE_{\Pi_{s,T|t}}\left[\norm{ \mathbf{b}_{t,s}^{\alpha^\rmf, V}(\bfX_t, \bfX_s, \bfX_T)}^2\right] \\
     & \textstyle \qquad \qquad \qquad \qquad \qquad  - \norm{(T-t)^{-1} \int_t^T \PE_{\Pi_{u,T|t}}[\mathbf{b}_{t,u}^{\alpha^\rmf, V}(\bfX_t, \bfX_u, \bfX_T)] \rmd u }^2\Big] \eqsp \rmd s\rmd t \eqsp , \\
    \mathcal{L}_{\text{CV}}^{\rmb}( \alpha^\rmb)&=\textstyle\int_0^T t^{-1}\int_0^t \PE_{\Pi_{t}}\left[\PE_{\Pi_{0,s|t}}\Big[\norm{ \mathbf{b}_{t,s}^{\alpha^\rmb, V}(\bfX_0, \bfX_s, \bfX_t)}^2\right]\\
    & \textstyle \qquad \qquad \qquad \qquad \qquad - \norm{t^{-1} \int_0^t \PE_{\Pi_{0,u|t}}[\mathbf{b}_{t,u}^{\alpha^\rmf, V}(\bfX_0, \bfX_u, \bfX_t)] \rmd u }^2\Big] \eqsp \rmd s\rmd t \eqsp .
\end{align*}

Therefore, minimizing $\mathcal{L}_{\mathcal{M}_T}^{\rmf}(\theta^{\rmf}, \alpha^\rmf)$
with respect to $(\theta^{\rmf},\alpha^\rmf)$ gives $\theta_\star^\rmf$, the original minimizer from \eqref{tsb:eq:obj_f}, and $ \alpha^\rmf_\star$ that minimizes $\mathcal{L}_{\text{CV}}^{\rmf}$, i.e., the scalar variance of the regression target. The same result applies for the backward direction.
\end{proposition}

\begin{proof}[Proof of \Cref{tsbm:prop:tsb_markovian_loss_cv}.]
We only prove the forward identity, the backward one being identical after time reversal. By \Cref{tsbm:prop:markov_theory_cv}, for almost surely any $x_t\sim \Pi_t$,
$$ \mathrm{b}_t^{\mathrm{f},V}(x_t) = \mathrm{b}_t^{\alpha^\mathrm{f},V}(x_t) = \frac{1}{T-t}\int_t^T \PE_{\Pi_{s,T|t}}\!\left[\mathbf{b}_{t,s}^{\alpha^\rmf, V}(\bfX_t, \bfX_s, \bfX_T)\mid \bfX_t=x_t\right]\rmd s. $$

Let us write
$$ m_t=\mathrm{b}_t^{\mathrm{f},V}(\bfX_t), \qquad y_t=v_t^{\theta^{\rmf}}(\bfX_t). $$

We begin with the following computation
\begin{align*}
&\frac{1}{T-t}\int_t^T \PE_{\Pi_{s,T|t}}\!\left[\|\mathbf{b}_{t,s}^{\alpha^\rmf, V}(\bfX_t, \bfX_s, \bfX_T)-y_t\|^2\mid \bfX_t\right]\rmd s \\
&\qquad= \frac{1}{T-t}\int_t^T \PE_{\Pi_{s,T|t}}\!\left[\|\mathbf{b}_{t,s}^{\alpha^\rmf, V}(\bfX_t, \bfX_s, \bfX_T)\|^2\mid \bfX_t\right]\rmd s +\frac{1}{T-t}\int_t^T \|y_t\|^2 \rmd s \\
&\qquad\qquad -\frac{2}{T-t}\int_t^T \left\langle \PE_{\Pi_{s,T|t}}\!\left[\mathbf{b}_{t,s}^{\alpha^\rmf, V}(\bfX_t, \bfX_s, \bfX_T)\mid \bfX_t\right], y_t\right\rangle \rmd s\\
&\qquad= \frac{1}{T-t}\int_t^T \PE_{\Pi_{s,T|t}}\!\left[\|\mathbf{b}_{t,s}^{\alpha^\rmf, V}(\bfX_t, \bfX_s, \bfX_T)\|^2\mid \bfX_t\right]\rmd s +\|y_t\|^2 -2\langle m_t,y_t\rangle\\
&\qquad= \frac{1}{T-t}\int_t^T \PE_{\Pi_{s,T|t}}\!\left[\|\mathbf{b}_{t,s}^{\alpha^\rmf, V}(\bfX_t, \bfX_s, \bfX_T)\|^2\mid \bfX_t\right]\rmd s +\|m_t-y_t\|^2 - \|m_t\|^2 \eqsp .\\
\end{align*}
Rearranging the obtain terms, we have
\begin{align*}
\|m_t-y_t\|^2 &= \frac{1}{T-t}\int_t^T \PE_{\Pi_{s,T|t}}\!\left[\|\mathbf{b}_{t,s}^{\alpha^\rmf, V}(\bfX_t, \bfX_s, \bfX_T)-y_t\|^2\mid \bfX_t\right]\rmd s \\
&\qquad- \left\{ \frac{1}{T-t}\int_t^T \PE_{\Pi_{s,T|t}}\!\left[\|\mathbf{b}_{t,s}^{\alpha^\rmf, V}(\bfX_t, \bfX_s, \bfX_T)\|^2\mid \bfX_t\right]\rmd s -\|m_t\|^2 \right\}.
\end{align*}
We now take the expectation with respect to $\bfX_t \sim \Pi_t$, and integrate over $t\in[0,T]$, to get
\begin{align*}
&\int_0^T \PE_{\Pi_t}\!\left[\|\mathrm{b}_{t,s}^{\mathrm{f},V}(\bfX_t)-v_t^{\theta^{\rmf}}(\bfX_t)\|^2\right]\rmd t \\
&\qquad= \int_0^T \PE_{\Pi_t}\!\left[ \frac{1}{T-t}\int_t^T \PE_{\Pi_{s,T|t}}\!\left[\|\mathbf{b}_{t,s}^{\alpha^\rmf, V}(\bfX_t, \bfX_s, \bfX_T)-v_t^{\theta^{\rmf}}(\bfX_t)\|^2 \mid \bfX_t\right]\rmd s \right]\rmd t \\
&\qquad\qquad- \int_0^T \PE_{\Pi_t}\!\left[ \frac{1}{T-t}\int_t^T \PE_{\Pi_{s,T|t}}\!\left[\|\mathbf{b}_{t,s}^{\alpha^\rmf, V}(\bfX_t, \bfX_s, \bfX_T)\|^2 \mid \bfX_t\right]\rmd s -\|\mathrm{b}_t^{\alpha^\mathrm{f},V}(\bfX_t)\|^2 \right]\rmd t \\
&\qquad= \int_0^T \frac{1}{T-t}\int_t^T \PE_{\Pi_{t,s,T}}\!\left[\|\mathbf{b}_{t,s}^{\alpha^\rmf, V}(\bfX_t, \bfX_s, \bfX_T)-v_t^{\theta^{\rmf}}(\bfX_t)\|^2 \mid \bfX_t\right]\rmd s\rmd t \\
&\qquad\qquad- \int_0^T \PE_{\Pi_t}\!\left[ \frac{1}{T-t}\int_t^T \PE_{\Pi_{s,T|t}}\!\left[\|\mathbf{b}_{t,s}^{\alpha^\rmf, V}(\bfX_t, \bfX_s, \bfX_T)\|^2 \mid \bfX_t\right]\rmd s -\|\mathrm{b}_t^{\alpha^\mathrm{f},V}(\bfX_t)\|^2 \right]\rmd t  \eqsp .
\end{align*}
This amounts exactly to the identity
\begin{align*}
    \int_0^T \PE_{\Pi_t}\!\left[\|\mathrm{b}_t^{\mathrm{f},V}(\bfX_t)-v_t^{\theta^{\rmf}}(\bfX_t)\|^2\right]\rmd t = \mathcal{L}_{\mathcal{M}_T}^\rmf(\theta^{\rmf}, \alpha^\rmf)- \mathcal{L}_{\mathrm{CV}}^\rmf( \alpha^\rmf) \eqsp .
\end{align*}
Recall that, by \eqref{tsbm:eq:computation_loss_f}, there exists a constant $\mathrm{c}^f$, independent of $\theta^{\rmf}$ and $ \alpha^\rmf$, such that
$$ \int_0^T \PE_{\Pi_t}\!\left[\|\mathrm{b}_t^{\mathrm{f},V}(\bfX_t)-v_t^{\theta^{\rmf}}(\bfX_t)\|^2\right]\rmd t = \mathcal{L}_{\mathcal{M}_T}^\rmf(\theta^{\rmf})+ \mathrm{c}^\rmf \eqsp .$$
Combining the two identities yields
$$ \mathcal{L}_{\mathcal{M}_T}^\rmf(\theta^{\rmf}, \alpha^\rmf) =
\mathcal{L}_{\mathcal{M}_T}^\rmf(\theta^{\rmf})+ \mathcal{L}_{\mathrm{CV}}^\rmf( \alpha^\rmf) + \mathrm{c}^\rmf \eqsp . $$
\end{proof}

\subsection{Results on the Markovian projection (discrete time)} \label{tsbm:app:markovian_proofs_d}

In this section, we assume that $\QbbVsg$ is given by \eqref{tsbm:eq:q_v_discret}. As in the continuous-time setting, we derive the optimal control drifts of the Markovian projection, along with their equivalent $\alpha$-based reformulation and the corresponding loss functions. We notably show that the potential-based control term admits an analogous interpretation: it can be written as an average contribution over future trajectory states at which the state cost is activated. We omit the proofs as they closely parallel those of the continuous-time case.

\begin{proposition}[Markovian projection -- discrete time] Let $\Pi \in \mathcal{R}(\QbbVsg)$. Then,
$\projM(\Pi)$ is equivalently induced by the diffusion processes
$(\bfXf_t)_{t\in [0,T]}$ and $(\bfXb_{T-t})_{t\in [0,T]}$, which are time-reversed of each other and respectively defined by the SDEs
\begin{align*}
    \rmd \bfXf_t &= \mathrm{b}_t^{\rmf, V}(\bfXf_t)\rmd t + \sigma\rmd\bfB_t\eqsp, \quad \bfXf_0 \sim \Pi_0 \eqsp ,\\
    \rmd \bfXb_t &= \mathrm{b}_{T-t}^{\rmb, V}(\bfXb_t)\rmd t + \sigma\rmd\bfB_t\eqsp, \quad \bfXb_0 \sim \Pi_T \eqsp ,
\end{align*}
where
\begin{align}
     \mathrm{b}_t^{\mathrm{f},V}(x_t) &=\textstyle\PE_{\Pi_{|t}}\left[(\bfX_T-\bfX_t)/(T-t)-\frac{1}{K_f(t)}\sum_{t_k>t}(T-t_k)\nabla V_{k}(\bfX_{t_k})  |\bfX_t=x_t\right] \eqsp , \label{tsbm:eq:v_f_star_d}\\
    \mathrm{b}_t^{\rmb, V}(x_t)&=\textstyle\PE_{\Pi_{|t}}\left[(\bfX_0-\bfX_t)/t-\frac{1}{K_b(t)}\sum_{t_k<t}t_k\nabla V_{k}(\bfX_{t_k})  |\bfX_t=x_t\right] \eqsp ,\label{tsbm:eq:v_b_star_d}
\end{align}
where $K_f(t)=|\{k\in \{1, \hdots, K\}: t_k >t\}|$ and $K_b(t)=|\{k\in \{1, \hdots, K\}: t_k <t\}|$.

Moreover, for any $t\in [0,T]$, we have $\mathrm{Law}(\bfXf_t)=\mathrm{Law}(\bfXb_{T-t})=\Pi_t$.
Additionally, $\mathrm{b}^{\rmf, V}$ and $\mathrm{b}^{\rmb, V}$ can be approximated by $v^{\theta^{\rmf}_\star}$ and $v^{\theta^{\rmb}_\star}$, where $\theta^{\rmf}_\star$ and $\theta^{\rmb}_\star$ are respectively the minimizers of forward and backward objectives $\mathcal{L}^\rmf_{\mathcal{M}_T}$ and $\mathcal{L}_{\mathcal{M}_T}^\rmb$, defined by
\begin{align}
     &\mathcal{L}_{\mathcal{M}_T}^f(\theta^\rmf) =\int_0^{t_K} \frac{1}{K_f(t)}\sum_{t_k>t} \PE_{\Pi_{t,t_k,T}}\left[\norm{ \mathbf{b}^{\rmf, V}_{t,k}(\bfX_t, \bfX_{t_k}, \bfX_T)-v_t^{\theta^\rmf}(\bfX_t)}^2\right] \rmd t \label{tsbm:app:obj_forward_d} \\
     &\qquad \qquad \quad  + \int_{t_K}^T \PE_{\Pi_{t,T}}\left[\norm{(\bfX_T-\bfX_t)/(T-t) -v_t^{\theta^\rmf}(\bfX_t)}^2\right] \rmd t \notag  \eqsp ,\\
     &\mathcal{L}_{\mathcal{M}_T}^b(\theta^\rmb)=\int_{t_1}^T \frac{1}{K_b(t)}\sum_{t_k<t} \PE_{\Pi_{0,t_k,t}}\left[\norm{\mathbf{b}^{\rmb, V}_{t,k}(\bfX_0, \bfX_{t_k}, \bfX_t) -v_t^{\theta^\rmb}(\bfX_t)}^2\right] \rmd t \notag\\
     &\qquad \qquad \quad  + \int_{0}^{t_1}\PE_{\Pi_{0,t}}\left[\norm{(\bfX_0-\bfX_t)/t -v_t^{\theta^\rmb}(\bfX_t)}^2\right] \rmd t \notag \eqsp ,
\end{align}
with the following regression targets
\begin{align*}
    \mathbf{b}^{\rmf, V}_{t,k}(x_t, x_{t_k}, x_T) &= (x_T-x_t)/(T-t)-(T-t_k)\nabla V_{k}(x_{t_k}) \eqsp , \\
    \mathbf{b}^{\rmb, V}_{t,k}(x_0, x_{t_k}, x_t) &= (x_0-x_t)/t-t_k\nabla V_{k}(x_{t_k}) \eqsp .
\end{align*}
\end{proposition}

\begin{proposition}[Markovian projection with control variates -- discrete time] Let $\Pi \in \mathcal{R}(\QbbVsg)$, and $\alpha^\rmf$, $\alpha^\rmb: \ccint{0,T}^2 \to \rset$, continuously differentiable in their second variable.. Define the velocity fields
\begin{align*}
     \mathrm{b}_t^{\alpha^\rmf,V}(x_t)&\textstyle=\PE_{\Pi_{|t}}\Big[\frac{\bfX_T-\bfX_t}{T-t}-\alpha^{\rmf}(t,t)[\bfX_T-\bfX_t] - \int_t^T\frac{\partial_s \tilde{\alpha}^{\rmf}(t,s)}{T-s}[\bfX_T-\bfX_s]\rmd s \\
    & \textstyle \qquad\quad-\sum_{t_k>t}\left\{\frac{1}{K_f(t)}+\int_{t}^{t_k} \frac{\alpha^{\rmf}(t,u)}{K_f(u)} \rmd u\right\}(T-t_k) \nabla V_k(\bfX_{t_k}) |\bfX_t=x_t\Big] \eqsp, \\
     \mathrm{b}_t^{\alpha^\rmb,V}(x_t)&=\textstyle\PE_{\Pi_{|t}}\Big[ \frac{\bfX_0-\bfX_t}{t}-\alpha^{\rmb}(t,t)[\bfX_0-\bfX_t] + \int_0^t\frac{\partial_s \tilde{\alpha}^{\rmb}(t,s)}{s}[\bfX_0-\bfX_s]\rmd s \\
    & \textstyle\qquad \quad-\sum_{t_k<t}\left\{\frac{1}{K_b(t)}-\int_{t_k}^{t} \frac{\alpha^{\rmb}(t,u)}{K_b(u)} \rmd u\right\}t_k \nabla V_k(\bfX_{t_k}) |\bfX_t=x_t\Big] \eqsp ,
\end{align*}
where $\tilde{\alpha}^{\rmb}(t,s)=s\alpha^\rmb(t,s)$ and  $\tilde{\alpha}^{\rmf}(t,s)=(T-s)\alpha^\rmf(t,s)$.  Then, we have the equality $\mathrm{b}^{\alpha^\rmf,V}=\mathrm{b}^{\rmf,V}$ and $\mathrm{b}^{\alpha^\rmb,V}=\mathrm{b}^{\rmb,V}$, where
the drifts $\mathrm{b}^{\rmf,V}$ and $\mathrm{b}^{\rmb,V}$ are respectively defined in \eqref{tsbm:eq:v_f_star_d} and \eqref{tsbm:eq:v_b_star_d}\footnote{Note that the $\alpha$-based expressions of the drifts are still valid in out-of-potential settings ($t>t_K$ in forward case, $t<t_0$ in backward case) : the sum over $t_k$ is simply 0.}. 
\end{proposition}

\begin{proposition}[TSBM bridge matching losses with control variates -- discrete time]\label{tsbm:app:prop_cv_discrete}
Let $\Pi\in \mathcal{R}(\QbbVsg)$. Consider the respective forward and  objectives
\begin{align}
     &\textstyle\mathcal{L}_{\mathcal{M}_T}^{\rmf}(\theta^{\rmf}, \alpha^\rmf) =\int_0^{t_K}  \rmd t\sum_{t_k>t}\frac{1}{K_f(t)}\int_{t_{k-1}(t)}^{t_k} \frac{\rmd s}{t_k-t_{k-1}}\int_{t_K}^T \frac{\rmd r}{T-t_K}\PE_{\Pi_{t,s,t_k,r, T}}\left[\norm{ \mathbf{b}_{t,s,t_k, r}^{ \alpha^\rmf,V} -v_t^{\theta^{\rmf}}(\bfX_t)}^2\right] \notag \\
     &  \qquad\qquad \qquad \textstyle+ \int_{t_K}^T \rmd t \int_t^T \frac{\rmd s}{T-t} \PE_{\Pi_{t,s, T}}\left[\norm{ \mathbf{b}_{t,s}^{\alpha^\rmf} -v_t^{\theta^{\rmf}}(\bfX_t)}^2\right] \eqsp ,\label{tsb:eq:obj_f_cv_discrete} \\
     &\textstyle\mathcal{L}_{\mathcal{M}_T}^{\rmb}(\theta^{\rmb}, \alpha^\rmb) =\int_{t_1}^{T} \rmd t \sum_{t_k<t}\frac{1}{K_b(t)} \int_{t_{k}}^{t_{k+1}(t)}\frac{\rmd s}{t_{k+1}-t_{k}} \int_{0}^{t_1} \frac{\rmd r}{t_1}\PE_{\Pi_{0,r, s, t_k, t}}\left[\norm{ \mathbf{b}_{t,s,t_k, r}^{ \alpha^\rmb,V} -v_t^{\theta^{\rmb}}(\bfX_t)}^2\right]  \notag \\
     & \qquad\qquad \qquad \textstyle+ \int_{0}^{t_1}\rmd t \int_0^t \frac{\rmd s}{t} \PE_{\Pi_{0,s,t}}\left[\norm{ \mathbf{b}_{t,s}^{\alpha^\rmb} -v_t^{\theta^{\rmb}}(\bfX_t)}^2\right]\label{tsb:eq:obj_b_cv_discrete} \eqsp ,
\end{align}
with the following notation: in the forward direction, we set
\[
t_{k-1}(t)=
\begin{cases}
t, & \text{if } k=\arg\min\{k' : t_{k'}>t\},\\
t_{k-1}, & \text{otherwise},
\end{cases}
\]
while in the backward direction, we set
\[
t_{k+1}(t)=
\begin{cases}
t, & \text{if } k=\arg\max\{k' : t_{k'}<t\},\\
t_{k+1}, & \text{otherwise}.
\end{cases}
\]
Moreover, we write
\[
\mathbf{b}_{t,s}^{\alpha^\rmf}
=
\mathbf{b}_{t,s}^{\alpha^\rmf,0}(\bfX_t,\bfX_s,\bfX_T),
\qquad
\mathbf{b}_{t,s}^{\alpha^\rmb}
=
\mathbf{b}_{t,s}^{\alpha^\rmb,0}(\bfX_0,\bfX_s,\bfX_t),
\]
where these quantities are inherited from the continuous-time case with $V\equiv 0$, see \Cref{tsbm:prop:tsb_markovian_loss_cv}. The discrete-time $\alpha$-based regression targets are then given by
\begin{align}
    \mathbf{b}_{t,s,t_k, r}^{ \alpha^\rmf,V}&\textstyle=\frac{\bfX_T-\bfX_t}{T-t}-\alpha^{\rmf}(t,t)[\bfX_T-\bfX_t] -\frac{\partial_s \tilde{\alpha}^{ \rmf}(t,s)}{T-s}K_f(t) (t_k- t_{k-1})[\bfX_T-\bfX_s]\notag\\
    &\textstyle\quad - \frac{\partial_s \tilde{\alpha}^{ \rmf}(t,r)}{T-r}(T-t_K)[\bfX_T-\bfX_r] -\left\{1+K_f(t)\int_{t}^{t_k} \frac{\alpha^{\rmf}(t,u)}{K_f(u)} \rmd u\right\}(T-t_k) \nabla V_k(\bfX_{t_k}) \eqsp ,\notag\\
    \mathbf{b}_{t,s,t_k, r}^{ \alpha^\rmb,V} &\textstyle= \frac{\bfX_0-\bfX_t}{t}-\alpha^{ \rmb}(t,t)[\bfX_0-\bfX_t] +\frac{\partial_s \tilde{\alpha}^{ \rmb}(t,s)}{s}K_b(t)(t_{k+1}-t_k)[\bfX_0-\bfX_s]\notag \\
    &\textstyle\quad + \frac{\partial_s \tilde{\alpha}^{ \rmf}(t,r)}{r}t_1[\bfX_0-\bfX_r] -\left\{1-K_b(t)\int_{t_k}^t \frac{\alpha^{ \rmb}(t,u)}{K_b(u)}\rmd u\right\}t_k \nabla V_k(\bfX_{t_k}) \eqsp , \notag
\end{align}
with $\tilde{\alpha}^{ \rmf}(t,s)=(T-s)\alpha^{ \rmf}(t,s)$ and $\tilde{\alpha}^{ \rmb}(t,s)=s\alpha^{ \rmb}(t,s)$, where $\alpha^{\rmf}$ and $\alpha^{\rmb} \ccint{0,T}^2 \to \rset$ are continuously differentiable in their second variable.

Then, up to additive numerical constants independent of $\theta^{\rmf}$ and $ \alpha^\rmf$, the following bias-variance loss decomposition holds
\begin{align*}
    \mathcal{L}_{\mathcal{M}_T}^{\rmf}(\theta^{\rmf}, \alpha^\rmf)= \mathcal{L}_{\mathcal{M}_T}^{\rmf}(\theta^{\rmf}) + \mathcal{L}_{\text{CV}}^{\rmf}( \alpha^\rmf) \eqsp,
\end{align*}
with
\begin{align*}
    \mathcal{L}_{\text{CV}}^{\rmf}( \alpha^\rmf)& \textstyle= \int_0^{t_K} \rmd t \sum_{t_k>t}\frac{1}{K_f(t)}\int_{t_{k-1}}^{t_k}\frac{\rmd s}{t_k-t_{k-1}} \int_{t_K}^T \frac{\rmd r}{T-t_K}\PE_{\Pi_{t}}\Bigg[\PE_{\Pi_{s, t_k,r,T|t}}\left[\norm{ \mathbf{b}_{t,s,t_k, r}^{ \alpha^\rmf,V}}^2\right]\\
    &\qquad \qquad \qquad \textstyle- \norm{\sum_{t_k>t} \frac{1}{K_f(t)}\int_{t_{k-1}}^{t_k}\frac{\rmd s}{t_k-t_{k-1}} \int_{t_K}^T \frac{\rmd r}{T-t_K} \PE_{\Pi_{s, t_k,r,T|t}}[\mathbf{b}_{t,s,t_k, r}^{ \alpha^\rmf,V}]}^2\Bigg]\\
    & \textstyle+\int_{t_K}^T (T-t)^{-1}\int_t^T \PE_{\Pi_{t}}\left[\PE_{\Pi_{s,T|t}}\left[\norm{ \mathbf{b}_{t,s}^{\alpha^\rmf}}^2\right] - \norm{(T-t)^{-1} \int_t^T \PE_{\Pi_{u,T|t}}[\mathbf{b}_{t,u}^{\alpha^\rmf}] \rmd u }^2\right] \eqsp \rmd s\rmd t \eqsp .
\end{align*}

Therefore, minimizing $\mathcal{L}_{\mathcal{M}_T}^{\rmf}(\theta^{\rmf}, \alpha^\rmf)$
with respect to $(\theta^{\rmf},\alpha^\rmf)$ gives $\theta_\star^\rmf$, the original minimizer from \eqref{tsbm:app:obj_forward_d}, and $ \alpha^\rmf_\star$ that minimizes $\mathcal{L}_{\text{CV}}^{\rmf}$, i.e., the scalar variance of the regression target. The same result applies for the backward direction.
\end{proposition}

\section{Fully discrete-time Twisted Schrödinger Bridge Matching (D-TSBM)}
\label{tsbm:app:full_discrete}

In practice, numerical methods for solving SB problems rely on translating continuous-time objectives to a discrete-time setting, which can introduce unavoidable bias. To mitigate this issue, we reformulate the SB problem directly in discrete time. This strategy was recently proposed by \cite{gushchin2024adversarial} in the canonical SB case; here, we extend it to the twisted setting, which is the focus of this paper. In \Cref{tsbm:subsec:discree_time_sb_setting}, we introduce the resulting discrete-time formulation and derive an ideal IMF scheme aligned with the continuous-time procedure reviewed in \Cref{tsbm:sec:background}. We then translate this ideal scheme into a practical algorithm, namely the discrete-time version of TSBM (D-TSBM), with the only main change being the Markovian projection, described in \Cref{tsbm:subsec:discrete_time_markov}. Proofs of the propositions stated in this section are deferred to \Cref{tsbm:app:discrete_proofs}.

Throughout this section, we consider a time-discretization grid $\{t_n\}_{n=0}^{N+1}$ of $[0,T]$ with $(N+2)$ increasing time points, such that $t_0=0 < t_1 < \hdots < t_{N+1}=T$, and we define the step sizes $\delta_n=t_n-t_{n-1}$ for $n\in\{1,\hdots,N+1\}$.

\subsection{Setting and ideal IMF scheme}\label{tsbm:subsec:discree_time_sb_setting}

\paragraph{Discrete-time twisted reference measure.}
Let $q^\sigma\in \Pmeasure^{(N+2)}$ denote the joint distribution given by the finite-dimensional projection of $\Qbbsg$ at times $\{t_n\}_{n=0}^{N+1}$. In particular, $q^\sigma$ is Markov and can be written as
\[
q^\sigma(x_{0:N+1})= \delta_0(x_0)\prod_{n=1}^{N+1}q^\sigma_{n|n-1}(x_{n}|x_{n-1}) \eqsp , \eqsp \text{where} \eqsp q^\sigma_{n|n-1}(\cdot|x_{n-1})= \densityGaussian(x_{n-1}, \sigma^2 (t_n -t_{n-1})\Idd) \eqsp .
\]
To define a discrete-time analogue of the SB problem \eqref{tsbm:eq:classic_sb}, we introduce a twisted version of $q^\sigma$ that will serve as reference. To this end, we replace the continuous-time potential $V$ by a collection of $N$ potentials $\{g_n\}_{n=1}^N$ satisfying $g_n\in \mathrm{C}^1(\rset^d,(0,\infty))$, and we define the \emph{reference} distribution $q^{g,\sigma}\in \Pmeasure^{(N+2)}$, absolutely continuous with respect to $q^\sigma$, by
\begin{align} \label{tsbm:eq:q_V}
     q^{g,\sigma}(x_{0:N+1})= \frac{1}{Z_g}\left(\prod_{n=1}^N g_n(x_n)\right) q^\sigma(x_{0:N+1}) \eqsp ,
\end{align}
with normalizing constant $Z_g= \int_{(\rset^d)^{N+2}}\prod_{n=1}^N g_n(x_n) \rmd q^\sigma(x_{0:N+1})$. Thus, $q^{g,\sigma}$ is a Feynman--Kac model with prior $q^\sigma$; in line with the continuous-time setting, we refer to $q^{g,\sigma}$ as a twisted version of $q^\sigma$. Such twisted measures are ubiquitous in hidden Markov modeling and often simulated by Sequential Monte Carlo (SMC), see \eg, \cite{heng2020controlled}. Intuitively, each potential $g_n$ provides guidance on top of the prior transition kernel $q^\sigma_{n|n-1}$, with the uninformative limit case given by $g_n\equiv 1$.
In \Cref{tsbm:app:discrete_proofs}, we show that, as in the continuous-time setting, both $q^{g,\sigma}$ and its associated bridge $q^{g,\sigma}_{|0,N+1}$ are Markov. Moreover, their transition kernels are explicit but typically intractable, and can be written as Doob's $h$-transforms of the corresponding kernels of $q^\sigma$ and $q^\sigma_{|0,N+1}$, respectively, induced by the potentials $\{g_n\}_{n=1}^N$.

\paragraph{Discrete-time Twisted Schrödinger Bridge problem.}
Based on \eqref{tsbm:eq:q_V}, we define the discrete-time TSB problem as
\begin{align}
    \label{tsbm:eq:discrete_time_SB}
  \argmin \ensembleLigne{\KL(p\|q^{g,\sigma})}{p \in \Pmeasure^{(N+2)}, p_0 = \mu, \ p_{N+1} = \nu} \eqsp .
\end{align}
By the KL chain rule, the solution of \eqref{tsbm:eq:discrete_time_SB}, denoted $p^\star$, is the finite-dimensional projection of the continuous-time solution $\Pbb^N$ at times $\{t_n\}_{n=0}^{N+1}$, where $\Pbb^N$ solves \eqref{tsbm:eq:classic_sb} with reference measure $\Qbb^{\text{ref}}=\Qbb^{g,N}$, where
\begin{align}
    \rmd\Qbb^{g, N}(x_{[0:T]})\propto \prod_{n=1}^N g_n(x_{t_n}) \rmd \Qbbsg(x_{[0:T]}) \eqsp .\label{tsbm:app:qbb_g_N}
\end{align}

This observation allows us to make explicit the correspondence between the continuous-time formulation \eqref{tsbm:eq:classic_sb} and the discrete-time formulation \eqref{tsbm:eq:discrete_time_SB}.

\begin{proposition}[Informal version]\label{tsbm:prop:equivalence_cont_disc}
Let $V\in \rmC^{0,1}([0,T]\times \rset^d,\rset)$ be such that $Z_V<\infty$. Define the potentials $g_n=\exp(-\delta_n V_{t_n}/\sigma^2)$ for any $n\in\{1,\hdots,N\}$. Then, in the asymptotic regime $\delta_n\to 0$ (equivalently $N\to\infty$), $\Pbb^N$ converges to $\Pbb^\star$.
\end{proposition}

We now restate \cite[Theorem 3.1]{gushchin2024adversarial}, which was originally formulated for the canonical setting but remains valid here under the same assumptions. In particular, this result provides a discrete-time characterization of $p^\star$, consistent with the continuous-time formulation established in \cite{leonard2014survey}.
\begin{proposition}[From \cite{gushchin2024adversarial}]\label{tsbm:prop:carac_SB_discrete}
Let $p\in \Pmeasure^{(N+2)}$ such that $p\in \mathcal{M}_N$, $p=p_{0,N+1}q^{g,\sigma}_{1:N|0,N+1}$, $p_0=\mu$, and $p_{N+1}=\nu$. Then, $p$ is a solution to \eqref{tsbm:eq:discrete_time_SB}.
\end{proposition}
This result motivates the design of a discrete-time IMF scheme, which we detail now.

\paragraph{Discrete-time ideal IMF scheme.}
Let $p\in \Pmeasure^{(N+2)}$. We define the reciprocal projection of $p$ with respect to $q^{g,\sigma}$, denoted $\projRdisc(p)\in \Pmeasure^{(N+2)}$, as the solution of
\begin{align}
    \argmin \ensembleLigne{\KL(p \| r)}{r \in \mathcal{R}(q^{g,\sigma})} \eqsp ,
\end{align}
where $\mathcal{R}(q^{g,\sigma})$ denotes the class of distributions $r\in\Pmeasure^{(N+2)}$ that are mixtures of $q^{g,\sigma}$-bridges, \ie, $r=r_{0,N+1}q^{g,\sigma}_{1:N|0,N+1}$. In analogy with the continuous-time setting, we refer to this class as the set of reciprocal distributions. Following \citep[Proposition 3.3]{gushchin2024adversarial}, this projection admits the explicit form
\begin{align}
    \projRdisc(p)= p_{0,N+1}q^{g,\sigma}_{1:N|0,N+1} \eqsp .
\end{align}

We also define the Markovian projection of $p$, denoted $\projMdisc(p)\in \Pmeasure^{(N+2)}$, as the solution of
\begin{align} \label{tsbm:eq:proj_m_discret}
   \argmin \ensembleLigne{\KL(p \| m)}{m \in \mathcal{M}_N} \eqsp ,
\end{align}
where $\mathcal{M}_N$ denotes the set of Markov distributions $m\in \Pmeasure^{(N+2)}$ admitting the same factorization as $q^\sigma$, \ie,
\[
m(x_0,x_1,\hdots,x_{N+1})= m_0(x_0)\prod_{n=1}^{N+1}m_{n|n-1}(x_n|x_{n-1}) \eqsp .
\]
In general, this minimization does not yield a closed-form solution. However, when $p\in \mathcal{R}(q^{g,\sigma})$, we can expand the KL objective explicitly, as stated in the following lemma.

\begin{lemma}\label{tsbm:lemma:markov_proj}
Let $p\in \mathcal{R}(q^{g,\sigma})$ and $m\in \mathcal{M}_N$. Then, the following decomposition holds
\begin{align*}
    \KL(p\|m) & = \KL(p_0\|m_0) + \PE_{p_{0,N}}[\KL(p_{N+1|0}(\cdot|X_0)\|m_{N+1|N}(\cdot|X_N))] \\
    & \quad + \sum_{n=1}^{N}\PE_{p_{n-1,N+1}}[\KL(q^{g,\sigma}_{n|n-1,N+1}(\cdot|X_{n-1}, X_{N+1})\|m_{n|n-1}(\cdot|X_{n-1}))] \eqsp .
\end{align*}
\end{lemma}

Using the equivalent backward factorization
\[
m(x_0,x_1,\hdots,x_{N+1})= m_{N+1}(x_{N+1})\prod_{n=1}^{N+1}m_{n-1|n}(x_{n-1}|x_{n}) \eqsp ,
\]
we derive a symmetric result.

\begin{lemma}\label{tsbm:lemma:markov_proj_b}
Let $p\in \mathcal{R}(q^{g,\sigma})$ and $m\in \mathcal{M}_N$. Then, the following decomposition holds
\begin{align*}
    \KL(p\|m) & = \KL(p_{N+1}\|m_{N+1}) + \PE_{p_{1,N+1}}[\KL(p_{0|N+1}(\cdot|X_{N+1})\|m_{0|1}(\cdot|X_1))] \\
    & \quad + \sum_{n=2}^{N+1}\PE_{p_{0,n}}[\KL(q^{g,\sigma}_{n-1|0,n}(\cdot|X_{0}, X_{n})\|m_{n-1|n}(\cdot|X_{n}))] \eqsp .
\end{align*}

\end{lemma}

In particular, the KL objective in \Cref{tsbm:lemma:markov_proj_b} is equivalent to the discrete-time DDPM variational objective derived in \cite[Equation 5]{Ho2020denoising} in the case of uninformative potentials, highlighting a tight connection between bridge matching procedures and diffusion-model training objectives beyond the continuous-time setting.

We can now define the ideal discrete-time IMF procedure for solving \eqref{tsbm:eq:discrete_time_SB} by alternating between the Markovian and Reciprocal projections defined above. Starting from an arbitrary $p^{-1}\in \Pmeasure^{(N+2)}$ satisfying $p^{-1}_0=\mu$ and $p^{-1}_{N+1}=\nu$, we define recursively for $n\in\nset$
\begin{align} \label{tsbm:eq:imf_discrete}
    p^{2n} = \projRdisc(p^{2n-1}) \eqsp , \quad p^{2n+1} = \projMdisc(p^{2n}) \eqsp .
\end{align}

\subsection{Computation of the Reciprocal projection}

Let $p\in\mathcal{M}_N$, and assume that samples from $p_{0,N+1}$ are available. Observe that $\projRdisc(p)$ coincides with the finite-dimensional projection of the reciprocal path measure $p_{0,N+1}\Qbb^{g,N}_{\mid 0,T}$, where $\Qbb^{g,N}$ is defined in \eqref{tsbm:app:qbb_g_N}. Since the conditional bridge $\Qbb^{g,N}_{\mid 0,T}$ is intractable, we replace the discrete reciprocal projection step with the following \emph{continuous-time} variational problem:
\begin{align}\label{tsbm:eq:variational_full_discrete}
    \textstyle\argmin_{\psi \in \Psi}
    \PE_{p_{0,N+1}}
    \left[
    \KL(\Qbb^{\psi}_{\mid 0,T}\|\Qbb^{g,N}_{\mid 0,T})
    \right] \eqsp ,
\end{align}
using the same variational bridge family $\{\Qbb^\psi_{\mid 0,T}\}_{\psi\in \Psi}$ as in the continuous-time setting. The following proposition, which provides tractable objectives, follows directly from \Cref{tsbm:app:prop_reciprocal_discrete}.

\begin{proposition}
Assume that the variational family $\{\Qbb^{\psi}_{|0,T}\}_{\psi \in \Psi}$ in \eqref{tsbm:eq:variational_cont} describes the pinned diffusion processes associated, for almost surely any $(x_0,x_T)\sim \Pbb_{0,T}$, to forward-time and backward-time parametric SDEs of the form
\begin{align*}
\rmd \bfX^{0,T}_t &=  v_t^{\rmf,\psi} (\bfX^{0,T}_t|x_0, x_T) \rmd t + \sigma \rmd \bfB_t\eqsp , \eqsp \bfX^{0,T}_0=x_0, \eqsp \bfX^{0,T}_T=x_T \eqsp, \\
\rmd \bfY^{0,T}_t &=  v_{T-t}^{\rmb,\psi} (\bfY^{0,T}_t|x_0, x_T) \rmd t + \sigma \rmd \bfB_t\eqsp , \eqsp \bfY^{0,T}_0=x_T, \eqsp \bfY^{0,T}_T=x_0 \eqsp,
\end{align*}
with $v^{\rmf,\psi},v^{\rmb,\psi} \in \adminControlT$. Then, the minimizer of \eqref{tsbm:eq:variational_full_discrete} for $\Qbb^{g,N}$ defined by \eqref{tsbm:app:qbb_g_N}, denoted by $\psi_\star$, is equivalently the minimizer of the two following losses
\begin{align*}
  \mathcal{L}_{\mathcal{R}}^{\rmf}(\psi) &= \textstyle\int_{0}^T\PE_{\Pi^{\psi}_{0,t,T}}\left[ \frac{1}{2}\|(\bfX_T -\bfX_t)/(T-t)-v_t^{\rmf,\psi}(\bfX_t|\bfX_0, \bfX_T)\|^2\right]\rmd t \\
  & \textstyle \quad -\sigma^2 \sum_{n=1}^N \PE_{\Pi^{\psi}_{t_n}}[\log g_n(\bfX_{t_n})] \eqsp ,\\
  \mathcal{L}_{\mathcal{R}}^{\rmb}(\psi) &= \textstyle\int_{0}^T\PE_{\Pi^{\psi}_{0,t,T}}\left[ \frac{1}{2}\|(\bfX_0 -\bfX_t)/t-v_t^{\rmb,\psi}(\bfX_t|\bfX_0, \bfX_T)\|^2 \right]\rmd t\\
  & \textstyle \quad -\sigma^2 \sum_{n=1}^N \PE_{\Pi^{\psi}_{t_n}}[\log g_n(\bfX_{t_n})] \eqsp ,
\end{align*}
where $\Pi^{\psi}= p_{0,N+1}\Qbb^{\psi}_{|0,T}$. Moreover, $\Pi^{\psi_\star}_{t_0, t_1, \ldots, t_{N+1}}$ is a valid variational approximation of the discrete reciprocal measure $\projRdisc(p)$.
\end{proposition}

\subsection{Computation of the Markovian projection} \label{tsbm:subsec:discrete_time_markov}

In this subsection, we focus on the practical implementation of the Markovian projection \eqref{tsbm:eq:proj_m_discret}, restricted to reciprocal distributions $p\in \mathcal{R}(q^{g,\sigma})$, which is precisely the setting encountered in the IMF scheme \eqref{tsbm:eq:imf_discrete}.

As in the continuous-time setting, we restrict the Markovian projection to a prescribed parametric family indexed by $\theta \in \Theta$, and reformulate the KL objectives of \Cref{tsbm:lemma:markov_proj} and \Cref{tsbm:lemma:markov_proj_b} as tractable training losses.
More specifically, we consider Markov distributions with Gaussian transition kernels, equipped with a denoiser-based parameterization consistent with standard diffusion-model practice and applicable for arbitrary $N$.

Before introducing this parameterization, we recall the general expression of the discrete-time Brownian bridge marginals: for any triplet $(i,j,k)\in \{0,\hdots,N+1\}$ such that $j<i<k$,
\begin{align} \label{tsbm:eq:brownian_bridge_discrete}
    q^\sigma_{i|j, k}(x_i|x_{j}, x_{k})=\gauss\left(\alpha_{i}^{j,k}x_j + \beta_{i}^{j,k}x_k, \sigma^2 (\gamma_{i}^{j,k})^2 \Idd\right) \eqsp ,\\
    \text{where } \alpha_{i}^{j,k}= \frac{t_k -t_i}{t_k -t_j} \eqsp , \eqsp \beta_{i}^{j,k}= \frac{t_i-t_j}{t_k -t_j} \eqsp, \eqsp \gamma_{i}^{j,k}=  \sqrt{\frac{(t_i-t_j)(t_k -t_i)}{t_k -t_j}} \eqsp .
\end{align}

\paragraph{Denoiser-like Markovian parameterization.}
Inspecting the KL objectives in \Cref{tsbm:lemma:markov_proj} and \Cref{tsbm:lemma:markov_proj_b}, a natural strategy would be to parameterize Markov kernels using the twisted bridge $q^{g,\sigma}_{|0,N+1}$, mirroring the DDPM parameterization \citep{Ho2020denoising} in the canonical case where $g_n\equiv1$. However, since $q^{g,\sigma}_{|0,N+1}$ is intractable for general potentials, this approach is not directly applicable.
We therefore retain the original DDPM parameterization based on the Brownian bridge $q^\sigma_{|0,N+1}$. For the forward Markovian projection, we consider
\begin{align}\label{tsb:eq:markov_param_denoiser_f}
    m^{\theta^\rmf}_{n|n-1}(x_n|x_{n-1})&=q^\sigma_{n|n-1, N+1}(x_n|x_{n-1}, {x}_{N+1}^{\theta^\rmf}(t_{n-1}, x_{n-1})) &&\text{if }  n\in\{1, \hdots, N\},\\
    m^{\theta^\rmf}_{N+1|N}(x_{N+1}|x_{N})&=\gauss(x_{N+1}; {x}_{N+1}^{\theta^\rmf}(t_N, x_{N}), \sigma_{N+1}^2 \Idd)\eqsp ,\notag
\end{align}
where ${x}^{\theta^\rmf}_{N+1}$ is intended to predict $x_{N+1}$, and $\sigma_{N+1}>0$ is left as a hyperparameter.

Symmetrically, for the backward projection, we set
\begin{align}\label{tsb:eq:markov_param_denoiser_b}
    m^{\theta^\rmb}_{n-1|n}(x_{n-1}|x_{n})&=q^\sigma_{n-1|0, n}(x_{n-1}|{x}^{\theta^\rmb}_{0}(t_{n}, x_{n}), x_n) &&\text{if }  n\in\{2, \hdots, N+1\} \eqsp ,\\
    m^{\theta^\rmb}_{0|1}(x_{0}|x_{1})&=\gauss(x_0; {x}^{\theta^\rmb}_{0}(t_1, x_{1}), \sigma_{0}^2 \Idd) \eqsp,\notag
\end{align}
where ${x}^{\theta^\rmb}_{0}$ is intended to predict $x_{0}$, and $\sigma_{0}>0$ is left as a hyperparameter.
This variational family yields computable Markovian projection objectives for any $N\ge1$.

\begin{proposition} \label{tsbm:prop:markov_denoiser} Let $p\in \mathcal{R}(q^{g,\sigma})$. Define $\theta^\rmf_\star$ and $\theta^\rmb_\star$ as the minimizers over $\Theta$ of the following forward and backward objectives
\begin{align}
    \mathcal{L}^\rmf_{\mathcal{M}_N}(\theta^\rmf) =&\textstyle\sum_{n=1}^{N}\frac{1}{N-n+1} \sum_{k=n}^N \PE_{p_{n-1,k,N+1}}\Big[\| A^\rmf_n \sigma^2 (T-t_k)\nabla \log g_k(X_k)\notag\\
    & \textstyle\qquad   \qquad \qquad  \qquad \qquad \qquad \qquad \qquad+B^\rmf_n\left(X_{N+1}-{x}^{\theta^\rmf}_{N+1}(t_{n-1},X_{n-1})\right) \|^2\Big]\notag\\
            & \textstyle\quad + \frac{\sigma^2}{\sigma_{N+1}^2}\PE_{p_{N,N+1}}\left[\norm{X_{N+1} -{x}^{\theta^\rmf}_{N+1}(t_N, X_{N}) }^2\right] \eqsp , \label{tsbm:eq:markov_f_loss_denoiser}\\
    \mathcal{L}^\rmb_{\mathcal{M}_N}(\theta^\rmb) =&\textstyle \sum_{n=2}^{N+1} \frac{1}{n-1}\sum_{k=1}^{n-1} \PE_{p_{0,k,n}}\left[\norm{ A_{n}^\rmb \sigma^2 t_k\nabla \log g_k(X_k) +B^\rmb_n\left(X_{0}-{x}_{0}^{\theta^\rmb}(t_{n},X_{n})\right) }^2\right]\notag\\
            & \textstyle\quad + \frac{\sigma^2}{\sigma_{0}^2}\PE_{p_{0,1}}\left[\norm{X_{0} -{x}_{0}^{\theta^\rmb}(t_1, X_{1}) }^2\right] \eqsp,\label{tsbm:eq:markov_b_loss_denoiser}
\end{align}
where $A^\rmf_n= \gamma_{n}^{n-1, N+1}\times\frac{N-n+1}{T-t_n}$, $B^\rmf_n= \frac{\beta_{n}^{n-1, N+1}}{\gamma_{n}^{n-1, N+1}}$, $A^\rmb_n= \gamma_{n-1}^{0, n}\times\frac{n-1}{t_{n-1}}$ and $B^\rmb_n= \frac{\alpha_{n-1}^{0, n}}{\gamma_{n-1}^{0, n}}$.

Then, the Markovian projection of $p$ restricted to the denoiser-based family, denoted $m^{\star}$, is equivalently defined by (a) $m^{\star}_0=\mu$ and forward kernels \eqref{tsb:eq:markov_param_denoiser_f} parameterized by ${x}^{\theta^\rmf_\star}_{N+1}$ or (b) $m^{\star}_{N+1}=\nu$ and backward kernels \eqref{tsb:eq:markov_param_denoiser_b} parameterized by ${x}^{\theta^\rmb_\star}_{0}$.
\end{proposition}
In practice, ${x}^{\theta^\rmf}_{N+1}$ and ${x}^{\theta^\rmb}_{0}$ are implemented as neural networks. In the case $g_n\equiv1$, we recover exactly the classical DDPM objective for the backward loss \eqref{tsbm:eq:markov_b_loss_denoiser}, as expected.

\paragraph{Link with continuous-time setting.}

Consider the denoiser-paramaterization inherited from the canonical diffusion setting, defined by
\begin{align*}
    v_t^{\theta^\rmf}(x)= \frac{{x}_T^{\theta^\rmf}(t, x)-x}{T-t} \eqsp , \quad  v_t^{\theta^\rmb}(x)= \frac{{x}_0^{\theta^\rmb}(t, x)-x}{t} \eqsp ,
\end{align*}
where ${x}_T^{\theta^\rmf}$ and ${x}_0^{\theta^\rmb}$ are predictors of the terminal state $x_T$ and the initial state $x_0$, respectively, given $(t,x)$.
Under this parameterization, the TSBM Markovian objectives \eqref{tsb:eq:obj_f} and \eqref{tsb:eq:obj_b} admit the simplified forms
\begin{align}
     &\textstyle\mathcal{L}_{\mathcal{M}_T}^\rmf(\theta^\rmf) =\int_0^T (T-t)^{-1}\int_t^T \PE_{\Pi_{t,s,T}}\left[\norm{-(T-s)\nabla V_{s}(\bfX_s) + \frac{1}{T-t}\left(\bfX_T - {x}_T^{\theta^\rmf}(t, \bfX_t)\right)}^2\right] \eqsp \rmd s\rmd t \eqsp , \label{tsb:eq:obj_f_2}\\
     &\textstyle\mathcal{L}_{\mathcal{M}_T}^\rmb(\theta^\rmb)=\int_0^T t^{-1}\int_0^t \PE_{\Pi_{0,s,t}}\left[\norm{-s\nabla V_{s}(\bfX_s) +\frac{1}{t}\left(\bfX_0 - {x}_0^{\theta^\rmb}(t, \bfX_t)\right)}^2\right] \eqsp \rmd s\rmd t \eqsp .\label{tsb:eq:obj_b_2}
\end{align}

When setting $g_n=\exp(-\delta_n V_{t_n}/\sigma^2)$ to match the continuous-time formulation (see \Cref{tsbm:prop:equivalence_cont_disc}) and taking $\delta_n = o(1)$, the following approximations hold
\begin{align*}
&A^\rmf_n \sigma^2 (T-t_k)\nabla \log g_k(X_k) \approx - \sqrt{\delta_n}(T-t_k)\nabla V_{t_k}(X_k)  \eqsp , \quad B^\rmf_n \approx \frac{\sqrt{\delta_n}}{T-t_n} \eqsp ,\\
&A_{n}^\rmb \sigma^2 t_k\nabla \log g_k(X_k) \approx - \sqrt{\delta_n}t_k\nabla V_{t_k}(X_k) \eqsp , \quad B^\rmb_n \approx \frac{\sqrt{\delta_n}}{t_n} \eqsp .
\end{align*}
Hence, it is clear that the denoiser-like TSBM objectives \eqref{tsb:eq:obj_f_2} and \eqref{tsb:eq:obj_b_2} can respectively be seen as the time-continuous limit of objectives \eqref{tsbm:eq:markov_f_loss_denoiser} and \eqref{tsbm:eq:markov_b_loss_denoiser}.
Note also, that in this case, the regression targets become independent of $\sigma$ (with suitable choices of $\sigma_0$ and $\sigma_{N+1}$), thereby connecting to the deterministic regime.

\section{Proofs for D-TSBM} \label{tsbm:app:discrete_proofs}

\subsection{Results on twisted measures}

In this section, we present general results concerning the Markovian decomposition of the twisted bridge $q^{g,\sigma}_{|0,N+1}$ relative to the discrete-time Brownian bridge $q^\sigma_{|0,N+1}$.
In particular, we derive closed-form expressions for the corresponding quantities. Although explicit, these expressions remain computationally intractable in practice.

\begin{lemma}\label{tsbm:lemma:q^g_markov}
    The bridge $q^{g,\sigma}_{1:N|0, N+1}\in \mathcal{M}_{N-2}$ admits the following forward Markov decomposition
$$
q^{g,\sigma}_{1:N|0,N+1}(x_{1:N}|x_0, x_{N+1})= \prod_{n=1}^{N} q^{g,\sigma}_{n|n-1,N+1}(x_n|x_{n-1}, x_{N+1})
$$
where for any $n\in \{1, \hdots, N\}$, we have
    \begin{align*}
        q^{g,\sigma}_{n|n-1,N+1}(x_n|x_{n-1}, x_{N+1}) = \frac{h_n(x_n|x_{N+1})g_n(x_n)}{h_{n-1}(x_{n-1}| x_{N+1})}q^\sigma_{n|n-1,N+1}(x_n|x_{n-1}, x_{N+1})
    \end{align*}
    with
    \begin{align*}
    h_n(x_n|x_{N+1})&=\int_{(\rset^d)^{N-n}}\prod_{k=n+1}^{N}g_k(x_k) \rmd q^\sigma_{k|k-1,N+1}(x_k|x_{k-1},x_{N+1}) \eqsp .
    \end{align*}
\end{lemma}

\begin{lemma}\label{tsbm:lemma:q^g_markov_detail}
    For any $n\in \{1,\hdots, N\}$, the score of the forward bridge kernel $q^{g,\sigma}_{n|n-1,N+1}$ is given by
    \begin{align*}
        & \nabla \log q^{g,\sigma}_{n|n-1,N+1}(x_n|x_{n-1}, x_{N+1}) = \nabla \log q^\sigma_{n|n-1,N+1}(x_n|x_{n-1}, x_{N+1})\\
        & \quad + \sum_{k=n}^N \mathbb{E}_{q^{g,\sigma}_{k|n, N+1}}\left[\frac{T-t_k}{T-t_n}\nabla \log g_k(X_k)|X_n=x_n, X_{N+1}=x_{N+1}\right]
    \end{align*}
\end{lemma}
\begin{proof}
    Since we have for any $n\in \{1,\hdots, N\}$,
    \begin{align*}
    \nabla \log q^{g,\sigma}_{n|n-1,N+1}(x_n|x_{n-1}, x_{N+1}) &= \nabla \log q^\sigma_{n|n-1,N+1}(x_n|x_{n-1}, x_{N+1})\\
    & \quad +  \nabla \log h_n(x_n|x_{N+1}) + \nabla \log g_n(x_n) \eqsp ,
    \end{align*}
    proving this result amounts to prove the identity
    \begin{align}
    \nabla \log h_n(x_n|x_{N+1}) = \sum_{k=n+1}^N \mathbb{E}_{q^{g,\sigma}_{k|n, N+1}}\left[\frac{T-t_k}{T-t_n}\nabla \log g_k(X_k)|X_n=x_n, X_{N+1}=x_{N+1}\right] \eqsp . \label{tsbm:eq:recursion}
    \end{align}
    We prove this equivalent result via backward recursion on $n$. First note that this result stands trivially for $n=N$, since
    $h(x_N|x_{N+1})=1$, and therefore, $\nabla \log h_N(x_N|x_{N+1})=0$, validating the result.
    Assume \eqref{tsbm:eq:recursion} is verified at rank $(n+1)$ for some $n\in \{1, \hdots, N-1\}$. At rank $n$, we obtain by \Cref{tsbm:lemma:q^g_markov}
     the following recursion relation
    $$
    h_n(x_n|x_{N+1})= \int h_{n+1}(x_{n+1}|x_{N+1})g_{n+1}(x_{n+1}) \rmd q^\sigma_{n+1|n,N+1}(x_{n+1}|x_{n}, x_{N+1}) \eqsp .
    $$
    Hence, using the reparameterization trick in the discrete Brownian bridge $q^\sigma_{n+1|n,N+1}$ along with Tweedie formula, we obtain
    \begin{align*}
    \nabla \log h_n(x_n|x_{N+1}) &= \frac{T-t_{n+1}}{T-t_n}\mathbb{E}_{q^{g,\sigma}_{n+1|n,N}}\Bigg[ \nabla \log h_{n+1}(X_{n+1}|x_{N+1}) \\
    &  \qquad \qquad  \qquad  \qquad  \qquad + \nabla \log g_{n+1}(X_{n+1}) \Bigg|X_n=x_n, X_{N+1}=x_{N+1} \Bigg] \eqsp ,\\
    & = \frac{T-t_{n+1}}{T-t_n}\sum_{k=n+2}^N \mathbb{E}_{q^{g,\sigma}_{k|n, N+1}}\left[\frac{T-t_k}{T-t_{n+1}}\nabla \log g_k(X_k)|X_n=x_n, X_{N+1}=x_{N+1}\right]\\
    & \quad + \mathbb{E}_{q^{g,\sigma}_{n+1|n, N+1}}\left[\frac{T-t_{n+1}}{T-t_n}\nabla \log g_{n+1}(X_{n+1})|X_n=x_n, X_{N+1}=x_{N+1}\right] \eqsp ,
    \end{align*}
    where the simplification comes from the recursion assumption from rank $(n+1)$. Hence, we obtain the result at rank $n$, which proves the result for any $n\in \{1,\hdots, N\}$.
\end{proof}

\begin{lemma}\label{tsbm:lemma:q^g_markov_b}
    Symmetrically, $q^{g,\sigma}_{1:N|0, N+1}\in \mathcal{M}_{N-2}$ admits the following backward Markov decomposition
    $$
    q^{g,\sigma}_{1:n|0,N+1}(x_{1:N}|x_0, x_{N+1})= \prod_{n=2}^{N+1} q^{g,\sigma}_{n-1|0,n}(x_{n-1}|x_{0}, x_{n})
    $$
    where for any $n\in \{2, \hdots, N+1\}$
    \begin{align*}
        q^{g,\sigma}_{n-1|0,n}(x_{n-1}|x_{0}, x_{n}) = \frac{h_{n-1}(x_{n-1}|x_{0})g_{n-1}(x_{n-1})}{h_{n}(x_{n}| x_{0})}q^\sigma_{n-1|0,n}(x_{n-1}|x_{0}, x_{n}) \eqsp ,
    \end{align*}
    where
    \begin{align*}
        h_n(x_n|x_{0})= \int_{(\rset^d)^{n-1}}\prod_{k=2}^{n}g_{k-1}(x_{k-1}) \rmd q^\sigma_{k-1|0,k}(x_{k-1}|x_{0},x_{k}) \eqsp .
    \end{align*}
\end{lemma}

\begin{lemma}\label{tsbm:lemma:q^g_markov_detail_b}
    For any $n\in \{2, \hdots, N+1\}$, the score of the forward bridge kernel $q^{g,\sigma}_{n-1|0,n}$ is given by
    \begin{align}
        & \nabla \log q^{g,\sigma}_{n-1|0,n}(x_{n-1}|x_{0}, x_{n}) = \nabla \log q^\sigma_{n-1|0,n}(x_{n-1}|x_{0}, x_{n})\\
        & \quad + \sum_{k=1}^{n-1} \mathbb{E}_{q^{g,\sigma}_{k|0, n-1}}\left[\frac{t_k}{t_{n-1}}\nabla \log g_k(X_k)|X_0=x_0, X_{n-1}=x_{n-1}\right]
    \end{align}
\end{lemma}
\begin{proof}
    The proof structure follows the same recursion principle as for \Cref{tsbm:lemma:q^g_markov_detail}, but with increasing $n$, and exploiting the formula from \Cref{tsbm:lemma:q^g_markov_b}.
\end{proof}

\subsection{Results on the discrete Twisted SB problem}

\begin{proposition}[Formal version of \Cref{tsbm:prop:equivalence_cont_disc}]
Let $V\in \rmC^{0,1}([0,T]\times \rset^d,\rset)$ be such that $Z_V<\infty$. Define the potentials $g_n=\exp(-\delta_n V_{t_n}/\sigma^2)$ for any $n\in\{1,\hdots,N\}$.
Denote by $\Pbb^N$ the solution to the continuous-time SB problem \eqref{tsbm:eq:classic_sb} with reference measure given by
$$
\rmd\Qbb^{g,N}(x_{[0:T]})=\frac{1}{Z_N} \prod_{n=1}^N g_n(x_{t_n}) \rmd \Qbbsg(x_{[0:T]}) \eqsp , \quad \text{where} \quad Z_N= \int \prod_{n=1}^N g_n(x_{t_n}) \rmd \Qbbsg(x_{[0:T]}) \eqsp .
$$
Assume that $\Pbb^{N}$ converges to a certain $\Pbb^{\infty}\in \Pmeasure(\setFunConT)$ as $N \to \infty$ in the weak topology.
Then, $\Pbb^{\infty}$ is solution to \eqref{tsbm:eq:classic_sb} with reference measure $\QbbVsg$.
\end{proposition}

\begin{proof} Denote by $\Pbb^{\star}$ the solution to \eqref{tsbm:eq:classic_sb} with reference measure $\QbbVsg$.
    First note that for any $\Pbb\in \Pmeasure(\setFunConT)$,
    \begin{align}
        \PE_{\Pbb}[\log\frac{\rmd \QbbVsg}{\rmd\Qbb^{g,N}}]= \PE_{\Pbb}\left[\exp\left(\sum_{n=0}^{N}\int_{t_n}^{t_{n+1}} \frac{V_{t_{n+1}}(\bfX_{t_{n+1}})-V_{t}(\bfX_t)}{\sigma^2}\rmd t\right)\right] + \log \frac{Z_N}{Z_V} \to 0
    \end{align}
     using the dominated convergence theorem. Then, we have
    \begin{align*}
        \KL(\Pbb^{\infty}\|\QbbVsg) & \leq \lim \inf \KL(\Pbb^{N}\|\Qbb^{g,N}) & \text{ (KL property)}\\
        & \leq \lim \inf \KL(\Pbb^{\star}\|\Qbb^{g,N}) & \text{ (definition of $\Qbb^{g,N}$)}\\
        &= \KL(\Pbb^{\star}\|\QbbVsg) + \lim \inf \PE_{\Pbb^\star}[\log\frac{\rmd \QbbVsg}{\Qbb^{g,N}}]\\
        &= \KL(\Pbb^{\star}\|\QbbVsg) \eqsp .
    \end{align*}
    Then, we conclude by uniqueness of $\Pbb^{\star}$ that $\Pbb^{\infty}=\Pbb^{\star}$.
\end{proof}

\subsection{Results on the Markovian projection}

\begin{proof}[Proofs of \Cref{tsbm:lemma:markov_proj} and \Cref{tsbm:lemma:markov_proj_b}]
    These results follow from the observation that any $p\in \mathcal{R}(q^{g,\sigma})$ writes
    \begin{align*}
        p_{0:N+1}&= p_{0,N+1}q^{g,\sigma}_{1:N|0,N+1}\\
                 & = p_0 p_{N+1|0} \prod_{n=1}^N q^{g,\sigma}_{n|n-1, N+1} && \text{(forward decomposition)}\\
                 & = p_{N+1} p_{0|N+1} \prod_{n=2}^{N+1} q^{g,\sigma}_{n-1|0, n}  && \text{(backward decomposition)} \eqsp ,
    \end{align*}
    using the Markov structure of the twisted bridge, see \Cref{tsbm:lemma:q^g_markov,tsbm:lemma:q^g_markov_b}.
    Then, we obtain the KL decompositions by applying the KL chain rule.
\end{proof}

\begin{proof}[Proof of \Cref{tsbm:prop:markov_denoiser}]
    In this proof, we solely consider the forward setting, since the backward result can be proved with very similar tools.
    Let $\theta^\rmf \in \Theta$, and denote by $m^{\theta^\rmf}\in \mathcal{M}_N$ the Markovian measure obtained via forward denoiser-based parameterization
    \eqref{tsb:eq:markov_param_denoiser_f} with $\theta^\rmf$. In this case, we have for any $n\in \{1, \hdots, N\}$, almost surely any $(x_{n-1}, x_{N+1})\sim p_{n-1,N+1}$,
    up to additive constants independent from $\theta^\rmf$,
    \begin{align*}
        &\KL(q^{g,\sigma}_{n|n-1,N+1}(\cdot|x_{n-1}, x_{N+1})\|m^{\theta^\rmf}_{n|n-1}(\cdot|x_{n-1})) \\
        &\quad \overset{c}{=} \frac{1}{2\sigma^2(\gamma_n^f)^2}\mathbb{E}\left[\norm{X_{n}-\alpha^\rmf_n x_{n-1}-\beta_n^\rmf \tilde{x}_{N+1}}^2\right]\\
        & \quad \overset{c}{=} \frac{1}{2}\mathbb{E}\left[\norm{-\sigma \gamma_n^\rmf \nabla \log q^\sigma_{n|n-1,N+1}(X_n|x_{n-1}, x_{N+1})+\frac{\beta_n^\rmf}{\sigma \gamma_n^\rmf} (x_{N+1}-\tilde{x}_{N+1})}^2\right]
    \end{align*}
    where $\tilde{x}_{N+1}= {x}^{\theta^\rmf}_{N+1}(t_{n-1}, x_{n-1})$, and the expectation is taken with respect to $q^{g,\sigma}_{n|n-1,N+1}(\cdot|x_{n-1}, x_{N+1})$.
    Using the result from \Cref{tsbm:lemma:q^g_markov_detail}, combined with the fact that
    $$
        \mathbb{E}_{q^{g,\sigma}_{n|n-1,N+1}}[\nabla \log q^{g,\sigma}_{n|n-1,N+1}(X_n|x_{n-1}, x_{N+1})]=0 \eqsp,
    $$
    we obtain, up to additive constants independent from $\theta^f$,
    \begin{align*}
        &\KL(q^{g,\sigma}_{n|n-1,N+1}(\cdot|x_{n-1}, x_{N+1})\|m^{\theta^\rmf}_{n|n-1}(\cdot|x_{n-1})) \\
        & \textstyle \quad \overset{c}{=} \frac{1}{2}\mathbb{E}\left[\norm{\sigma\gamma_n^\rmf \sum_{k=n}^N \mathbb{E}\left[\frac{T-t_k}{T-t_n}\nabla \log g_k(X_k)|X_n=x_n, X_{N+1}=x_{N+1}\right]+\frac{\beta_n^\rmf}{\sigma\gamma_n^\rmf} (x_{N+1}-\tilde{x}_{N+1})}^2\right]\\
        & \textstyle \quad \overset{c}{=} \frac{1}{2 (N-n+1)\sigma^2}\sum_{k=n}^N \mathbb{E}\left[\norm{\sigma^2\gamma_n^\rmf (N-n+1)\frac{T-t_k}{T-t_n}\nabla \log g_k(X_k)+\frac{\beta_n^\rmf}{\gamma_n^\rmf} (x_{N+1}-\tilde{x}_{N+1})}^2\right] \eqsp .
    \end{align*}
    By deriving this identity for all $n\in \{1, \hdots, N\}$, and combining the results, we conclude that
    $$
        \KL(p\|m^{\theta^\rmf})= \frac{1}{2\sigma^2}\mathcal{L}^\rmf_{\mathcal{M}_N}(\theta^\rmf) + \mathrm{C}^\rmf \eqsp ,
    $$
    where $ \mathrm{C}^\rmf$ is a constant independent from $\theta^\rmf$. This equivalence shows that minimizing $\mathcal{L}_{\mathcal{M}_N}^\rmf(\theta^\rmf)$ ensures that $\theta^\rmf_\star$ yields
    the forward Markovian projection of $p$ within the Markov measures with denoiser-like parameterization.
\end{proof}

\begin{algorithm}[h!]
\caption{Twisted Schrödinger Bridge Matching (TSBM) - Continuous time state cost}
\label{alg:tsbm}
{
\algrenewcommand\algorithmicindent{0.5em}
\begin{algorithmic}
 \Input source dataset $\mu$, target dataset $\nu$, regularization $\sigma>0$ , number of outer iterations $N\in \nset$
\State Start from the independent coupling $\hat{\Pbb}_{0,T}=\mu \otimes \nu$
\For{$n = 0$ to $N-1$}
    \State Set $\operatorname{dir}=\rmf$ if $n\mod 2=0$ else $\operatorname{dir}=\rmb$
    \State Learn $\Qbb^{\psi_\star}_{|0,T}$ by running $\texttt{ReciprocalProj}(\operatorname{dir})$, see \Cref{alg:tsbm_reciprocal} \Comment{Reciprocal projection}
    \State Learn $v^{\theta_\star^{\operatorname{dir}}}$ by running $\texttt{MarkovProj}(\operatorname{dir}, \psi_\star)$, see \Cref{alg:tsbm_markov} \Comment{Markovian projection}
    \State \textbf{if} $\operatorname{dir}=\rmf$
    \State\quad \textbf{then} Solve $\rmd \hbfX^{\rmf}_t= v^{\theta^\rmf_\star}_t(\hbfX^{\rmf}_t)\rmd t + \sigma \rmd \bfB_t, \hbfX^{\rmf}_0\sim \hat{\Pbb}_0$ and set $(\hat{\bfX}_0,\hat{\bfX}_T)=(\hbfX^{\rmf}_0,\hbfX^{\rmf}_T)$
    \State \textbf{else} Solve $\rmd \hbfX^{\rmb}_t= v^{\theta^\rmb_\star}_{T-t}(\hbfX^{\rmb}_t)\rmd t + \sigma \rmd \bfB_t, \hbfX^{\rmb}_0\sim \hat{\Pbb}_T$ and set $(\hat{\bfX}_0,\hat{\bfX}_T)=(\hbfX^{\rmb}_T,\hbfX^{\rmb}_0)$
    \State Update the coupling as $\hat{\Pbb}_{0,T}=\law(\hat{\bfX}_0,\hat{\bfX}_T)$
\EndFor
\Output Approximate solutions $(\hbfX^\rmf_t)_{t\in[0,T]}$ and $(\hbfX^\rmb_t)_{t\in[0,T]}$ to \eqref{tsbm:eq:extended_soc}, respectively in the forward and backward time directions
\end{algorithmic}
}
\end{algorithm}

\begin{figure}[h!]
\centering
\begin{minipage}{0.44\textwidth}
\begin{algorithm}[H]
\caption{$\texttt{ReciprocalProj}(\operatorname{dir})$ }
\label{alg:tsbm_reciprocal}
{
\algrenewcommand\algorithmicindent{0.7em}
\begin{algorithmic}
 \Input $\operatorname{dir}\in \{\rmf, \rmb\}$, coupling $\hat{\Pbb}_{0,T}$ 
\Repeat
\State Sample $(\bfX_0,\bfX_T)\sim \hat{\Pbb}_{0,T}$, $t\sim \mathrm{U}(0,T)$
\State Sample $\bfX_t\sim \Qbb^\psi_{t|0,T}(\cdot|\bfX_0,\bfX_T)$
\State Compute $v^{\operatorname{dir},\psi}$, see \Cref{tsbm:app:var_reciprocal}
\State Minimize $\mathcal{L}_{\mathcal{R}}^{\operatorname{dir}}(\psi)$, see \eqref{tsbm:eq:objective_reciprocal_f}--\eqref{tsbm:eq:objective_reciprocal_b_app}
\Until{it converges}
\Output learned bridge $\Qbb^{\psi_\star}_{|0,T}\approx \QQVsg_{|0,T}$
\vspace{0.09cm}
\end{algorithmic}
}
\end{algorithm}

\end{minipage}
\hfill
\begin{minipage}{0.55\textwidth}
\vspace{0pt}
\begin{algorithm}[H]
\caption{$\texttt{MarkovProj}(\operatorname{dir}, \psi_\star)$}
\label{alg:tsbm_markov}
{
\algrenewcommand\algorithmicindent{0.7em}
\begin{algorithmic}
\Input $\operatorname{dir}\in \{\rmf, \rmb\}$, coupling $\hat{\Pbb}_{0,T}$, bridge $\Qbb^{\psi_\star}_{|0,T}$
\Repeat
    \State Sample $(\bfX_0,\bfX_T)\sim \hat{\Pbb}_{0,T}$, $t\sim \mathrm{U}(0,T)$
    \State Sample $s\sim \mathrm{U}(t,T)$ if $\operatorname{dir}=\rmf$ else $s\sim \mathrm{U}(0,t)$
    \State Sample $(\bfX_t,\bfX_s)\sim \Qbb^{\psi_\star}_{t,s|0,T}(\cdot|\bfX_0,\bfX_T)$
    \State Minimize $\mathcal{L}_{\mathcal{M}_T}^{\operatorname{dir}}(\theta^{\operatorname{dir}})$, see \eqref{tsb:eq:obj_f}--\eqref{tsb:eq:obj_b}
\Until{it converges}
\Output learned $v^{\theta_\star^{\operatorname{dir}}}$ for $\projM(\hat{\Pbb}_{0,T}\Qbb^{\psi_\star}_{|0,T})$
\end{algorithmic}
}
\end{algorithm}

\end{minipage}

\end{figure}

\newpage
\section{Practical implementation of TSBM}
\label{tsbm:app:practical}

\subsection{Pseudo-code for TSBM}\label{tsbm:app:pseudo}

We provide the pseudo-code of the TSBM methodology in \Cref{alg:tsbm} (for the continuous time version), further detailing the reciprocal projection step in \Cref{alg:tsbm_reciprocal} and the Markovian projection step in \Cref{alg:tsbm_markov}.

\subsection{Variational approximation of the reciprocal projection} \label{tsbm:app:var_reciprocal}

In this section, we detail the construction of the variational family used in \Cref{tsbm:subsec:reciprocal_cont}, relying on the stochastic interpolants introduced by \citep{albergo2025stochastic}.
We assume access to conditioning pairs $(x_0,x_T)\sim \Pbb_{0,T}$ and emphasize that the families described below can be straightforwardly parallelized across multiple conditioning pairs in large data batches.

\paragraph{Gaussian variational family.}
We say that the stochastic process $\Pbb_{|0,T}(\cdot |x_0, x_T)$ is a Gaussian stochastic interpolant between $x_0$ and $x_T$ \citep{albergo2025stochastic} if its marginals satisfy
\begin{align*}
& \Pbb_{t|0,T}(\cdot |x_0, x_T)= \gauss(I_t, \gamma_t^2 \Idd) \eqsp,\\
\text{where } & I\in \mathcal{I}(x_0,x_T)=\{I\in \mathrm{C}^{1}([0,T], \rset^d) \eqsp : \eqsp I_0= x_0 \eqsp , I_T = x_T\} \eqsp ,\\
& \gamma\in \Gamma=\{\gamma\in \mathrm{C}^{1}([0,T], ( 0, \infty)) \eqsp : \eqsp \gamma_0 =\gamma_T= 0\} \eqsp .
\end{align*}
We denote by $\operatorname{SI}(x_0,x_T)$ the resulting class of stochastic interpolants.

For any $I \in \mathcal{I}(x_0,x_T)$ and $\gamma \in \Gamma$, we define the time-dependent velocity fields
$(t,x) \mapsto v^{\rmf}_t(x|x_0,x_T)$ (forward direction) and
$(t,x) \mapsto v^{\rmb}_t(x|x_0,x_T)$ (backward direction) by
\begin{align*}
    v^{\rmf}_t(x|x_0, x_T)  &=\partial_t I_t + \left(\frac{\partial_t \gamma_t}{\gamma_t}-\frac{\sigma^2}{2\gamma_t^2} \right)(x-I_t) \eqsp ,\\ v^{\rmb}_t(x|x_0, x_T) &=-\partial_t I_t + \left(-\frac{\partial_t \gamma_t}{\gamma_t}-\frac{\sigma^2}{2\gamma_t^2} \right)(x-I_t) \eqsp .
\end{align*}
As shown in \cite{albergo2025stochastic}, the SDEs
\begin{align}
    \rmd \bfX^{0,T}_t &= v^{\rmf}_t (\bfX^{0,T}_t|x_0, x_T) \rmd t + \sigma \rmd \bfB_t \eqsp, \bfX^{0,T}_0=x_0 \eqsp ,\label{tsb:eq:sde_cond_f} \\
    \rmd \bfY^{0,T}_t &= v^{\rmb}_{T-t} (\bfY^{0,T}_t|x_0, x_T) \rmd t + \sigma \rmd \bfB_t \eqsp ,\bfY^{0,T}_0=x_T \eqsp ,\label{tsb:eq:sde_cond_b}
\end{align}
induce identical path measures up to time reversal, and share the same marginals as $\Pbb_{|0,T}(\cdot|x_0,x_T)$. In particular, the boundary conditions are satisfied almost surely, namely $\bfX^{0,T}_T=x_T$ and $\bfY^{0,T}_T=x_0$.

Motivated by this characterization, we define our variational family as a parametric subset of $\operatorname{SI}(x_0,x_T)$ of the form
$$
\Qbb^\psi_{|0,T}(\cdot |x_0, x_T)= \gauss(I_t^\psi, (\gamma_t^\psi)^2 \Idd)  \eqsp ,
$$
where the mean path $I^\psi \in \mathcal{I}(x_0,x_T)$ and the standard deviation $\gamma^\psi \in \Gamma$ are parameterized by $\psi\in\Psi$.
In this case, $\Qbb^\psi_{|0,T}(\cdot |x_0, x_T)$ is associated with forward SDE \eqref{tsb:eq:sde_cond_f} and backward SDE \eqref{tsb:eq:sde_cond_b} where
\begin{align*}
    v^{\rmf, \psi}_t(x|x_0, x_T) &=\partial_t I_t^\psi + \left(\frac{\partial_t \gamma_t^\psi}{\gamma_t^\psi}-\frac{\sigma^2}{2(\gamma_t^\psi)^2} \right)(x-I_t^\psi) \eqsp ,\\
    v^{\rmb, \psi}_t(x|x_0, x_T) &=-\partial_t I_t^\psi + \left(-\frac{\partial_t \gamma_t^\psi}{\gamma_t^\psi}-\frac{\sigma^2}{2(\gamma_t^\psi)^2} \right)(x-I_t^\psi) \eqsp .
\end{align*}

\paragraph{Expression of the transition kernels.} The SDEs \eqref{tsb:eq:sde_cond_f} and \eqref{tsb:eq:sde_cond_b} can be integrated exactly as follows. The forward and backward transition kernels are Gaussian distributions, respectively given by
\begin{align}
    \Qbb^\psi_{s|t,T}(\cdot |x_t, x_T)&= \gauss\left( e^{-\frac{\sigma^2}{2}J^\psi(t,s)}\frac{\gamma_s^\psi}{\gamma_t^\psi}(x_t -I_t^{\psi}) + I_s^{\psi}, (\gamma_s^\psi)^2 \left[1-e^{-\frac{\sigma^2}{2}J^\psi(t,s)}\right] \right) \eqsp , \eqsp s \in [t,T] \eqsp ,\label{tsbm:eq:forward_q_kernel}\\
    \Qbb^\psi_{s|0,t}(\cdot |x_0, x_t)&= \gauss\left( e^{-\frac{\sigma^2}{2}J^\psi(s,t)}\frac{\gamma_s^\psi}{\gamma_t^\psi}(x_t -I_t^{\psi}) + I_s^{\psi}, (\gamma_s^\psi)^2 \left[1-e^{-\frac{\sigma^2}{2}J^\psi(s,t)}\right] \right) \eqsp , \eqsp s \in [0,t] \eqsp ,\label{tsbm:eq:backward_q_kernel}
\end{align}
where $J^\psi(r_1,r_2)=\int_{r_1}^{r_2} \frac{1}{(\gamma^\psi_u)^2}\rmd u$. For arbitrary $\gamma^\psi$, $J^\psi$ is however not tractable.

\paragraph{Combination with spline parameterization.}

Following \cite{liu2023generalized}, we parameterize the mean path $I^\psi$ as a linear $d$-dimensional spline and the standard deviation $\gamma^\psi$ as a linear one-dimensional spline defined through control points.
More precisely, we fix $C\geq 1$ intermediate control times $\{t_i\}_{i=1}^C$ in $(0,T)$ such that
\[
0 < t_1 < t_2 < \hdots < t_C < T \eqsp .
\]
For a given conditioning pair $(x_0,x_T)$, we define
\begin{align*}
    I_t^\psi= \text{Linear Spline}(t; x_0, \{X_{t_i}\}_{i=1}^C, x_T),
    \qquad
    \gamma_t^\psi = \text{Linear Spline}(t; 0, \{\gamma_{t_i}\}_{i=1}^C, 0) \eqsp ,
\end{align*}
where the control points $\{X_{t_i}\}_{i=1}^C\subset \rset^d$ and $\{\gamma_{t_i}\}_{i=1}^C \subset (0,\infty)$ are the parameters to optimize, which is much more lightweight than a neural network parameterization.
In practice, the number of control points $C$ can be chosen significantly smaller than the number of discretization steps used to simulate the SDEs. This yields a favorable trade-off between expressive power of the interpolant and computational efficiency. The second advantage of this method is that we can leverage samples generated at intermediary timesteps between endpoints when updating the coupling $\Pbb_{0,T}$ via SDE simulation, allowing for well-informed initialization of the spline at every reciprocal projection step (except the first where we initialize with the linear interpolant). Similarly, we initialize the standard deviation with the one from the Brownian bridge $\Qbbsg_{|0,T}$ at the very first step, and save the learned control standard deviations for initialization of future reciprocal steps.  

\paragraph{Optimizing the sampling computational cost.}
In practice, we first perform the variational optimization over the full training dataset of endpoint pairs $(x_0,x_T)\sim\Pbb_{0,T}$, before running the subsequent Markovian projection step. For each pair, we store the learned control points, namely the mean and standard deviation, both of which depend on $(x_0,x_T)$, and treat them as additional data features. Crucially, we also precompute and store the values $\{1/(\gamma^\psi_{u_\ell})^2\}_{\ell=1}^L$ on a fine grid $\{u_\ell\}_{\ell=1}^L$ of $[0,T]$. These operations are performed on large batches and benefit from efficient parallelization.

Then, during Markovian optimization, for any training pair $(x_0,x_T)$, we can immediately sample $x_t \sim \Qbb^\psi_{t\mid 0,T}(\cdot\mid x_0,x_T)$. Moreover, using the precomputed grid values to approximate $J^\psi$ by a Riemann sum, we can efficiently sample either $x_s \sim \Qbb^\psi_{s\mid t,T}(\cdot\mid x_t,x_T)$ in the forward direction, for $s\in[t,T]$, via \eqref{tsbm:eq:forward_q_kernel}, or $x_s \sim \Qbb^\psi_{s\mid 0,t}(\cdot\mid x_0,x_t)$ in the backward direction, for $s\in[0,t]$, via \eqref{tsbm:eq:backward_q_kernel}. Consequently, this additional sampling step, which is absent from GSBM and DSBM, introduces only a small computational overhead.

\paragraph{Extension to Gaussian mixture variational family.} Assume that we are provided with $M$ Gaussian stochastic interpolants described by $\{I^{\psi_m}, \gamma^{\psi_m}\}_{m=1}^M$ with associated drift functions $\{v^{\rmf,\psi_m}, v^{\rmb,\psi_m}\}_{m=1}^M$. Following \citep[Proposition 4]{du2024doob}, we may build Gaussian mixture stochastic interpolant with the following setting
\begin{align*}
    &\Qbb^{\psi}_{t|0,T}= \sum_{m=1}^M w_m \Qbb^{\psi_m}_{t|0,T} \eqsp ,\\
    &v_t^{\rmf,\psi}(x|x_0,x_T)=\sum_{m=1}^M \tilde{w}_m(x|x_0,x_T) v_t^{\rmf,\psi_m}(x|x_0,x_T)\eqsp ,\\
    & v_t^{\rmb,\psi}(x|x_0,x_T)=\sum_{m=1}^M \tilde{w}_m(x|x_0,x_T) v_t^{\rmb,\psi_m}(x|x_0,x_T)\eqsp ,\\
    \text{with }&\tilde{w}_m(x|x_0,x_T) = \frac{ w_m\Qbb^{\psi_m}_{t|0,T}(x|x_0,x_T)}{\sum_{j=1}^M w_j \Qbb^{\psi_j}_{t|0,T}(x|x_0,x_T)} \eqsp ,
\end{align*}
where $\{w_m\}_{m=1}^M$ are mixture weights, that can be parameterized as well. We leave the practical implementation of this setting for future work.

\subsection{Implementation of the score-based control variates} \label{tsbm:app:cv_param}

\paragraph{Continuous time setting.}In practice, we consider a light polynomial parameterization for the control variates given by
\begin{align}\label{tsbm:app:eq:cv_param}
    \alpha^{ \rmf}(t,s)=(T-s)\sum_{k=0}^K \phi^\rmf_k(t) \left(\frac{s-t}{T-t}\right)^k \eqsp , \eqsp \alpha^{\rmb}(t,s)=s\sum_{k=0}^K \phi^\rmb_k(t) \left(\frac{s-t}{t}\right)^k \eqsp ,
\end{align}
where $\phi^\rmf,\phi^\rmb:[0,T]\to\mathbb{R}^{K+1}$ are Fourier time-embedding neural networks with shared hyperparameter $K\geq 1$. The intuition is as follows, described for the forward direction, with the backward case being analogous. First, the control variate should have a limited effect for large values of $t$, when most of the trajectory has already been simulated; this motivates the rescaling by $(T-s)$, since $s$ is sampled in $[t,T]$. Second, for a fixed $t$, the control variate should leverage information from sufficiently far in the future to influence the estimate at time $t$; this motivates the dependence on the lag $(s-t)$, which increases as $s$ moves further away from $t$. With this parameterization, the three coefficient terms involving $\alpha^\rmf$ in \eqref{tsb:eq:target_f_cv} and $\alpha^\rmb$ in \eqref{tsb:eq:target_b_cv} remain fully tractable while relying on a lightweight parametrization.

\paragraph{Discrete time setting.}

We use the same parametrization as in the continuous-time setting for the schedules $\alpha^\rmf$ and $\alpha^\rmb$. The only difference lies in the computation of the integral terms,
$\int_{t}^{t_k} \frac{\alpha^{\rmf}(t,u)}{K_f(u)} \rmd u$ in the forward case and
$\int_{t_k}^{t} \frac{\alpha^{\rmb}(t,u)}{K_b(u)} \rmd u$ in the backward case.
Since $K_f$ and $K_b$ are discontinuous on $[0,T]$, these integrals are not directly handled as in the continuous-time setting. We therefore decompose them over the intermediate intervals $[t_i,t_{i+1}]$, on which $K_f$ and $K_b$ are constant, and compute the corresponding contributions in parallel.

\section{Experiment details and additional results}
\label{tsbm:app:xps}

\subsection{General setting}

Unless stated otherwise, all experiments are performed with $\sigma=2$. All models are trained on a single NVIDIA RTX Pro 6000 GPU, with training taking at most a few hours in all considered settings.

\paragraph{Data.}
For the endpoint marginals, both training and test datasets contain $2048$ independent samples.

\paragraph{Markovian projection training.}
Across all experiments, we use the same neural architecture for the Markovian projection, inspired by \cite{liu2023generalized}. The drift is parameterized by an MLP with four hidden layers of width $128$, Fourier time embeddings, skip connections, and SiLU activations, resulting in approximately $100{,}000$ trainable parameters. All networks are trained from scratch.

To balance expressivity and computational cost, we use a single architecture for both forward and backward directions: the two networks share the first layer as a common backbone, and then branch into two direction-specific heads. This choice is motivated by the fact that the forward and backward dynamics explore the same state space and can therefore benefit from a shared representation.

Stochastic gradient descent (SGD) is performed with Adam optimizer, with initial learning rate $3\times 10^{-4}$. In practice, the TSBM and GSBM objectives are optimized using Monte Carlo estimates of the corresponding losses. At each optimization step, we sample tuples $(\bfX_0,\bfX_T,t)$ with batch size $B=512$. Following the recommendation of \cite{Kingma2021variational}, we use a low-discrepancy sampler for $t$ to reduce the variance of the loss estimator. Time samples are drawn on $[\varepsilon,T-\varepsilon]$, with $\varepsilon=10^{-4}$, to avoid numerical issues at the boundaries. Each Markovian projection is trained for $2500$ epochs.

\paragraph{TSBM-specific hyperparameters and implementation details.}

For TSBM, we further reduce the variance of the Markovian projection loss by sampling several pairs $(s,\bfX_s)$, in practice 16, from the variational bridge conditionally on $(t,\bfX_t)$. This provides a richer exploration of the landscape induced by $\nabla V$. We also use a low-discrepancy sampler for $s$, again with boundary margin $\varepsilon=10^{-4}$.

Noisy evaluations of $\nabla V$ can increase the variance of the loss, particularly at the beginning of training. To alleviate this issue, we draw inspiration from recent warmup strategies for consistency models \cite{lu2025simplifying,sabour2025alignflowscalingcontinuoustime}, whose objectives also involve potential-gradient terms, and optimize an \emph{annealed} version of the TSBM regression target. More precisely, we rescale the $\nabla V$ contribution in the regression target by a factor that increases linearly from $0$ to $1$ until a prescribed fraction $\epsilon\in(0,1]$ of the optimization is reached. Concretely, at the $k$-th SGD epoch, the forward regression target in \eqref{tsb:eq:obj_f} is replaced by
\begin{align*}
    \mathbf{b}^{\mathrm{f},V}_{t,s}(x_t,x_s,x_T,k)
    =
    \frac{x_T-x_t}{T-t}
    -
    (T-s)
   \epsilon_k
    \nabla V_s(x_s) \eqsp , \eqsp \text{with } \epsilon_k=\min\left(\epsilon,\frac{k}{k_{\max}}\right) \eqsp ,
    \end{align*}
where $k_{\max}$ denotes the warmup horizon, chosen smaller than the total number of SGD epochs. This schedule progressively introduces the state-cost contribution into the TSBM objective. Equivalently, this can be interpreted as replacing the reference measure $\QbbVsg$ with the annealed reference
\begin{align*}
    \textstyle
    \rmd \Qbb^{V,\sigma,k}(x_{[0:T]})
    \propto
    \exp\left(
        -\int_0^T \frac{V_t(x_t)}{\sigma^2}\rmd t
    \right)^{\epsilon_k}
    \rmd \Qbbsg(x_{[0:T]}) \eqsp ,
\end{align*}
where the tilting is progressively introduced: it starts from $\epsilon_k=0$, corresponding to the Brownian reference $\Qbbsg$, and increases toward $\epsilon_k=\epsilon$, which is closer to the desired state-cost-based reference $\QbbVsg$. We emphasize that this strategy \emph{only applies during the first Markovian projection in each direction}; afterward, the original regression target is kept as such.
Importantly, we do not rescale the state cost during the spline optimization step, as this led to worse empirical performance. In our experiments, we use this strategy only for the crowd navigation tasks, with $\epsilon=0.9$ and a warmup lasting $50\%$ of the training epochs, corresponding to $k_{\max}=1250$.

Although the contributions of the $\nabla V$ terms in the TSBM loss are theoretically damped by the factors $s$ in the backward direction and $T-s$ in the forward direction, we sometimes observe undesirable boundary effects induced by $\nabla V$, which can degrade satisfaction of the marginal constraints. To mitigate this issue, we replace the TSBM regression target with the GSBM target whenever $t/T\leq 0.1$ or $t/T\geq 0.9$. This leverages the strong marginal-matching properties of GSBM while retaining the influence of $\nabla V$ over the central portion of the trajectory. In practice, this simple substitution substantially improves numerical stability, and we therefore use it by default.

\paragraph{Learnable control variates.} We use two separate networks, one for the forward direction and one for the backward direction. Following the parameterization in \eqref{tsbm:app:eq:cv_param}, each network uses an MLP with Fourier embeddings and a single hidden layer of width $128$. We set the polynomial degree to $4$ in all experiments; larger degrees did not improve performance. Each control-variate network contains approximately $15{,}000$ trainable parameters. We optimize it jointly with the drift network using Adam with learning rate $10^{-4}$. In our experiments, this joint optimization was stable and did not require any additional stabilization procedure.

\paragraph{Reciprocal projection.}
For the spline parameterization, we use the same number of control points as in \cite{liu2023generalized}, independently of the dimension: $15$ control points for the mean spline and $30$ for the standard-deviation spline. The spline parameters are optimized with Adam. At the first reciprocal projection step, we run $1500$ gradient steps; for subsequent reciprocal projections, we use $200$ gradient steps. Spline optimization is performed using batches of size $1024$ sampled from $\Pbb_{0,T}$. For each batch, we sample $100$ time points $t\in[0,T]$ to compute a Monte Carlo estimate of the reciprocal projection loss.

\paragraph{Bidirectional setting.}
For a fair comparison with the classical unidirectional variant, the pretraining stage is run for $2500$ epochs with learning rate $3\times 10^{-4}$, optimizing jointly the forward and backward objectives. The fine-tuning stage then uses a total of $10{,}000$ epochs in each time direction, matching the training budget of four Markovian projection steps in the unidirectional setting. Thus, both variants use the same total number of forward and backward training epochs. 

The main difference between the two settings lies in the frequency at which the training direction is alternated. In the bidirectional variant, we alternate more frequently to ensure regular coupling updates. In practice, every $250$ epochs, we simulate samples in the current direction and use them as the training dataset for the opposite direction. This corresponds to $10\%$ of the refresh period used in the classical procedure.

During fine-tuning, we reduce the learning rate of the drift network to $3\times 10^{-5}$, following the recommandations from \cite{bortoli2024schrodinger}. For TSBM, we also reduce the learning rate of the control variate to $10^{-5}$ during this phase.

\paragraph{Test setup.}
For evaluation and sample generation, we discretize $[0,T]$ using $200$ uniformly spaced time steps. All test metrics are computed with batch size $2048$. Following the terminology of \cite{liu2023generalized}, we report two complementary metrics to assess how accurately the Twisted Schrödinger Bridge problem \eqref{tsbm:eq:standard_soc}, or its discrete-time version \eqref{tsbm:eq:extended_soc}, is solved:
\begin{itemize}[wide, labelindent=0pt]
    \item \textbf{Feasibility}: in the forward direction, this measures the discrepancy between the generated terminal samples $\hat \bfX_T$ and ground-truth samples from the target marginal $\nu$. Conversely, in the backward direction, it measures the discrepancy between samples $\hat \bfX_0$ generated by the backward SDE and ground-truth samples from $\mu$. We compute this discrepancy using the Sinkhorn divergence implemented in the \texttt{SamplesLoss} Python package \citep{feydy2019interpolating}, with default hyperparameters.
    \item \textbf{Optimality}: this measures how well the stochastic optimal control cost associated with the Twisted Schrödinger Bridge problem is minimized. We estimate a Riemann-discretized version of this cost in expectation along the generated dynamics, either in the forward or backward direction.
\end{itemize}
The metrics reported in the main paper concern only the forward direction. For crowd navigation tasks, we complement them with their backward counterparts\footnote{We omit backward metrics for single-cell data, since the forward time direction is imposed by the experimental design.}.

\subsection{Experiments on the reciprocal projection}

For both the TSBM and GSBM objectives, we run $1500$ optimization steps on a test dataset constructed from samples of the independent coupling, and verify convergence in each setting. Following the standard spline routine used in GSBM and TSBM, we initialize the spline as the Brownian bridge, which is precisely the bridge to be recovered. Despite this favorable initialization, the GSBM loss underperforms relative to TSBM, highlighting a limitation of its reciprocal projection step.

\subsection{Toy discrete time experiments}\label{tsbm:app:toy_discrete}

In this section, we present 2D toy experiments illustrating the performance of TSBM with discrete-time state costs, whose Markovian and reciprocal projection steps are described in \Cref{tsbm:app:markovian_proofs_d} and \Cref{tsbm:app:reciprocal}, respectively. In all settings, the source and target distributions are independent Gaussian distributions. We mimic a sparse-observation scenario by defining, at each time $t_k$, a quadratic state cost of the form
\begin{align*}
    \textstyle V_k(x)
    =
    \frac{1}{M_k}
    \sum_{m=1}^{M_k}
    \frac{\|x-y_m^k\|^2}{2\gamma^2} \eqsp,
\end{align*}
where $\mathcal{Y}^k=\{y^k_m\}_{m=1}^{M_k}\subset \rset^d$ denotes the set of observed points and $\gamma>0$ is a bandwidth parameter. This cost attracts trajectories toward the observations in $\mathcal{Y}^k$ with uniform weights, while $\gamma$ controls the interaction strength. 

In the experiments below, we set $M_k=50$ at all observation times and use $\gamma=0.05$. Since the state cost is simple, we do not apply the annealing strategy for TSBM in this setting; empirically, doing so degraded convergence and led to worse numerical results. Moreover, we evaluate only the unidirectional variants of TSBM and GSBM, and use by default the control variate TSBM loss detailed in \Cref{tsbm:app:prop_cv_discrete}.

\paragraph{Single-observation setting.}
We first consider a single intermediate observation set at time $t_1=0.5$, with horizon $T=1$, source distribution $\mu=\densityGaussian((-10,0),0.2\Idd)$, and target distribution $\nu=\densityGaussian((10,0),0.5\Idd)$. The observations are sampled from $\densityGaussian((0,10),0.5\Idd)$. Although the source and target distributions are aligned along the $x$-axis, the observations are deliberately placed far away along the $y$-axis, forcing the Brownian-bridge linear interpolation to be twisted in order to satisfy the additional constraint.

We observe that TSBM and GSBM generate very similar trajectories in both forward and backward directions, as shown in \Cref{tsbm:fig:discrete_2d_gaussian}. This visual similarity is reflected in the metrics reported in the left column of \Cref{tsbm:app:table:toy_both_settings}: the feasibility scores are identical, while optimality is comparable in the forward direction and only slightly worse for TSBM in the backward direction.

\begin{figure}[h!]
  \centering
\begin{minipage}{0.49\textwidth}
    \centering
    \includegraphics[width=\linewidth]{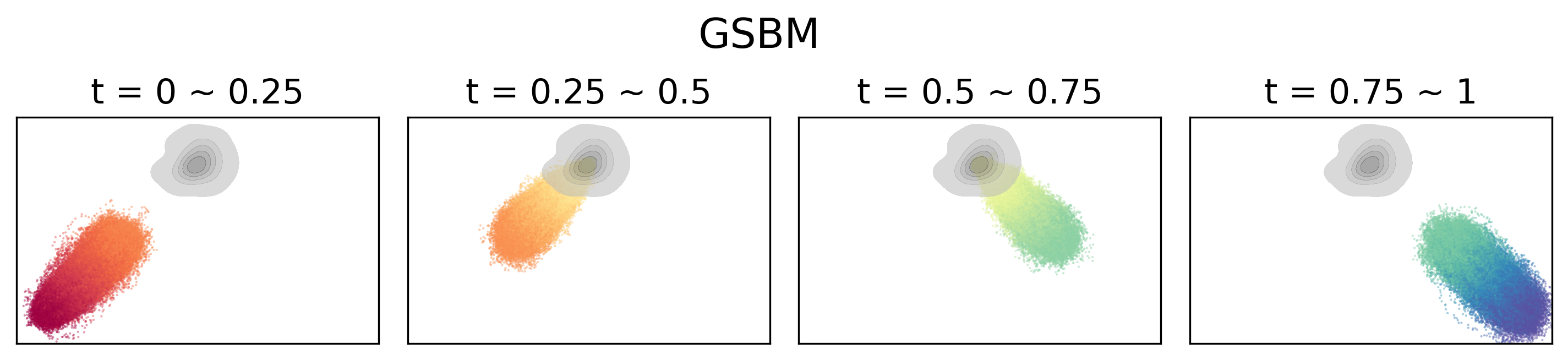}
    \includegraphics[width=\linewidth]{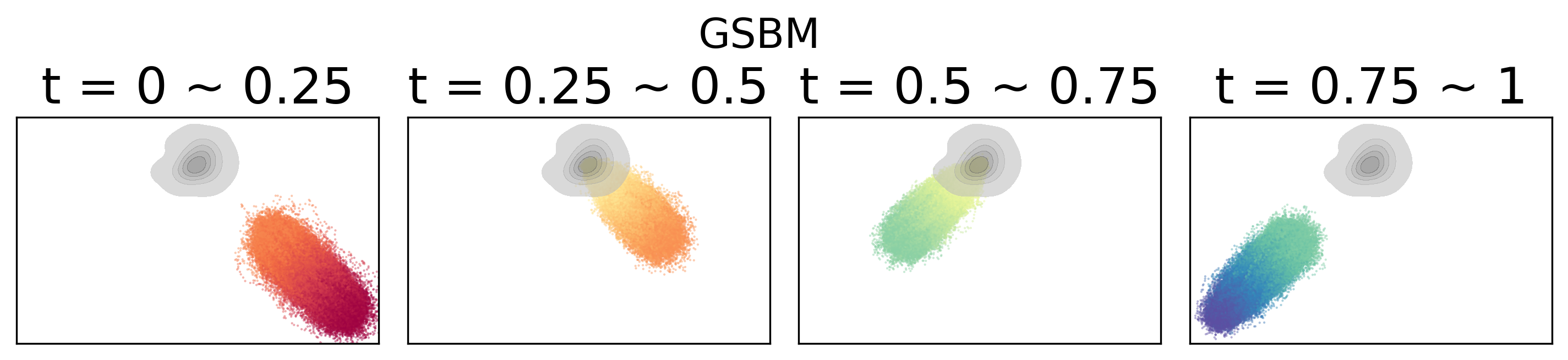}
  \end{minipage}
  \hfill
\begin{minipage}{0.49\textwidth}
    \centering
    \includegraphics[width=\linewidth]{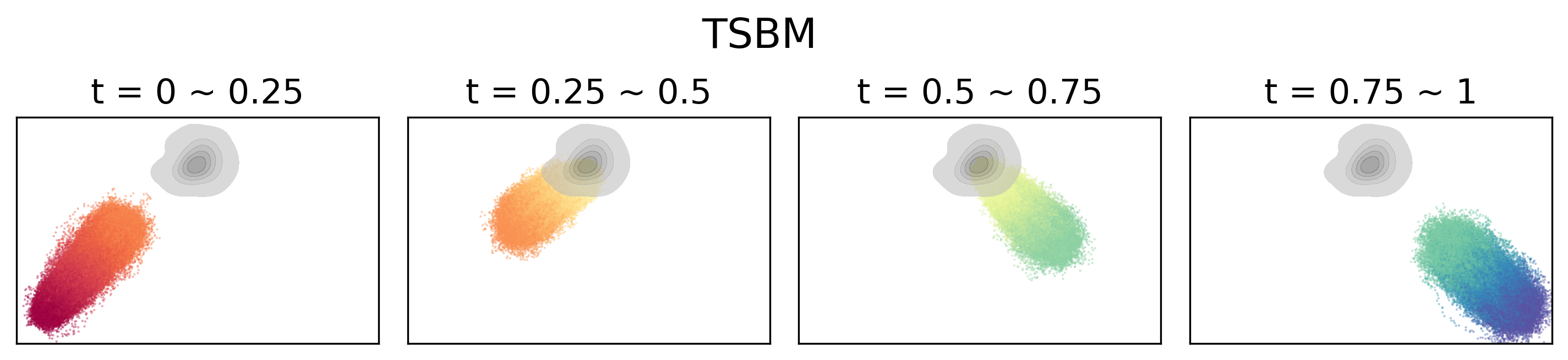}
    \includegraphics[width=\linewidth]{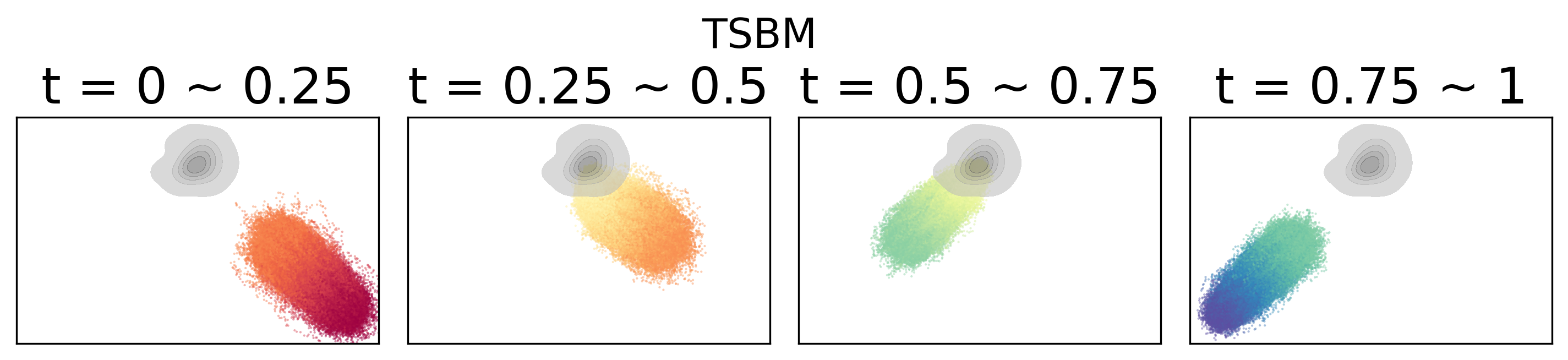}
  \end{minipage}  
  \caption{Learned TSB dynamics in the \textbf{single-observation} discrete-time setting. The intermediate observation distribution at $t_1=0.5$ is shown in grey. \textbf{Top}: forward time direction. \textbf{Bottom}: backward time direction.}
  \label{tsbm:fig:discrete_2d_gaussian}
\end{figure}

\begin{table}[h!]
\vspace{-0.4cm}
    \caption{Results of 2D toy discrete-time experiments (single- and multi-observation settings). Results are averaged over $8$ metric evaluations. Bold indicates the best result in each setting.}
    \label{tsbm:app:table:toy_both_settings}
  \resizebox{\textwidth}{!}{%
  \begin{tabular}{l|cc|cc|cc|cc}
    \toprule
     Setting & \multicolumn{4}{c}{Single-observation (full training)} & \multicolumn{4}{c}{Multi-observation (full training)}  \\
     \cmidrule(lr){1-1} \cmidrule(lr){2-5} \cmidrule(lr){6-9}
    Metric & \multicolumn{2}{c}{Feasibility ($\downarrow$)} & \multicolumn{2}{c}{Optimality ($\downarrow$)} & \multicolumn{2}{c}{Feasibility ($\downarrow$)} & \multicolumn{2}{c}{Optimality ($\downarrow$)}    \\
    \cmidrule(lr){1-1} \cmidrule(lr){2-3} \cmidrule(lr){4-5} \cmidrule(lr){6-7} \cmidrule(lr){8-9}
    Method & GSBM & TSBM     & GSBM & TSBM      & GSBM & TSBM & GSBM & TSBM \\
    \midrule
    Forward & $ \bf 0.01{\scriptstyle \pm 0.003}$  &  $ 0.02{\scriptstyle \pm 0.003}$ & $\bf 651{\scriptstyle\pm 1.0}$ & $\bf  649{\scriptstyle\pm 1.0}$  & $ \bf 0.005{\scriptstyle \pm 0.001}$ &   $  0.016{\scriptstyle \pm 0.002}$ & $\bf 636{\scriptstyle\pm 0.8}$ & $ 696{\scriptstyle\pm 1.7}$ \\
    Backward & $\bf 0.003 {\scriptstyle\pm 0.001}$  & $\bf 0.004{\scriptstyle\pm 0.001}$ & $ \bf 651 {\scriptstyle\pm 1.3}$ & $ 686{\scriptstyle\pm 1.4}$  & $\bf 0.006 {\scriptstyle\pm 0.001}$ &    $  0.010{\scriptstyle\pm 0.001}$ & $\bf 636 {\scriptstyle\pm 1.0}$ & $ 666 {\scriptstyle\pm 1.0}$  \\
    \bottomrule
  \end{tabular}
  }
\end{table}

\paragraph{Multi-observation setting.}
We next consider a more challenging setting with two intermediate observation sets at times $t_1=0.4$ and $t_2=1.6$, with horizon $T=2$, source distribution $\mu=\densityGaussian((-10,-5),0.2\Idd)$, and target distribution $\nu=\densityGaussian((10,2.5),0.2\Idd)$. The two observation sets are sampled independently from $\densityGaussian((-3,2.5),0.5\Idd)$ and $\densityGaussian((5,7.5),0.5\Idd)$, respectively. As in the single-observation setting, the observations are placed so as to induce a clear deviation from the linear interpolation between source and target distributions.

We observe that TSBM successfully steers trajectories toward the first observation set, but deviates during the second part of the dynamics and slightly underperforms GSBM, which more accurately targets the second observation set. This behavior is confirmed by the optimality and feasibility metrics reported in the right column of \Cref{tsbm:app:table:toy_both_settings}, where TSBM achieves worse scores.

In this multi-observation setting, we also perform an ablation study highlighting the crucial role of the control variate (CV) in reducing the variance of the Markovian loss. We run TSBM for a single outer iteration, comparing the default version with CV to a variant without CV. Forward and backward trajectories are shown in \Cref{tsbm:fig:discrete_2d_bigaussian_no_cv}, and the corresponding metrics are reported in \Cref{tsbm:table:discrete_2d_bigaussian_no_cv}. While both variants satisfy the marginal constraints, removing the CV substantially degrades TSBM, as reflected by the large gap in optimality.

\begin{figure}[h!]
  \centering
\begin{minipage}{0.49\textwidth}
    \centering
    \includegraphics[width=\linewidth]{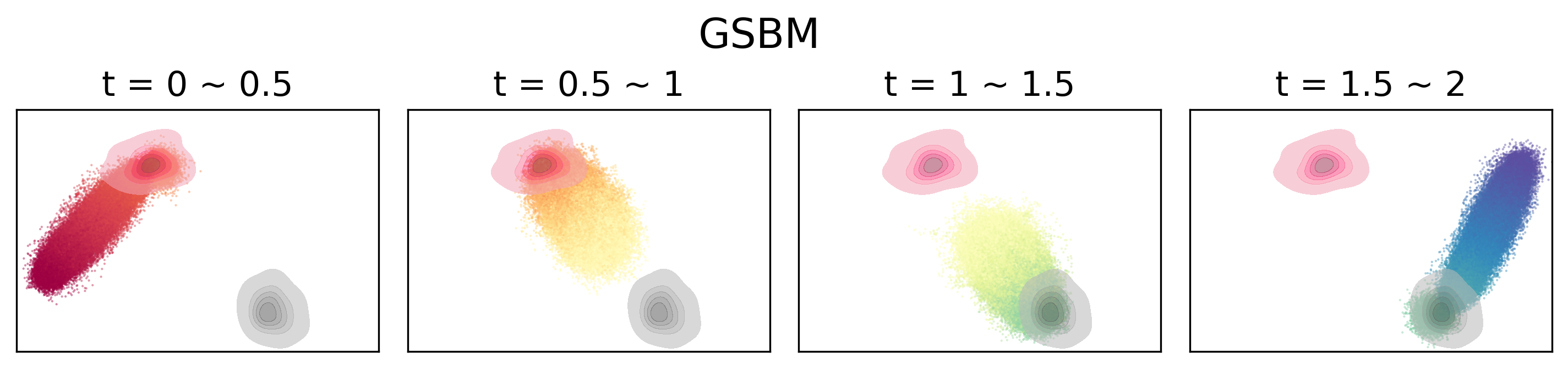}
    \includegraphics[width=\linewidth]{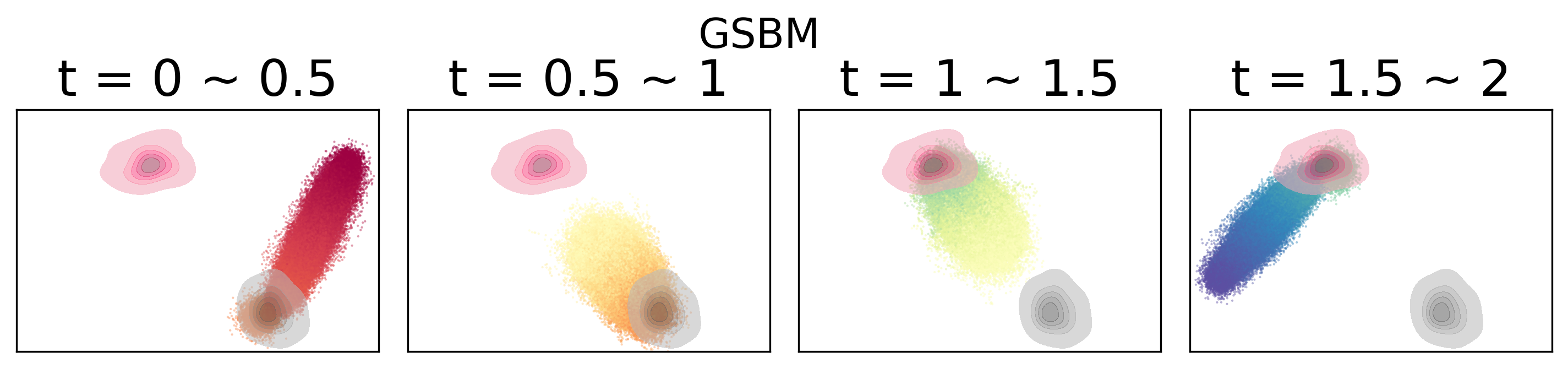}
  \end{minipage}
  \hfill
\begin{minipage}{0.49\textwidth}
    \centering
    \includegraphics[width=\linewidth]{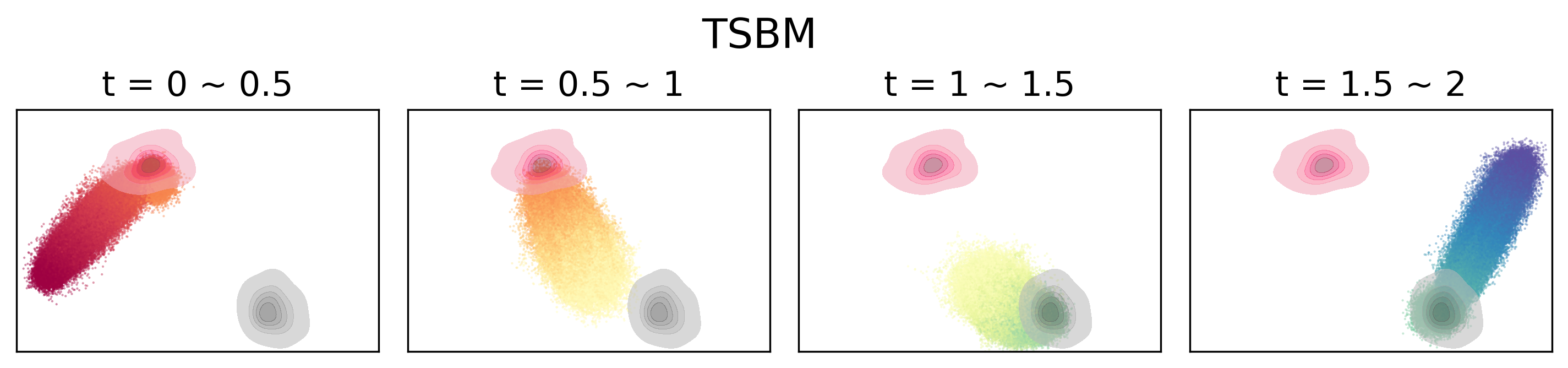}
    \includegraphics[width=\linewidth]{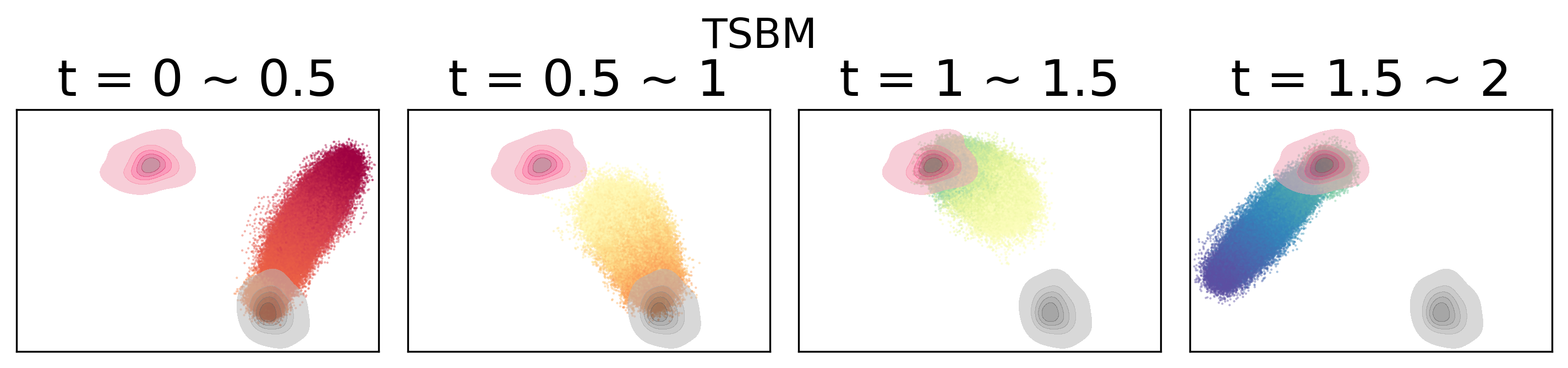}
  \end{minipage}  
  \caption{Learned TSB dynamics in the \textbf{multi-observation} discrete-time setting. The intermediate observation distributions at $t_1=0.4$ and $t_2=1.6$ are respectively shown in pink and grey. \textbf{Top}: forward time direction. \textbf{Bottom}: backward time direction. }
  \label{tsbm:fig:discrete_2d_bigaussian}
\end{figure}

\begin{figure}[h!]
  \centering
\begin{minipage}{0.49\textwidth}
    \centering
    \includegraphics[width=\linewidth]{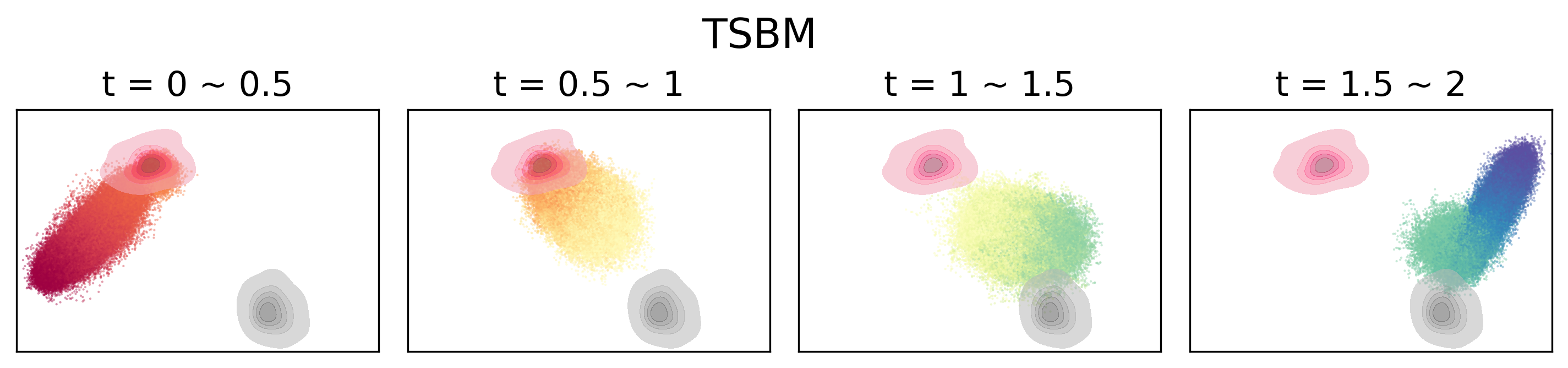}
    \includegraphics[width=\linewidth]{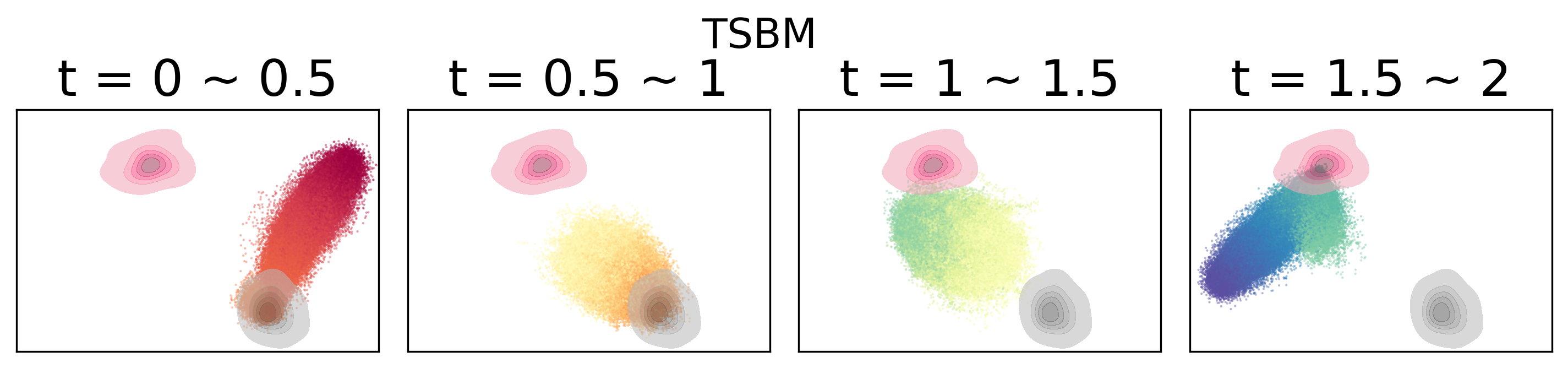}
  \end{minipage}
  \hfill
\begin{minipage}{0.49\textwidth}
    \centering
    \includegraphics[width=\linewidth]{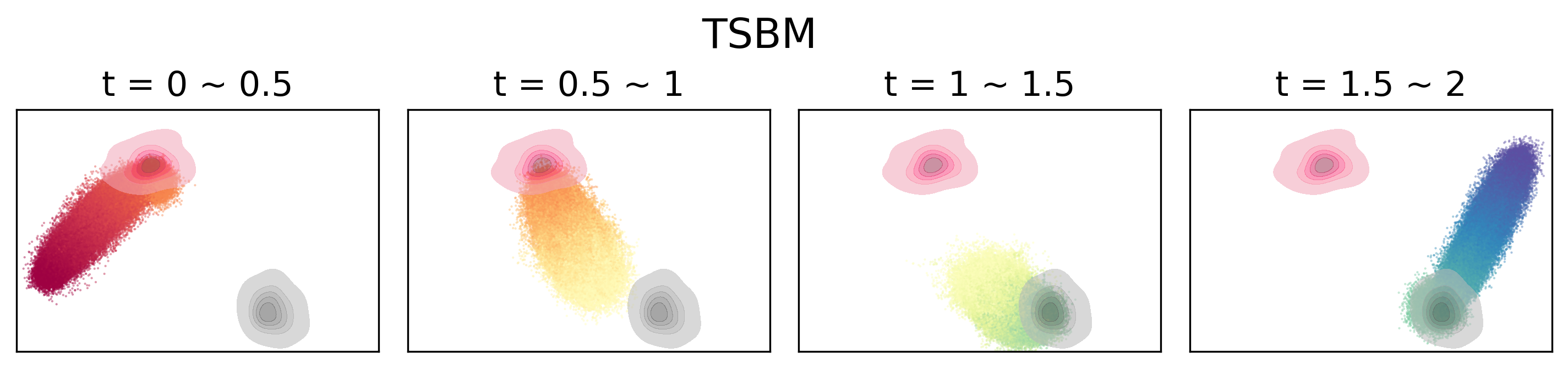}
    \includegraphics[width=\linewidth]{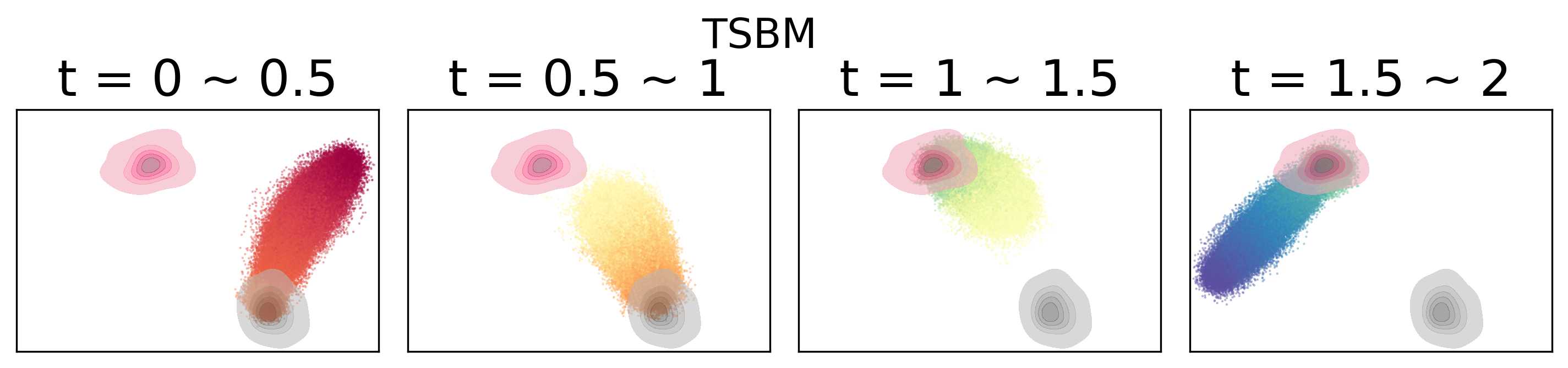}
  \end{minipage}  
  \caption{\textbf{Effect of control variate} on the learned TSBM dynamics in the \textbf{multi-observation} discrete-time setting (only 1 outer iteration is performed). The intermediate observation distributions at $t_1=0.4$ and $t_2=1.6$ are respectively shown in pink and grey. \textbf{Top}: forward time direction. \textbf{Bottom}: backward time direction. \textbf{(Left):} \underline{Without} control variate. \textbf{(Right):} \underline{With} control variate.}
  \label{tsbm:fig:discrete_2d_bigaussian_no_cv}
\end{figure}

\begin{table}[h!]
\vspace{-0.4cm}
    \caption{Results of 2D toy multi-observation ablation study. Results are averaged over $8$ metric evaluations. Bold indicates the best result in each setting.}
    \label{tsbm:table:discrete_2d_bigaussian_no_cv}
    \centering
  \resizebox{0.65\textwidth}{!}{%
  \begin{tabular}{l|cc|cc}
    \toprule
     Setting & \multicolumn{4}{c}{Multi-observation (only 1 outer iteration)}  \\
     \cmidrule(lr){1-1} \cmidrule(lr){2-5} 
    Metric & \multicolumn{2}{c}{Feasibility ($\downarrow$)} & \multicolumn{2}{c}{Optimality ($\downarrow$)}     \\
    \cmidrule(lr){1-1} \cmidrule(lr){2-3} \cmidrule(lr){4-5}
    Method & TSBM (no CV) & TSBM     & TSBM (no CV) & TSBM    \\
    \midrule
    Forward & $  0.07{\scriptstyle \pm 0.005}$  &  $ \bf 0.04{\scriptstyle \pm 0.004}$ & $ 4315{\scriptstyle\pm 19.2}$ & $\bf  795{\scriptstyle\pm 2.7}$ \\
    Backward & $ 0.05 {\scriptstyle\pm 0.005}$  & $\bf 0.02{\scriptstyle\pm 0.002}$ & $  1706 {\scriptstyle\pm 9.5}$ & $ \bf 720{\scriptstyle\pm 1.8}$   \\
    \bottomrule
  \end{tabular}
  }
\end{table}

\paragraph{On the underperformance of TSBM in the discrete-time setting.}
Although TSBM and GSBM yield visually comparable trajectory estimates in \Cref{tsbm:fig:discrete_2d_gaussian,tsbm:fig:discrete_2d_bigaussian}, we observe that TSBM systematically underperforms in this setting, despite the simplicity of the state cost. In contrast to the continuous-time case, TSBM does not reliably minimize the Twisted SB objective and exhibits large fluctuations during optimization, even when using the control variate strategy. Experiments conducted in the bidirectional setting led to the same conclusion.

We conjecture that this behavior stems from the ill-conditioning of the discrete-time Markovian projection loss, whose variance is high because samples from $\bfX_s\mid\bfX_t$ explore only a restricted region of the state space. This contrasts with the continuous-time formulation, where the state cost provides more distributed information along the trajectory. While introducing time dependence over large portions of the interval may seem to increase complexity, it appears in practice to regularize the problem by providing informative knowledge about future trajectory states. In the discrete-time formulation, this effect is absent: state-cost information is available only at the observation times. Improving the discrete-time formulation so as to make TSBM competitive with GSBM is left for future work.

\subsection{Crowd navigation experiments}\label{tsbm:app:crowd}

We consider the crowd navigation benchmarks of \cite{liu2022deep,liu2023generalized}, namely `Stunnel', `Gmm', and `Vneck'. These tasks model the transport of a population of particles from a source distribution to a target distribution while avoiding obstacles, encoded through a structured state cost. They therefore naturally fit within the SOC formulation in \eqref{tsbm:eq:extended_soc}.

For each benchmark, the state cost is decomposed into two components: an \emph{obstacle} cost $L_{\mathrm{obstacle}}$, which encodes the geometry of the environment, and an \emph{interaction} cost $L_{\mathrm{interaction}}$, which depends on the current population distribution. Specifically,
\begin{align*}
V_t(x_t)
=
\lambda_{\text{obs}} L_{\text{obstacle}}(x_t)
+
\lambda_{\text{int}} L_{\text{interaction}}(x_t;\hat{\Pbb}_t) \eqsp ,
\end{align*}
where $\lambda_{\text{obs}}>0$ and $\lambda_{\text{int}}>0$ control the relative strength of the two terms. The two components are defined as follows.
\begin{itemize}[wide, labelindent=0pt]
    \item \emph{Obstacle cost.}
    The obstacle term is task-specific and time-independent. In \cite{liu2022deep,liu2023generalized}, it is defined through sharp softplus-based penalties, which vanish outside the obstacle and increase approximately linearly inside it. In our experiments, we use a Gaussianized version of these penalties in order to provide more informative trajectory-level feedback through the $\nabla V$ term in the TSBM loss. This modification preserves the qualitative geometry of the obstacles while extending the influence of the cost beyond the immediate obstacle boundary, making the resulting gradients informative both inside and near the obstacles.

    \item \emph{Interaction cost.}
    The interaction term, when activated, is time-dependent and is not part of the static obstacle geometry. It is computed from the current estimated population distribution $\hat{\Pbb}_t$. Following \cite{liu2022deep,liu2023generalized}, we first consider the negative-entropy-like cost
    \[
    L_{\text{interaction}}(x_t;\hat{\Pbb}_t)=\log \hat{\Pbb}_t(x_t).
    \]
    Increasing $\lambda_{\text{int}}$ encourages higher marginal entropy along the path and therefore broadens the support of the dynamics. Since $\hat{\Pbb}_t$ is intractable, we follow \cite{liu2023generalized} and approximate it with a variational particle-based estimate,
    \[
    \hat{\Pbb}_t(x_t)
    \approx
    \sum_{(x_0,x_T)\in \text{batch}}
    \Qbb^\psi_t(x_t\mid x_0,x_T) \eqsp ,
    \]
    where the endpoint pairs $(x_0,x_T)$ are taken from the current data batch, either during variational optimization of the reciprocal projection or in the Markovian bridge-matching loss. 
    
    Here, $\Qbb^\psi(\cdot\mid x_0,x_T)$ denotes the spline bridge parametrization associated with the pair $(x_0,x_T)$, either being optimized in the reciprocal step or already learned in the Markovian step. This empirical approximation is accurate in practice due to the large batch sizes used, at least $512$ samples.

    We also consider the congestion cost
    \[
    L_{\text{interaction}}(x_t;\hat{\Pbb}_t)
    =
    \mathbb{E}_{y_t\sim\hat{\Pbb}_t}
    \left[
    \frac{2}{\|x_t-y_t\|^2+1}
    \right] \eqsp .
    \]
    In this case, increasing $\lambda_{\text{int}}$ encourages stronger dispersion of the generated particles. In practice, as in \cite{liu2023generalized}, we approximate this expectation by Monte Carlo using a single particle $y_t\sim\hat{\Pbb}_t$ for each $x_t$.
\end{itemize}
 
\paragraph{High-dimensional extension.}
The crowd navigation costs are originally defined in 2D through softplus penalties based on the distance between the evaluated point and the obstacle centers, with ellipses for Stunnel, circles for Gmm, and triangles for Vneck. To assess how the methods scale with the ambient dimension, we extend each geometry to dimension $d\in \{10,50\}$ as follows.

For an obstacle center with 2D coordinates $(c_1,c_2)$, we define its lifted counterpart as
\[
\tilde{c}=(c_1,c_1,\ldots,c_1,c_2)\in\mathbb{R}^d,
\]
where the first coordinate is repeated over the first $d-1$ longitudinal dimensions, while the last coordinate corresponds to the transverse dimension. This construction preserves the qualitative 2D geometry while allowing the ambient dimension to vary.

Distances to obstacle centers are then computed using the RMS-normalized squared distance
\begin{align*}
    \Delta_d(x,\tilde{c};s_{\mathrm{long}},s_{\mathrm{trans}})
    =
    \frac{2}{d}
    \left[
    \sum_{i=1}^{d-1}
    \left(\frac{x_i-\tilde{c}_i}{s_{\mathrm{long}}}\right)^2
    +
    \left(\frac{x_d-\tilde{c}_d}{s_{\mathrm{trans}}}\right)^2
    \right],
\end{align*}
where $s_{\mathrm{long}}>0$ and $s_{\mathrm{trans}}>0$ control the longitudinal and transverse scalings. These parameters are chosen so that the effective obstacle width remains comparable as $d$ increases, while recovering the standard 2D geometry when $d=2$.

\paragraph{Experimental setting.}
Following \cite{liu2023generalized}, we use a single interaction cost for each crowd navigation benchmark: entropy for Vneck, and congestion for Stunnel and Gmm. In 2D, we keep the same values of $\lambda_{\mathrm{obs}}$ and $\lambda_{\mathrm{int}}$ as in \cite{liu2023generalized}, where these weights have a natural interpretation in terms of physical modeling. In higher dimensions, we focus only on the obstacle cost, since the interaction cost becomes prohibitively memory-intensive. All weight values are reported in \Cref{tsbm:tab:crowd_weights}.

\begin{table}[h!]
\centering
\caption{State-cost weights used in the crowd navigation experiments. Here $\lambda_{\mathrm{ent}}$ indicates the weight of entropy interaction cost while $\lambda_{\mathrm{cgst}}$ indicates the weight of the congestion interaction cost.}
\label{tsbm:tab:crowd_weights}
\begin{tabular}{l|c|ccc}
\toprule
Setting & Dimension & $\lambda_{\mathrm{obs}}$ & $\lambda_{\mathrm{ent}}$ & $\lambda_{\mathrm{cgst}}$ \\
\midrule
\text{Stunnel}
& $d=2$  & $1500$ & $0$ & $50$ \\
& $d=10$ & $1000$ & $0$ & $0$ \\
& $d=50$ & $800$  & $0$ & $0$ \\
\midrule
\text{Gmm}
& $d=2$  & $1500$ & $0$ & $5$ \\
& $d=10$ & $30$   & $0$ & $0$ \\
& $d=50$ & $35$   & $0$ & $0$ \\
\midrule
\text{Vneck}
& $d=2$  & $3000$  & $8$ & $0$ \\
& $d=10$ & $10000$ & $0$ & $0$ \\
& $d=50$ & $10000$ & $0$ & $0$ \\
\bottomrule
\end{tabular}
\end{table}

\paragraph{Additional results (main experiment).} In complement of \Cref{tsbm:tab:crowd_main} giving only results of forward dynamics, we provide in \Cref{tsbm:tab:crowd_backward} the metrics computed with \emph{backward} generated dynamics. In the same manner, we give in \Cref{tsbm:fig:crowd_state_main_b} the backward counterpart of the 2D forward dynamics of \Cref{tsbm:fig:crowd_state_main}. We draw the same conclusion exposed in the main : TSBM better integrates the state cost than GSBM when optimzing the SB cost, but slightly underperforms at source marginal constraint.

\begin{table}[h!]
    \caption{\textbf{Backward} results of crowd navigation experiments for Stunnel and Vneck, with increasing $d$. Results are averaged over $8$ metric evaluations. Bold indicates the best result in each setting.}
    \label{tsbm:tab:crowd_backward}
  \resizebox{\textwidth}{!}{%
  \begin{tabular}{l|cc|cc|cc|cc}
    \toprule
     State cost & \multicolumn{4}{c}{Stunnel} & \multicolumn{4}{c}{Vneck}  \\
     \cmidrule(lr){1-1} \cmidrule(lr){2-5} \cmidrule(lr){6-9}
    Metric & \multicolumn{2}{c}{Feasibility ($\downarrow$)} & \multicolumn{2}{c}{Optimality ($\downarrow$)} & \multicolumn{2}{c}{Feasibility ($\downarrow$)} & \multicolumn{2}{c}{Optimality ($\downarrow$)}    \\
    \cmidrule(lr){1-1} \cmidrule(lr){2-3} \cmidrule(lr){4-5} \cmidrule(lr){6-7} \cmidrule(lr){8-9}
    Method & GSBM & TSBM     & GSBM & TSBM      & GSBM & TSBM & GSBM & TSBM \\
    \midrule
    $d=2$ & $ 0.040{\scriptstyle \pm 0.005}$  &  $\bf 0.030{\scriptstyle \pm 0.005}$ & $977{\scriptstyle\pm 0.8}$ & $\bf953{\scriptstyle\pm 0.9}$  & $\bf 0.02{\scriptstyle \pm 0.005}$ &   $0.13{\scriptstyle \pm 0.03}$ & $70.1{\scriptstyle\pm 0.4}$ & $\bf56.9{\scriptstyle\pm 0.4}$ \\
    $d=10$ & $\bf1.12 {\scriptstyle\pm 0.01}$  & $\bf 1.13{\scriptstyle\pm 0.01}$ & $759 {\scriptstyle\pm 0.8}$ & $\bf746{\scriptstyle\pm 0.7}$  & $\bf 1.05 {\scriptstyle\pm 0.01}$ &    $ 1.37{\scriptstyle\pm 0.02}$ & $\bf 257 {\scriptstyle\pm 1.3}$ & $259 {\scriptstyle\pm 0.5}$  \\
    $d=50$  & $\bf 17.33 {\scriptstyle\pm0.07}$  & $18.06 {\scriptstyle\pm0.09}$  & $3589 {\scriptstyle\pm2.7}$ & $\bf3577 {\scriptstyle\pm2.8}$  & $\bf18.1 {\scriptstyle\pm0.1}$ &  $20.3 {\scriptstyle\pm2.6}$ & $1386 {\scriptstyle\pm2.3}$ &  $\bf1380 {\scriptstyle\pm2.4}$ \\
    \bottomrule
  \end{tabular}
  }
\end{table}

\begin{figure}[h!]
  \centering
\begin{minipage}{0.49\textwidth}
    \centering
    \includegraphics[width=\linewidth]{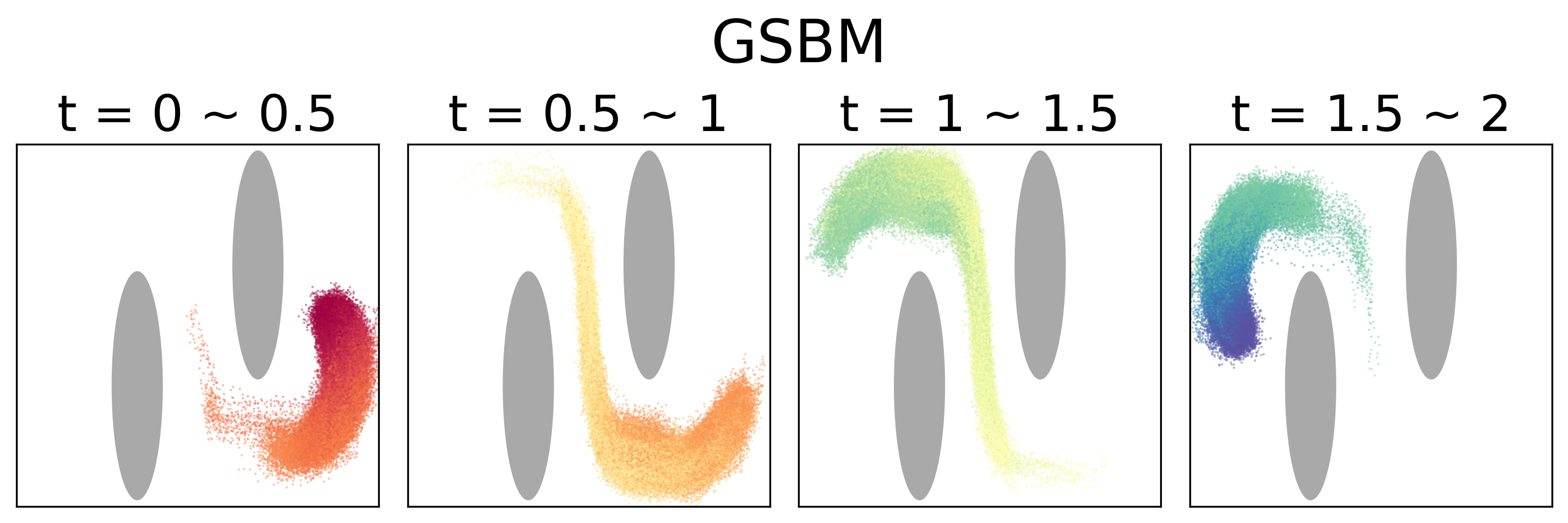}
    \vfill
    \includegraphics[width=\linewidth]{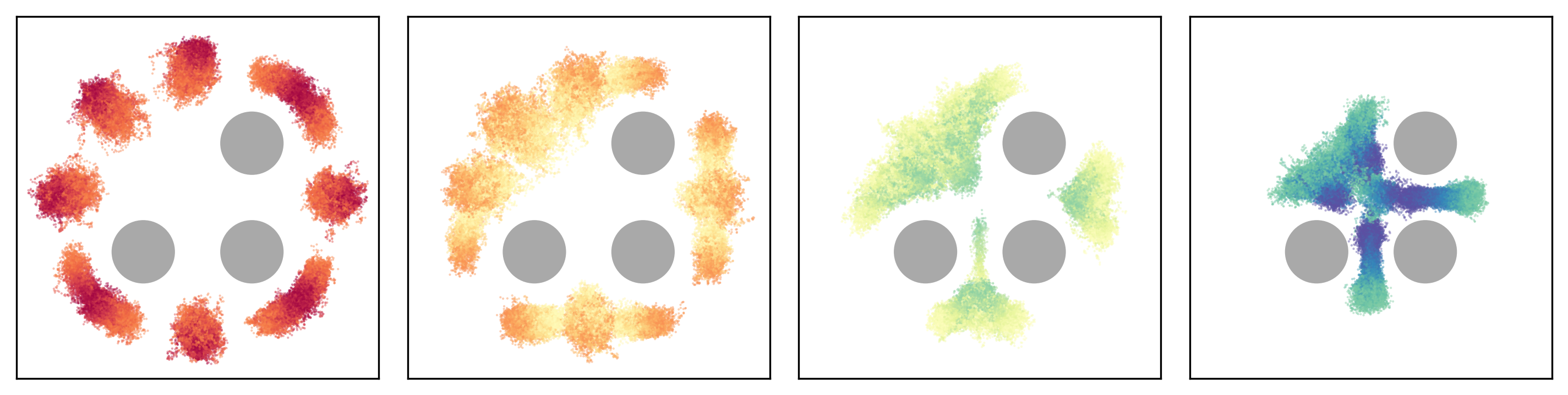}
    \vfill
    \includegraphics[width=\linewidth]{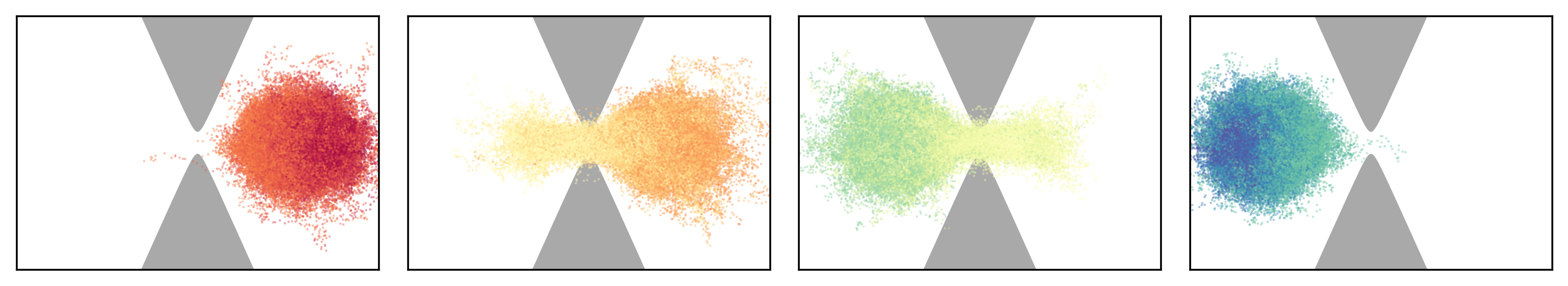}
  \end{minipage}
  \hfill
\begin{minipage}{0.49\textwidth}
    \centering
    \includegraphics[width=\linewidth]{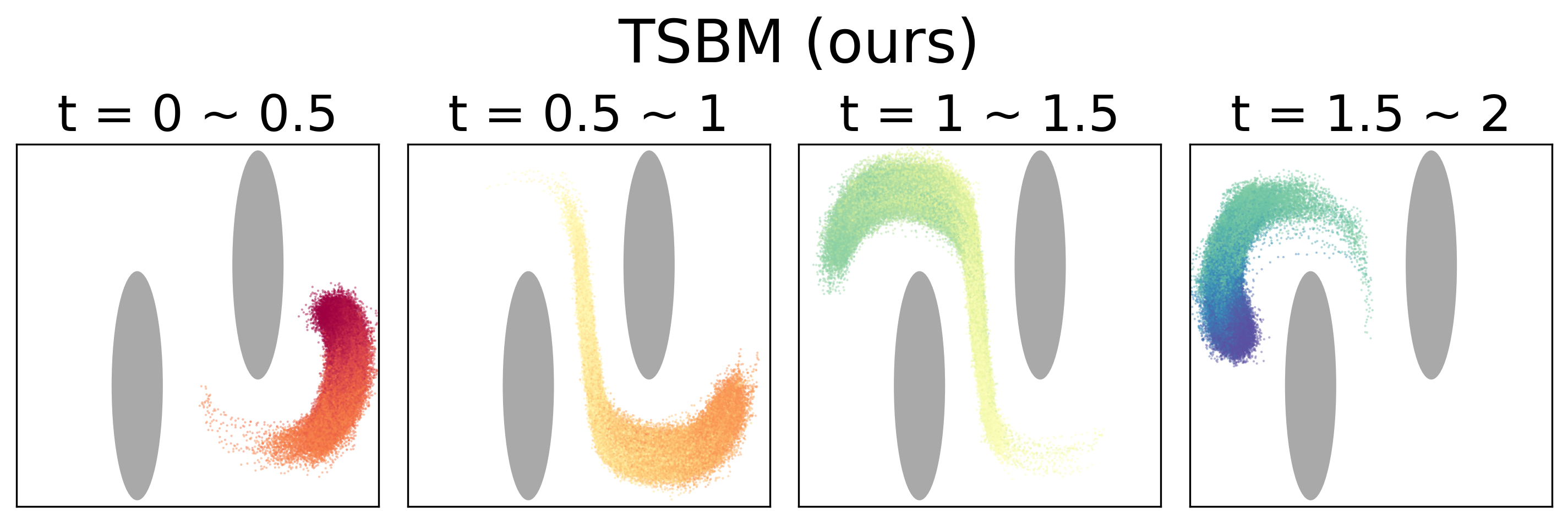}
    \vfill
    \includegraphics[width=\linewidth]{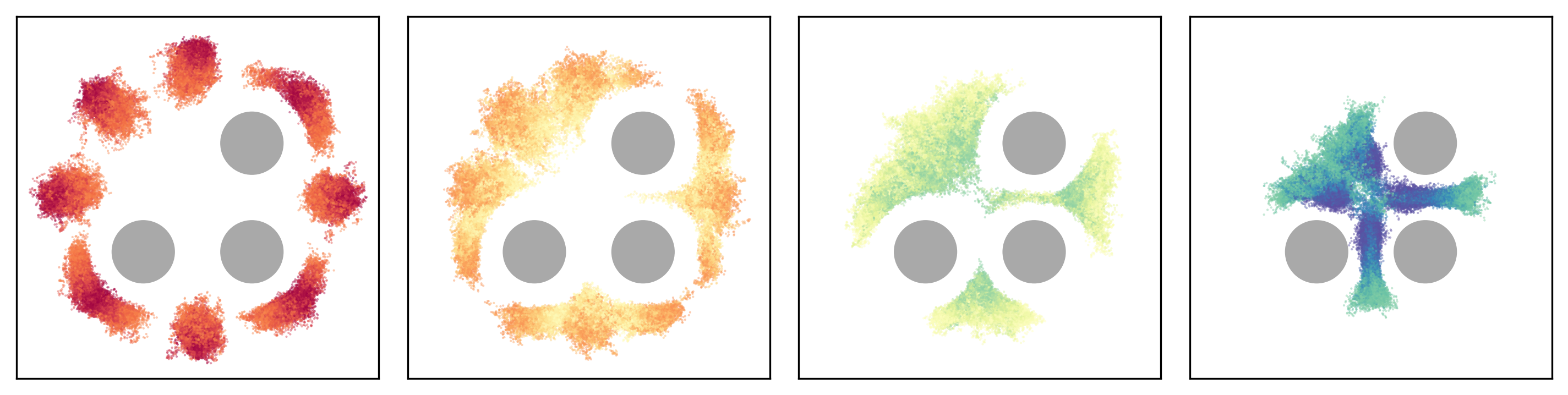}
    \vfill
    \includegraphics[width=\linewidth]{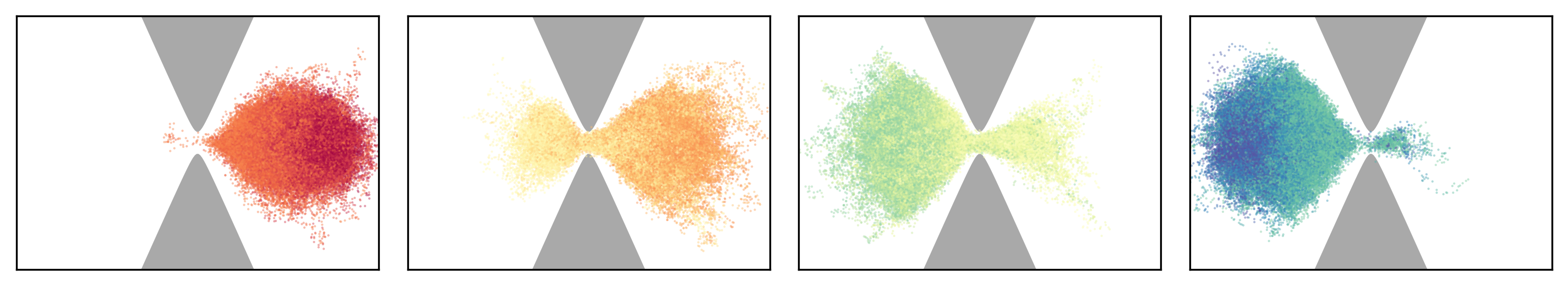}
  \end{minipage}  
  \vspace{-0.1cm}
  \caption{\textbf{Backward} learned TSB dynamics for Stunnel (top), Gmm (middle), Vneck (bottom) with $d=2$. This is complementary to \Cref{tsbm:fig:crowd_state_main}.}
  \label{tsbm:fig:crowd_state_main_b}
\end{figure}

\paragraph{Additional results (ablation studies).}
For each 2D ablation study, we report additional metrics and visualizations.
\begin{enumerate}[label=(\Alph*), labelindent=0pt, wide]
    \item {\textbf{Unidirectional vs. bidirectional.}}
    Complementing \Cref{tsbm:fig:bidir_vs_undir_stunnel}, we report its backward counterpart in \Cref{tsbm:app:fig:bidir_vs_undir_stunnel}. We also present the same ablation for 2D-Gmm and 2D-Vneck in \Cref{tsbm:app:fig:bidir_vs_undir_gmm} and \Cref{tsbm:app:fig:bidir_vs_undir_vneck}, respectively. To provide a complete comparison between the unidirectional and bidirectional variants across all crowd navigation settings, we further report the final converged metric values in \Cref{tsbm:app:crowd_bidir_f} for the forward dynamics and in \Cref{tsbm:app:crowd_bidir_b} for the backward dynamics. Overall, these experiments lead to the same conclusion: the bidirectional setting substantially stabilizes the training of TSBM and yields smoother convergence, although it does not significantly improve the final performance. By contrast, GSBM benefits less from the bidirectional variant and can even be degraded in some cases, such as in the Gmm setting.

    \item {\textbf{Lower-entropy setting.}}
    We report the backward counterpart of \Cref{tsbm:tab:crowd_main_ablation} in the left column of \Cref{tsbm:app:crowd_ablation_b}. The results again show that TSBM outperforms GSBM in minimizing the SB cost in the lower-entropy regime. The corresponding forward dynamics are shown in \Cref{tsbm:fig:crowd_sigma_1}.

    \item {\textbf{Effect of the control variate.}}
    Similarly, we provide the backward counterpart of \Cref{tsbm:tab:crowd_main_ablation} in the right column of \Cref{tsbm:app:crowd_ablation_b}, confirming the benefit of the control variate. The corresponding forward dynamics are shown in \Cref{tsbm:fig:crowd_control_variate}.
\end{enumerate}

\begin{figure}[h!]
    \centering
    \includegraphics[width=\linewidth]{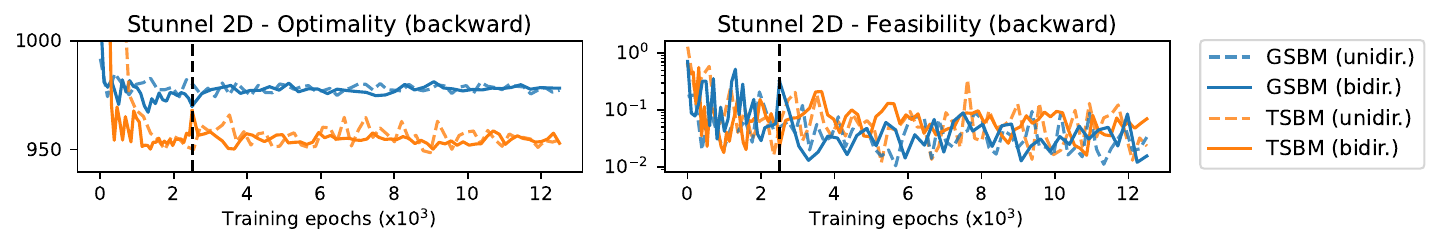}
    \vspace{-0.5cm}
    \caption{Evolution of \textbf{backward metrics} for TSBM and GSBM on \textbf{2D-Stunnel} with $\sigma=2$. The vertical dotted line marks the end of bidirectional pretraining or, equivalently, the first unidirectional outer iteration. This is complementary to \Cref{tsbm:fig:bidir_vs_undir_stunnel}.}
    \label{tsbm:app:fig:bidir_vs_undir_stunnel}
\end{figure}

\begin{figure}[h!]
    \centering
    \includegraphics[width=\linewidth]{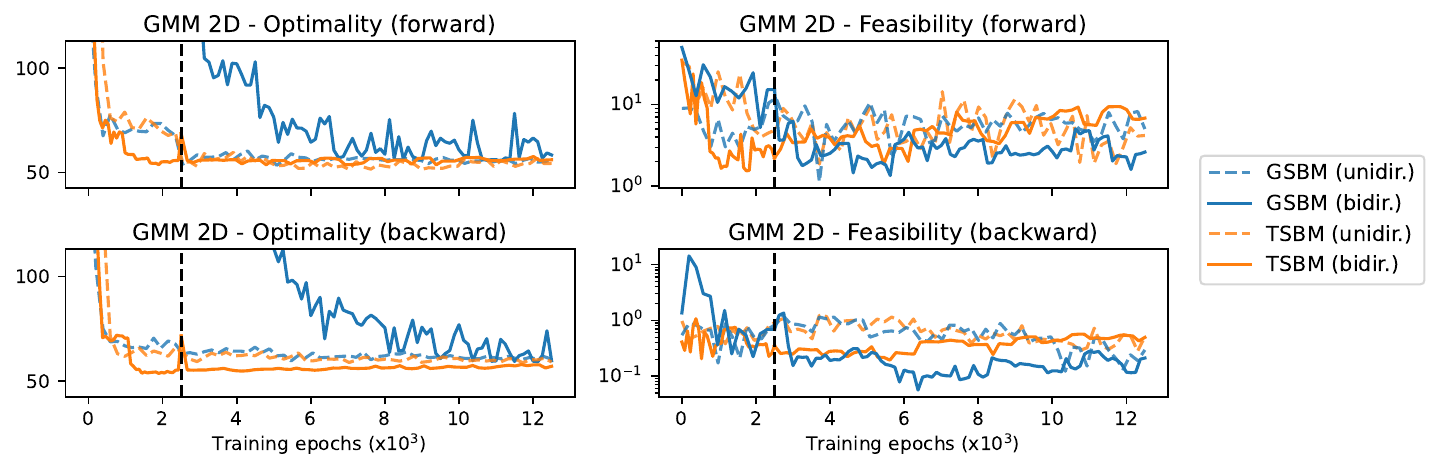}
    \vspace{-0.5cm}
    \caption{Evolution of \textbf{forward and backward metrics} for TSBM and GSBM on \textbf{2D-Gmm} with $\sigma=2$. The vertical dotted line marks the end of bidirectional pretraining or, equivalently, the first unidirectional outer iteration.}
    \label{tsbm:app:fig:bidir_vs_undir_gmm}
\end{figure}

\begin{figure}[h!]
\vspace{0.2cm}
    \centering
    \includegraphics[width=\linewidth]{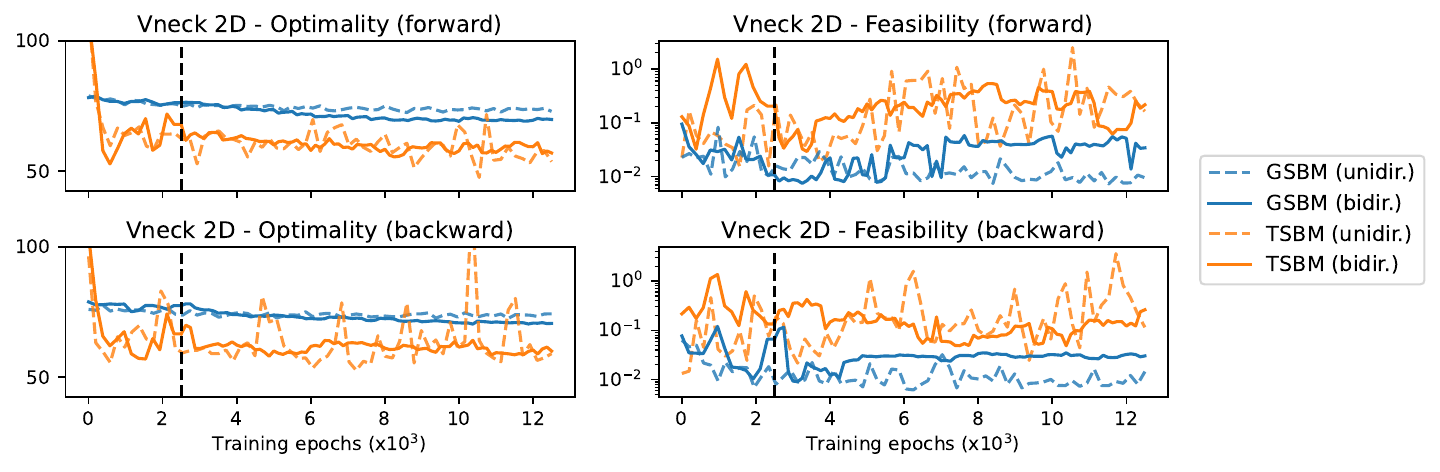}
    \vspace{-0.5cm}
    \caption{Evolution of \textbf{forward and backward} metrics for TSBM and GSBM on \textbf{2D-Vneck} with $\sigma=2$. The vertical dotted line marks the end of bidirectional pretraining or, equivalently, the first unidirectional outer iteration.}
    \label{tsbm:app:fig:bidir_vs_undir_vneck}
\end{figure}

\begin{table}[h!]
    \vspace{0.2cm}
    \caption{\textbf{Unidirectional vs. bidirectional (forward)}: results of crowd navigation experiments for 2D-Stunnel, 2D-GMM and 2D-Vneck. Results are averaged over $8$ metric evaluations. Bold indicates the best result in each setting.}
    \label{tsbm:app:crowd_bidir_f}
  \resizebox{\textwidth}{!}{%
  \begin{tabular}{l|cc|cc|cc|cc}
    \toprule
     Setting & \multicolumn{4}{c}{Unidirectional ($\sigma=2$, full training)} & \multicolumn{4}{c}{Bidirectional ($\sigma=2$, full training)}  \\
     \cmidrule(lr){1-1} \cmidrule(lr){2-5} \cmidrule(lr){6-9}
    Metric & \multicolumn{2}{c}{Feasibility ($\downarrow$)} & \multicolumn{2}{c}{Optimality ($\downarrow$)} & \multicolumn{2}{c}{Feasibility ($\downarrow$)} & \multicolumn{2}{c}{Optimality ($\downarrow$)}    \\
    \cmidrule(lr){1-1} \cmidrule(lr){2-3} \cmidrule(lr){4-5} \cmidrule(lr){6-7} \cmidrule(lr){8-9}
    Method & GSBM & TSBM     & GSBM & TSBM      & GSBM & TSBM & GSBM & TSBM \\
    \midrule
    2D-Stunnel & $ \bf 0.006{\scriptstyle \pm 0.001}$  &  $ 0.056{\scriptstyle \pm 0.007}$ & $978{\scriptstyle\pm 1.0}$ & $\bf 956{\scriptstyle\pm 1.0}$  & $ \bf 0.02{\scriptstyle \pm 0.002}$ &   $ \bf 0.02{\scriptstyle \pm 0.003}$ & $ 977{\scriptstyle\pm 1.2}$ & $\bf  952{\scriptstyle\pm 0.9}$ \\
    2D-Gmm & $ 3.66 {\scriptstyle\pm 0.45}$  & $\bf 2.83{\scriptstyle\pm 0.30}$ & $ 55.5 {\scriptstyle\pm 0.4}$ & $\bf 54.4{\scriptstyle\pm 0.3}$  & $ \bf 2.77 {\scriptstyle\pm 0.50}$ &    $  7.10{\scriptstyle\pm 0.50}$ & $ 58.2 {\scriptstyle\pm 1.9}$ & $ \bf 55.0 {\scriptstyle\pm 0.3}$  \\
    2D-Vneck  & $\bf 0.008 {\scriptstyle\pm0.004}$  & $0.16 {\scriptstyle\pm0.02}$  & $69.2 {\scriptstyle\pm0.5}$ & $\bf 53.1 {\scriptstyle\pm0.4}$  & $\bf 0.04 {\scriptstyle\pm0.02}$ &  $ 0.20 {\scriptstyle\pm0.02}$ & $ 67.5 {\scriptstyle\pm0.6}$ &  $\bf 55.2 {\scriptstyle\pm0.3}$ \\
    \bottomrule
  \end{tabular}
  }
\end{table}

\begin{table}[h!]
   \vspace{0.2cm}
    \caption{\textbf{Unidirectional vs. bidirectional (backward)}: results of crowd navigation experiments for 2D-Stunnel, 2D-GMM and 2D-Vneck. Results are averaged over $8$ metric evaluations. Bold indicates the best result in each setting.}
    \label{tsbm:app:crowd_bidir_b}
  \resizebox{\textwidth}{!}{%
  \begin{tabular}{l|cc|cc|cc|cc}
    \toprule
     Setting & \multicolumn{4}{c}{Unidirectional ($\sigma=2$, full training)} & \multicolumn{4}{c}{Bidirectional ($\sigma=2$, full training)}  \\
     \cmidrule(lr){1-1} \cmidrule(lr){2-5} \cmidrule(lr){6-9}
    Metric & \multicolumn{2}{c}{Feasibility ($\downarrow$)} & \multicolumn{2}{c}{Optimality ($\downarrow$)} & \multicolumn{2}{c}{Feasibility ($\downarrow$)} & \multicolumn{2}{c}{Optimality ($\downarrow$)}    \\
    \cmidrule(lr){1-1} \cmidrule(lr){2-3} \cmidrule(lr){4-5} \cmidrule(lr){6-7} \cmidrule(lr){8-9}
    Method & GSBM & TSBM     & GSBM & TSBM      & GSBM & TSBM & GSBM & TSBM \\
    \midrule
    2D-Stunnel & $  0.040{\scriptstyle \pm 0.005}$  &  $ \bf 0.030{\scriptstyle \pm 0.005}$ & $977{\scriptstyle\pm 0.8}$ & $\bf 953{\scriptstyle\pm 0.9}$  & $ \bf 0.01{\scriptstyle \pm 0.002}$ &   $  0.07{\scriptstyle \pm 0.007}$ & $ 978{\scriptstyle\pm 0.8}$ & $\bf  952{\scriptstyle\pm 1.0}$ \\
    2D-Gmm & $ \bf 0.31 {\scriptstyle\pm 0.03}$  & $\bf 0.30{\scriptstyle\pm 0.03}$ & $ 58.9 {\scriptstyle\pm 0.3}$ & $\bf 58.2{\scriptstyle\pm 0.3}$  & $ \bf 0.31 {\scriptstyle\pm 0.05}$ &    $  0.61{\scriptstyle\pm 0.05}$ & $ 53.1 {\scriptstyle\pm 3.8}$ & $ \bf 55.1 {\scriptstyle\pm 0.2}$  \\
    2D-Vneck  & $\bf 0.02 {\scriptstyle\pm0.005}$  & $0.13 {\scriptstyle\pm0.03}$  & $70.1 {\scriptstyle\pm0.4}$ & $\bf 56.9 {\scriptstyle\pm0.4}$  & $\bf 0.03{\scriptstyle\pm0.02}$ &  $ 0.32 {\scriptstyle\pm0.04}$ & $ 67.0 {\scriptstyle\pm0.3}$ &  $\bf 57.2 {\scriptstyle\pm0.2}$ \\
    \bottomrule
  \end{tabular}
  }
\end{table}

\newpage

\begin{table}[h!]
    \caption{\textbf{Backward} results of crowd navigation \textbf{ablation studies} for Stunnel, Gmm and Vneck with $d=2$. Results are averaged over $8$ metric evaluations. Bold indicates the best result in each setting.}
    \label{tsbm:app:crowd_ablation_b}
  \resizebox{\textwidth}{!}{%
  \begin{tabular}{l|cc|cc|cc|cc}
    \toprule
     Setting & \multicolumn{4}{c}{Bidirectional ($\sigma=1$, only pretraining)} & \multicolumn{4}{c}{Unidirectional ($\sigma=2$, only 1 outer iteration)}  \\
     \cmidrule(lr){1-1} \cmidrule(lr){2-5} \cmidrule(lr){6-9}
    Metric & \multicolumn{2}{c}{Feasibility ($\downarrow$)} & \multicolumn{2}{c}{Optimality ($\downarrow$)} & \multicolumn{2}{c}{Feasibility ($\downarrow$)} & \multicolumn{2}{c}{Optimality ($\downarrow$)}    \\
    \cmidrule(lr){1-1} \cmidrule(lr){2-3} \cmidrule(lr){4-5} \cmidrule(lr){6-7} \cmidrule(lr){8-9}
    Method & GSBM & TSBM     & GSBM & TSBM      & TSBM (no CV) & TSBM & TSBM (no CV) & TSBM \\
    \midrule
    2D-Stunnel & $ 0.51{\scriptstyle \pm 0.02}$  &  $\bf 0.02{\scriptstyle \pm 0.002}$ & $952{\scriptstyle\pm 0.6}$ & $\bf 947{\scriptstyle\pm 0.5}$  & $ 0.05{\scriptstyle \pm 0.005}$ &   $ \bf 0.02{\scriptstyle \pm 0.001}$ & $\bf 948{\scriptstyle\pm 0.7}$ & $ 951{\scriptstyle\pm 0.9}$ \\
    2D-Gmm & $\bf 0.13 {\scriptstyle\pm 0.02}$  & $0.28{\scriptstyle\pm 0.02}$ & $ 67.9 {\scriptstyle\pm 12.1}$ & $\bf 54.2{\scriptstyle\pm 0.3}$  & $ \bf 0.40 {\scriptstyle\pm 0.04}$ &    $  0.46{\scriptstyle\pm 0.07}$ & $\bf 59.0 {\scriptstyle\pm 0.3}$ & $ 62.4 {\scriptstyle\pm 0.2}$  \\
    2D-Vneck  & $\bf 0.08 {\scriptstyle\pm0.006}$  & $0.46 {\scriptstyle\pm0.06}$  & $66.0 {\scriptstyle\pm0.3}$ & $\bf 56.6 {\scriptstyle\pm0.4}$  & $\bf 0.06 {\scriptstyle\pm0.002}$ &  $\bf 0.06 {\scriptstyle\pm0.006}$ & $ 58.4 {\scriptstyle\pm0.2}$ &  $\bf 57.9 {\scriptstyle\pm0.4}$ \\
    \bottomrule
  \end{tabular}
  }
\end{table}

\begin{figure}[h!]
\vspace{0.3cm}
  \centering
\begin{minipage}{0.49\textwidth}
    \centering
    \includegraphics[width=\linewidth]{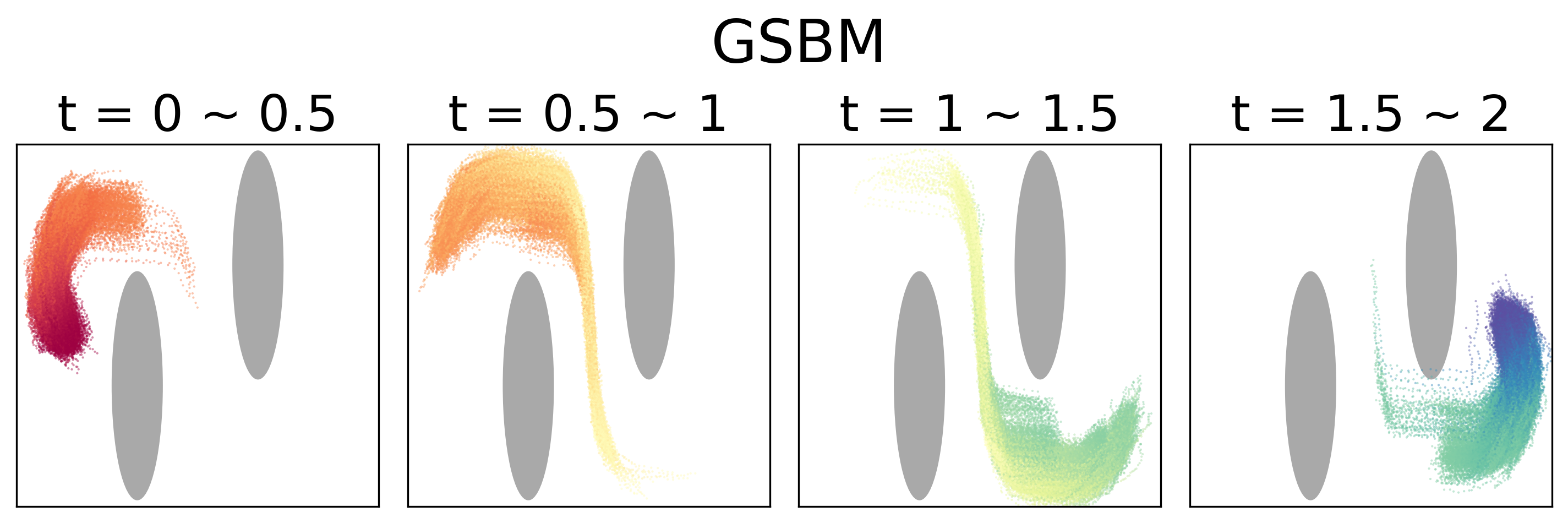}
    \vfill
    \includegraphics[width=\linewidth]{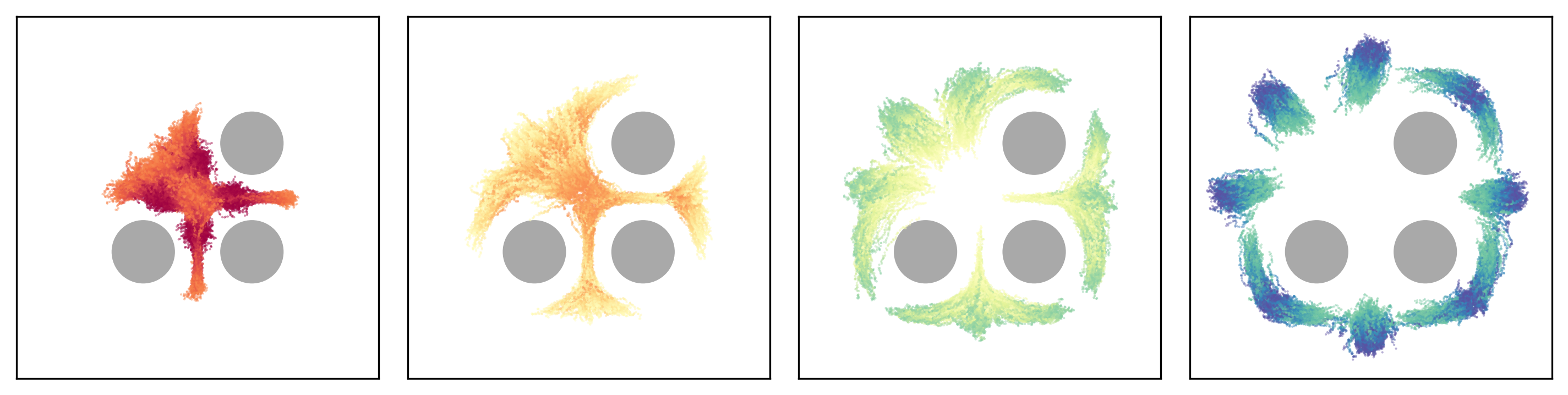}
    \vfill
    \includegraphics[width=\linewidth]{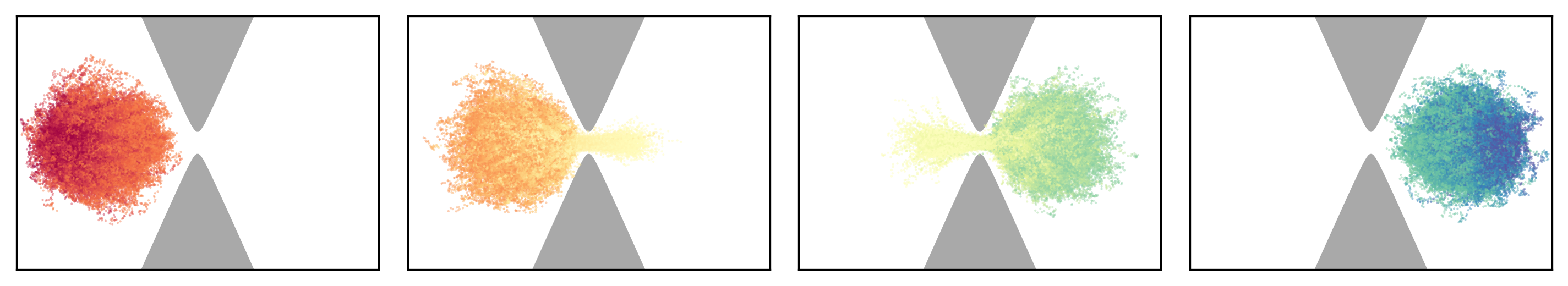}
  \end{minipage}
  \hfill
\begin{minipage}{0.49\textwidth}
    \centering
    \includegraphics[width=\linewidth]{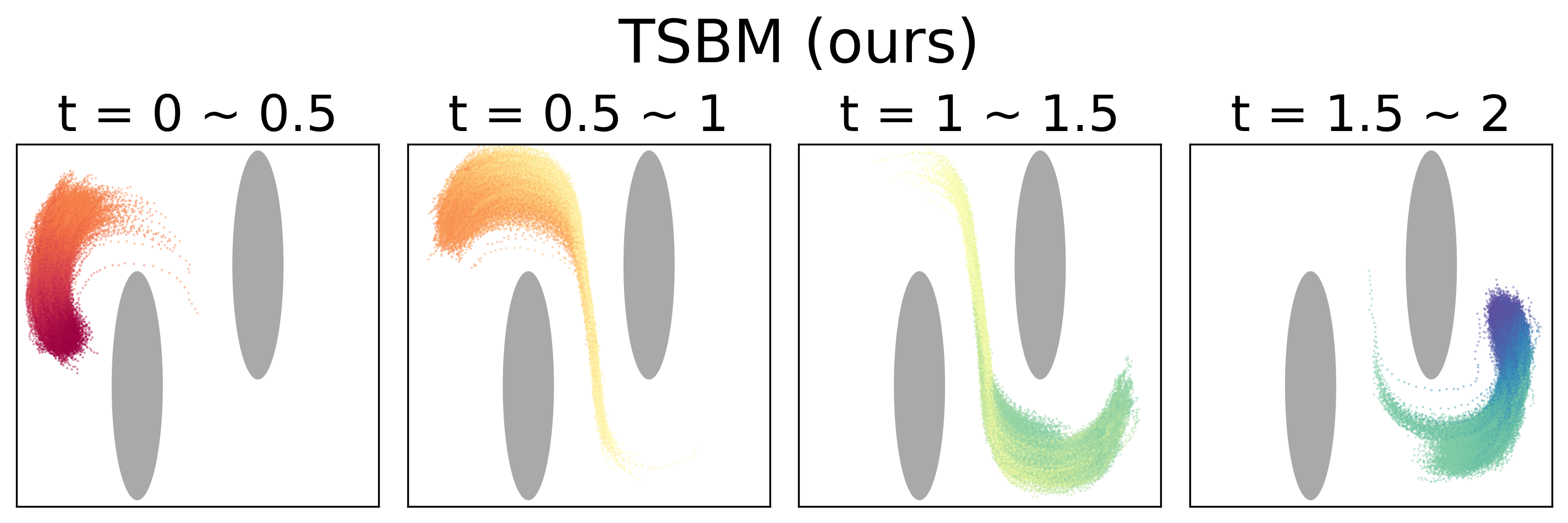}
    \vfill
    \includegraphics[width=\linewidth]{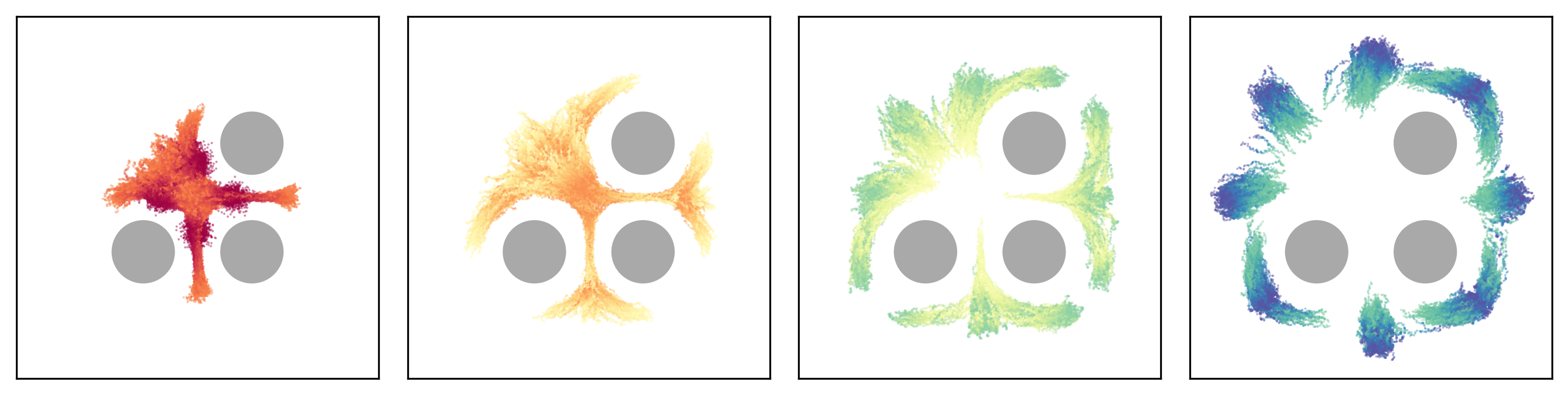}
    \vfill
    \includegraphics[width=\linewidth]{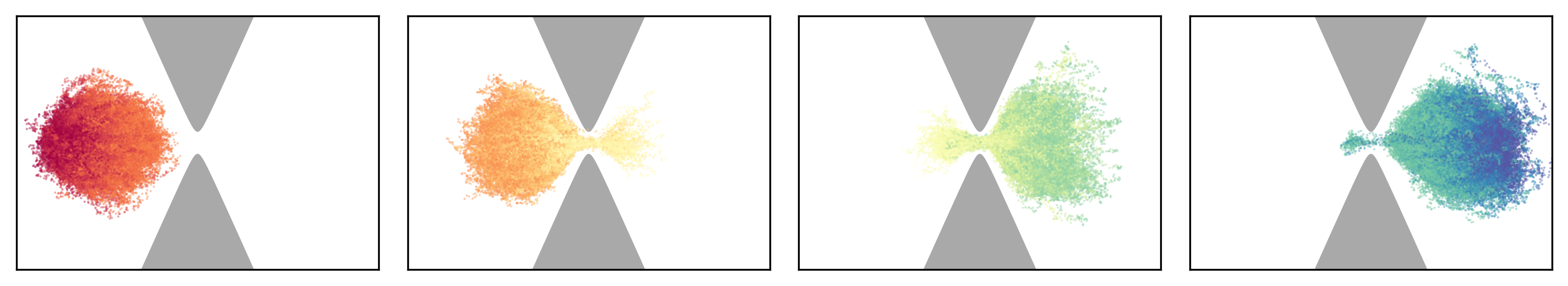}
  \end{minipage}  
\caption{Learned TSB dynamics for Stunnel (top), Gmm (middle), Vneck (bottom) with $d=2$ and $\sigma=1$. These results are obtained after the stage of pretraining within the \underline{bidirectional} variant.}
\label{tsbm:fig:crowd_sigma_1}
\end{figure}

\begin{figure}[h!]
\vspace{0.3cm}
  \centering
\begin{minipage}{0.49\textwidth}
    \centering
    \includegraphics[width=\linewidth]{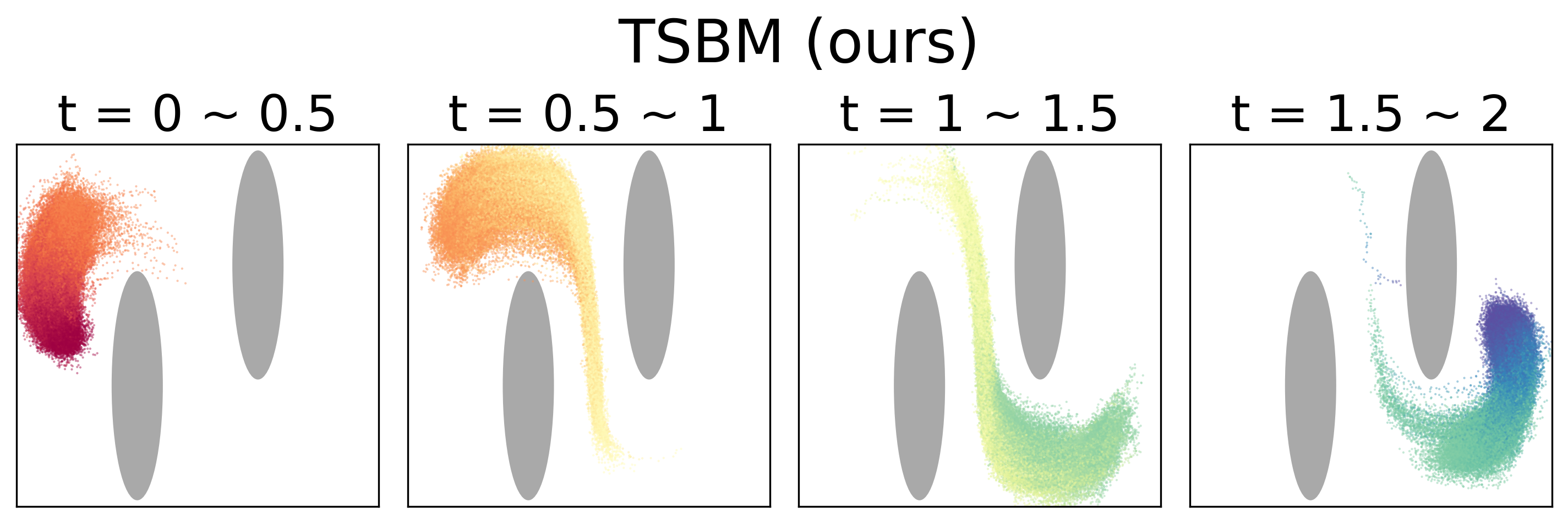}
    \vfill
    \includegraphics[width=\linewidth]{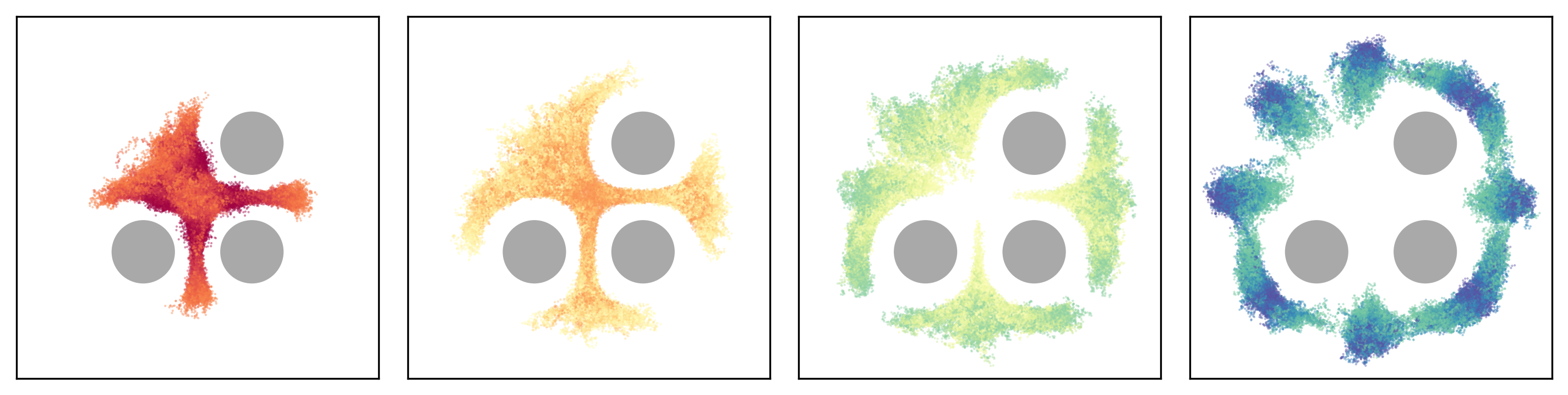}
    \vfill
    \includegraphics[width=\linewidth]{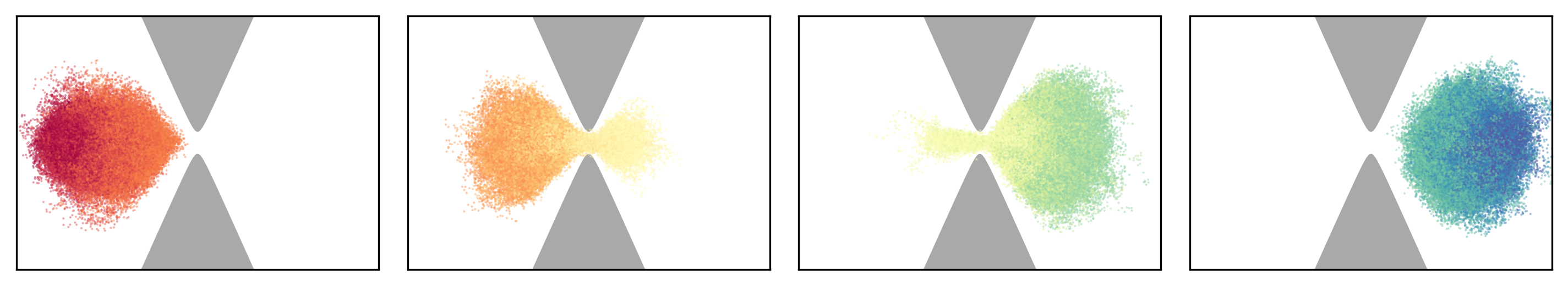}
  \end{minipage}
  \hfill
\begin{minipage}{0.49\textwidth}
    \centering
    \includegraphics[width=\linewidth]{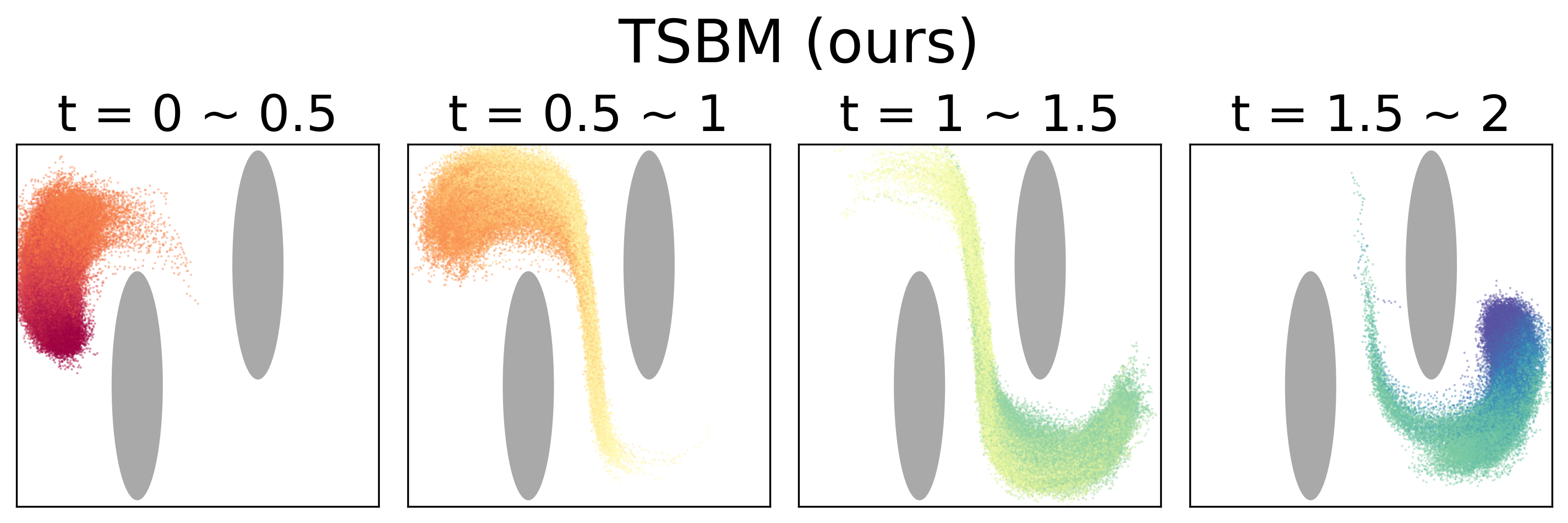}
    \vfill
    \includegraphics[width=\linewidth]{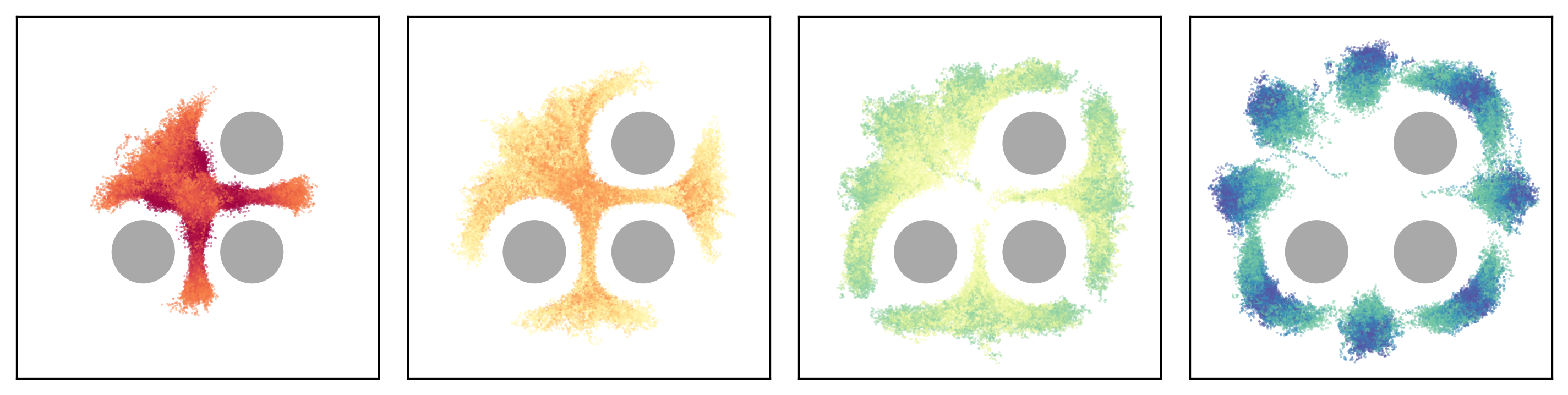}
    \vfill
    \includegraphics[width=\linewidth]{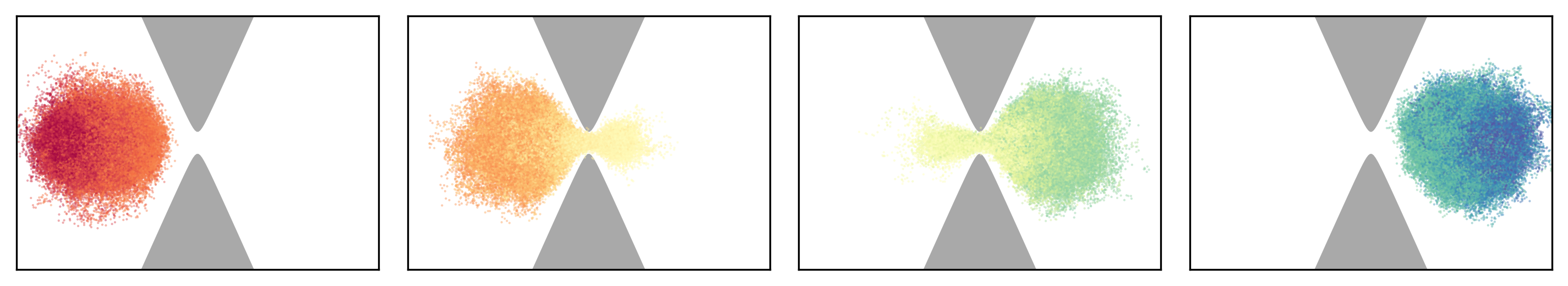}
  \end{minipage}  
  \caption{Learned TSBM dynamics for Stunnel (top), Gmm (middle), Vneck (bottom) with $d=2$ and $\sigma=1$. \textbf{(Left):} \underline{Without} control variate. \textbf{(Right):} \underline{With} control variate. These results are obtained after the first Markovian projection of the unidirectional variant.}
  \label{tsbm:fig:crowd_control_variate}
\end{figure}

\vfill

\newpage
\subsection{Single-cell experiments}
\label{tsbm:app:single_cell}

Single-cell trajectory inference aims to reconstruct the temporal evolution of a cell population from snapshot measurements. In typical single-cell RNA sequencing experiments, cells are profiled at several developmental or experimental time points, but the measurement process is destructive: a cell observed at one time cannot be observed again later. As a result, data collected at different times does not form individual trajectories. Consequently, the dataset provides marginal distributions over cellular states at selected time points, but no pairing between cells observed at different times. The task is to infer a continuous-time population dynamics whose time marginals are compatible with these snapshots.

\paragraph{Embryoid body dataset.}
We use the embryoid body (EB) single-cell RNA-seq dataset analyzed in the PHATE study of \cite{Moon2019} and later used by \cite{tong2020trajectorynet}. This dataset consists of a developmental single-cell RNA-seq time course, in which each cell is represented by a high-dimensional gene-expression vector and annotated with its collection time. More generally, single-cell datasets can be viewed as matrices whose rows correspond to cells and whose columns correspond to molecular features, such as genes in scRNA-seq.

For the EB dataset, the biological goal is to infer how a population of cells evolves over developmental time, as cells differentiate and branch toward distinct lineages. From a trajectory-inference perspective, the observations collected at each day define empirical population distributions over cellular states. The dataset contains $18{,}203$ cells, with snapshots at days $0$, $6$, $12$, $18$, and $24$. We perform the experiments in $2$-, $5$- and $50$-dimensional PCA spaces.

In our formulation, days $0$ and $24$ define the initial and terminal marginals, respectively, while the intermediate snapshots at days $6$, $12$, and $18$ are treated as intermediate observations. We rescale the physical time interval $[0,24]$ to a computational horizon $[0,T]$ with $T=2$, so that these intermediate observations correspond to $t=0.5$, $t=1$, and $t=1.5$.

\paragraph{Sparse-observation regime.}
Rather than using the full set of available cells at intermediate observation times, as is common in the generative modeling literature, e.g., \cite{tong2020trajectorynet,neklyudov2023computational,kapusniak2405metric}, \emph{we deliberately retain only a subset of these observations}, fixed to $25\%$, as illustrated in \Cref{tsbm:app:single_cell_obs}, and incorporate them through a Twisted SB formulation.

\begin{figure}[h!]
  \centering
\begin{minipage}{0.4\textwidth}
    \centering
    \includegraphics[width=\linewidth]{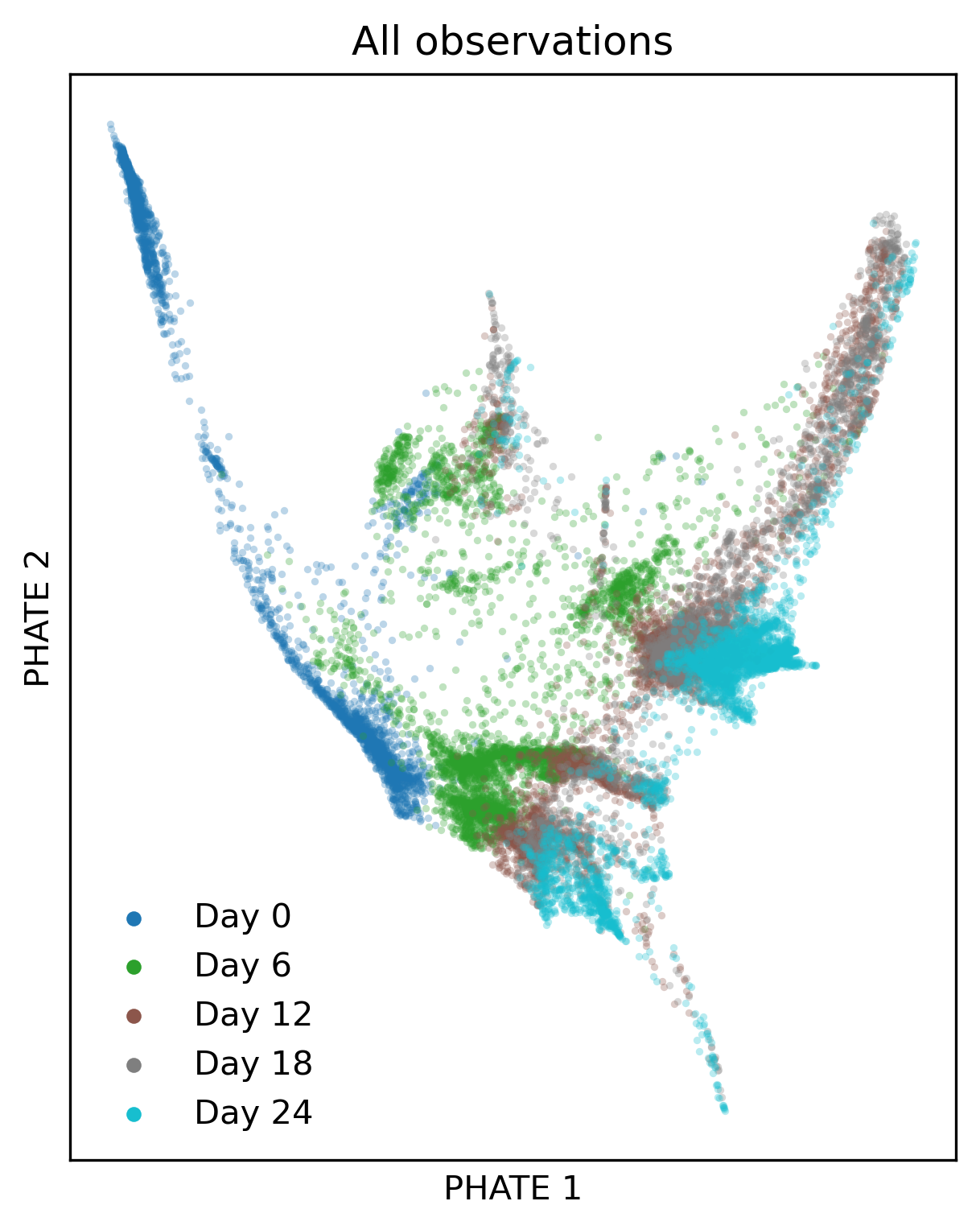}
  \end{minipage}
  \hfill
\begin{minipage}{0.4\textwidth}
    \centering
    \includegraphics[width=\linewidth]{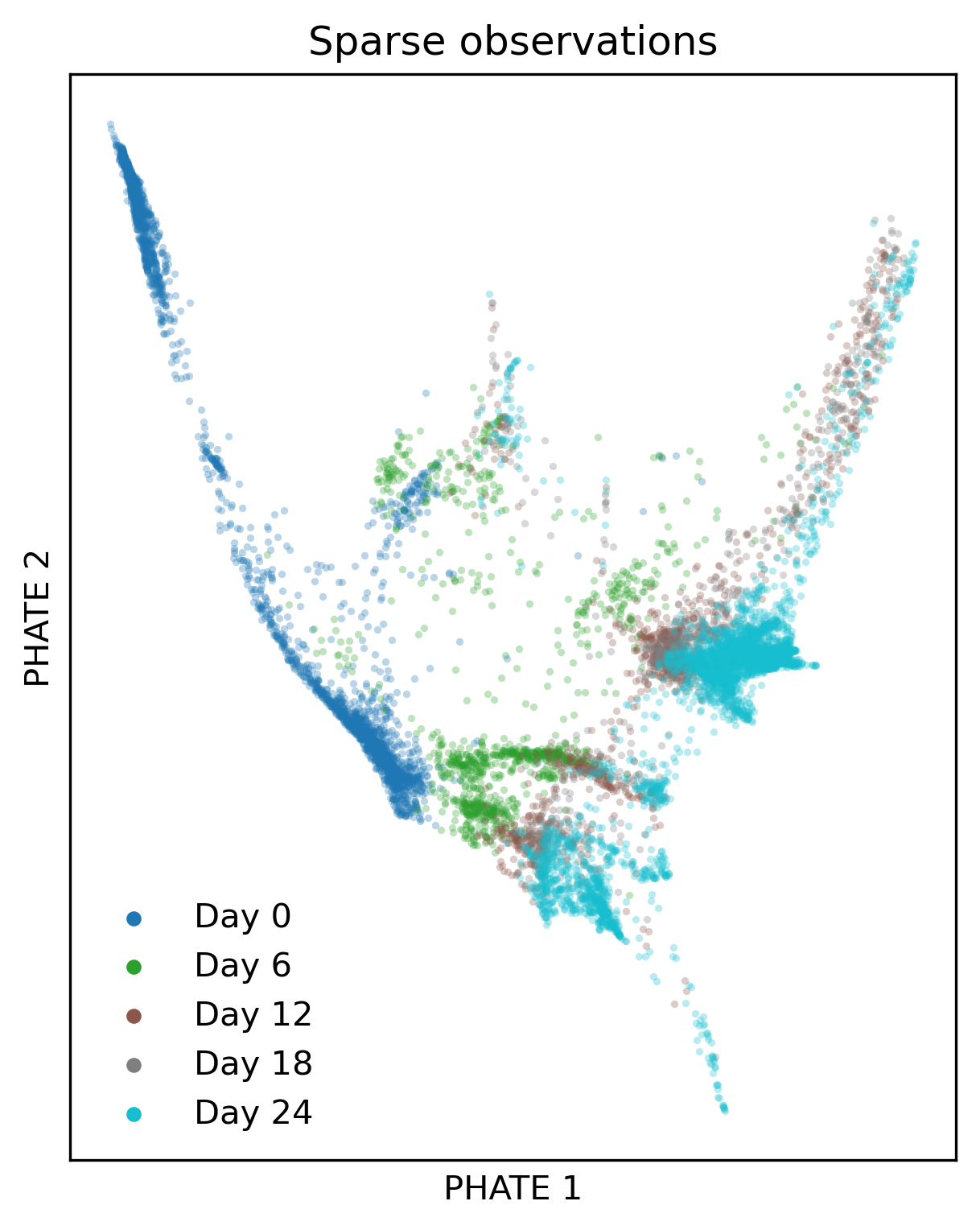}
  \end{minipage}  
  \caption{2D Single-cell data splitted by snapshot times : endpoint samples (day 0, day 24) and intermediary samples (day 6, day 12, day 18). \textbf{(Left):} Multi-marginal setting : all intermediary observations are available. \textbf{(Right):} Sparse-observation setting considered in our experiments : only $25\%$ of the intermediary observations are available, but the endpoint marginals are complete.}
  \label{tsbm:app:single_cell_obs}
\end{figure}

This choice is motivated by practical considerations rather than being purely artificial. In real single-cell studies, measurements can be costly, invasive, or technically limited, so observations may only be available at a few time points or for a restricted number of cells. The problem is therefore better viewed not as a fully specified multi-marginal interpolation task, but as a partially observed trajectory-inference problem, a regime for which TSBM is particularly well suited.

In this setting, the Brownian prior provides a baseline dynamics that bridges the source and target distributions without using additional biological information. Twisting this prior with a data-informed state cost then steers the dynamics toward regions compatible with the observed snapshots, while avoiding the hard enforcement of intermediate marginals that would arise in standard bridge-matching approaches. By contrast, when all intermediate marginal constraints are imposed, multi-marginal methods are more appropriate, as they are explicitly designed to interpolate between temporally ordered marginals.

Our experiments therefore focus on the sparse single-cell regime, where the goal is not to match a dense sequence of marginals, but to infer plausible continuous-time population dynamics from limited snapshot information. We rely on the SB formulation through the SOC objective in \eqref{tsbm:eq:extended_soc}, with a state cost specifically designed for this setting and detailed below. This choice precludes a direct comparison with methods that exploit all intermediate data, but enables a fair comparison with GSBM, which is trained under the same formulation as TSBM.

\paragraph{Design of the single-cell state cost.}
We now describe in detail the construction of the state cost used in our experiments, inspired by the seminal work of \cite{tong2020trajectorynet}. We present its definition from an abstract perspective, as it applies more broadly to sparse-observation settings beyond the single-cell experiments considered here. 

Let $K$ denote the number of snapshot times, including the boundary times, and write
$\{t_k\}_{k=1}^K \subset [0,T]$, with $t_1=0$ and $t_K=T$, for the corresponding increasing sequence of observation times ($K=5$ in our experiment). For each $k\in\{1,\ldots,K\}$, let
\[
\mathcal Z^k=\{z^k_1,\ldots,z^k_{M_k}\}\subset \mathbb R^d
\]
denote the collection of samples observed at time $t_k$, and set
$\mathcal Z=\bigcup_{k=1}^K \mathcal Z^k$. The endpoint samples are assumed to be fully observed, in contrast with the intermediate snapshots, which are only sparsely available. For $x\in\mathbb R^d$ and $z\in\mathbb R^d$, define
\begin{align}\label{tsbm:eq:distance_single_cell}
    r(x,z) = (\|x-z\|-h)_+\eqsp,
\end{align}
where $h>0$ is a distance threshold.

We recall the single-cell density potential introduced by \cite{tong2020trajectorynet}, inspired by biological priors,
\begin{align}\label{tsbm:eq:potential_traj}
     U(x) = \sum_{j=1}^{k_{NN}} r(x,z_{j(x)}) \eqsp ,
\end{align}
where $z_{j(x)}$ denotes the $j$-th nearest neighbor in $\mathcal{Z}$ of $x$ in Euclidean distance, and $k_{NN}\geq 1$ is a nearest-neighbor hyperparameter. We propose adapting this potential into an informative and differentiable state cost for our sparse-observation setting.

We first replace the non-smooth thresholded distance \eqref{tsbm:eq:distance_single_cell} by its smooth counterpart
\[
r(x,z) = \operatorname{Softplus}(\|x-z\|-h) \eqsp .
\]
Then, for each snapshot set $\mathcal Z^k$, we recursively define on $j=1,\ldots,k_{NN}$ the unnormalized weights
\begin{align*}
\textstyle\tilde w^{k,j}_m(x)
=
\left(1-\sum_{j'<j} w^{k,j'}_m(x)\right)_+
\exp\bigl(-\beta r(x,z^k_m)\bigr) \eqsp,
\qquad \forall z^k_m\in\mathcal Z^k \eqsp,
\end{align*}
with the convention that the empty sum is zero, along with their normalized version
\[
w^{k,j}_m(x)
=
\frac{\tilde w^{k,j}_m(x)}
{\sum_{m'=1}^{M_k}\tilde w^{k,j}_{m'}(x)},
\qquad \forall z^k_m\in\mathcal Z^k,
\]
where $\beta>0$ is an inverse-temperature parameter. We then define the snapshot-specific state cost
\[
V_k(x)
=
\frac{1}{k_{NN}}
\sum_{j=1}^{k_{NN}}
\sum_{m=1}^{M_k}
w^{k,j}_m(x)\, r(x,z^k_m) \eqsp .
\]
With this construction, observations closer to $x$ receive larger weights. In the limit $\beta\to+\infty$, the weights concentrate on the minimizers of $z\mapsto r(x,z)$ over $\mathcal Z^k$. Combined with the recursive masking in the definition of the weights, this recovers the contribution of the $k_{NN}$ nearest neighbors of $x$ within $\mathcal Z^k$, and \emph{therefore matches the potential \eqref{tsbm:eq:potential_traj} restricted to that snapshot} (up to renormalization by $k_{NN}$). In contrast, when $\beta\to 0$, the weights become nearly uniform over $\mathcal Z^k$, making the state cost weakly informative. Thus, $\beta$ controls the locality of the interaction with sparse observations, and mediates a trade-off between fidelity to the observed samples and generalization. In \Cref{tsbm:fig:ablation_beta_traj}, we illustrate the effect of $\beta$ on the simulated trajectories. We observe that lower values of $\beta$ induce unlikely trajectories with high transport cost, namely trajectories connecting distant data clusters. This behavior is expected, since uniform weighting over the observed data (obtained with low $\beta$) no longer exploits data locality. This motivates rather using large $\beta$, in line with the approach of \cite{tong2020trajectorynet}.

To exploit sparse observations across the full simulation time interval in our Twisted Schrödinger Bridge setting, we introduce a time-dependent state cost obtained by linearly interpolating between neighboring snapshot costs. For $t\in[0,T]$, let $k^-(t)$ and $k^+(t)$ denote the indices of the closest observation times immediately before and after $t$, respectively. We define
\[
V_t(x_t)=
\begin{cases}
\displaystyle
\frac{t_{k^+(t)}-t}{t_{k^+(t)}-t_{k^-(t)}}V_{k^-(t)}(x_t),
& \text{if } k^+(t)=K, \\[1.2em]
\displaystyle
\frac{t-t_{k^-(t)}}{t_{k^+(t)}-t_{k^-(t)}}V_{k^+(t)}(x_t),
& \text{if } k^-(t)=1, \\[1.2em]
\displaystyle
\frac{t-t_{k^-(t)}}{t_{k^+(t)}-t_{k^-(t)}}V_{k^+(t)}(x_t)
+
\frac{t_{k^+(t)}-t}{t_{k^+(t)}-t_{k^-(t)}}V_{k^-(t)}(x_t),
& \text{otherwise.}
\end{cases}
\]
This construction yields a state cost that is continuous in time and differentiable with respect to the state variable. Intuitively, at any time $t\in[0,T]$, $V_t$ combines information from nearby past and future observations, while reducing to the corresponding snapshot-specific cost at observation times. We note that more flexible, learnable interpolation weights could also be considered; however, in our experiments, the linear interpolation already provided satisfactory performance. To avoid introducing an artificial influence from the endpoint marginals, which have a different status from the sparse intermediate observations, we set their contribution to zero.

\paragraph{Experimental setting.}
For $d=2$, we use $h=0.1$ and $k_{NN}=5$, following the recommendation of \cite{tong2020trajectorynet}. In higher-dimensional settings, we keep $k_{NN}=5$, which provides a favorable trade-off between computational cost and expressivity, and choose $h$ manually according to the characteristic nearest-neighbor scale observed in the data. Specifically, we use $h=1$ for $d=5$ and $h=7$ for $d=50$. For numerical stability, we implement the state cost using logit weights rather than the normalized weights directly. To obtain smooth single-cell trajectories, we choose $\sigma=0.1$. In this setting, \emph{we only consider the bidirectional variant}, which was shown in the crowd navigation experiments to stabilize the training of TSBM, making it the most suitable choice for real-world applications.

\paragraph{Additional results.}
In \Cref{tsbm:fig:beta_100_snapshot}, we report the GSBM counterpart of \Cref{tsbm:fig:crowd_main} from \Cref{tsbm:sec:xps}. Interestingly, the estimated intermediate marginals appear to connect observation clusters more effectively under TSBM than under GSBM, in agreement with the 2D optimality metric reported in \Cref{tsbm:tab:single_cell}. For illustration, we also show the corresponding 2D snapshots for $\beta=20$ in \Cref{tsbm:fig:beta_20_snapshot}. In this case, no clear visual difference is observed between GSBM and TSBM, consistently with the metrics in \Cref{tsbm:tab:single_cell}.

\begin{figure}[h!]
  \centering
\begin{minipage}{0.49\textwidth}
    \centering
    \includegraphics[width=\linewidth]{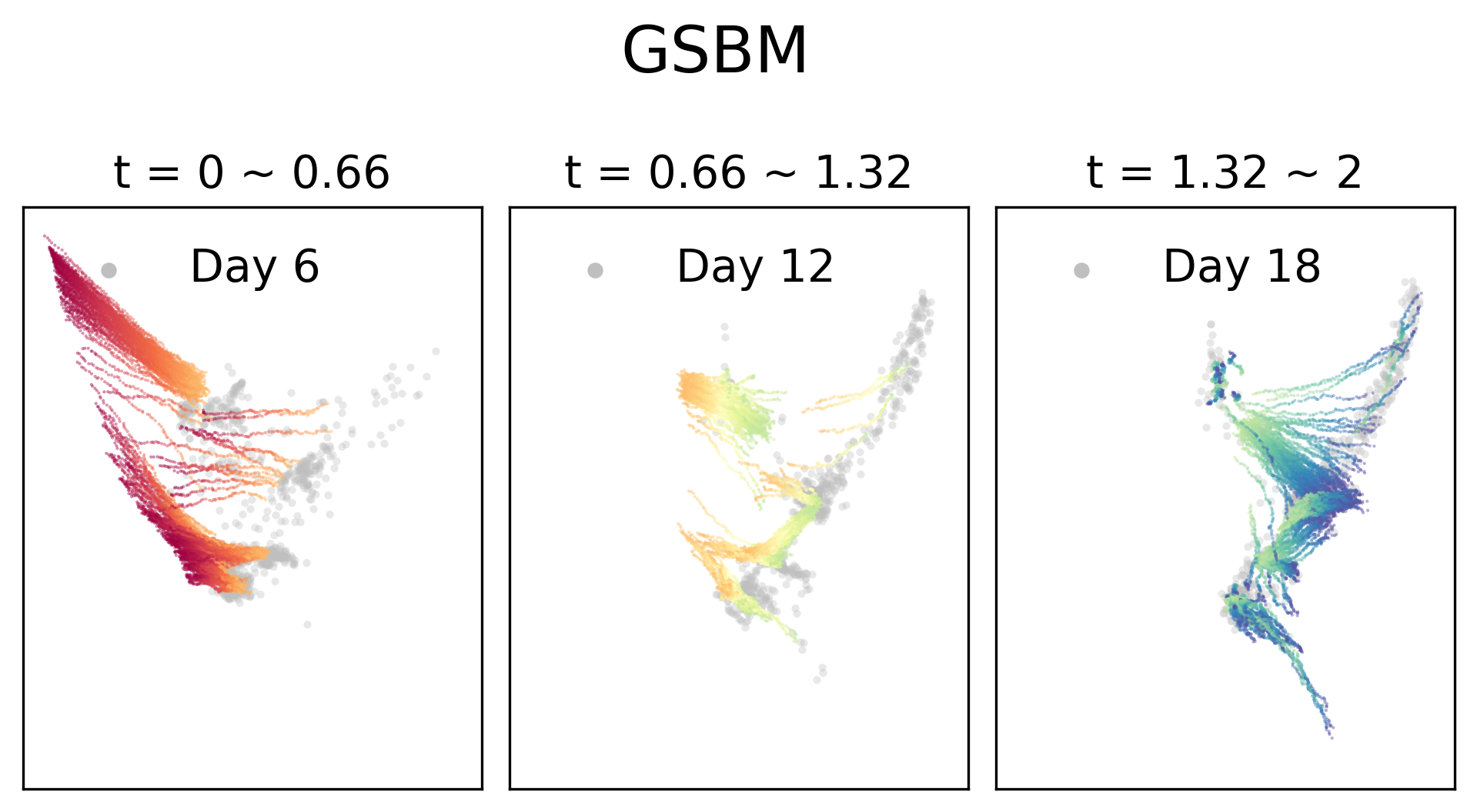}
    \vfill
    \includegraphics[width=\linewidth]{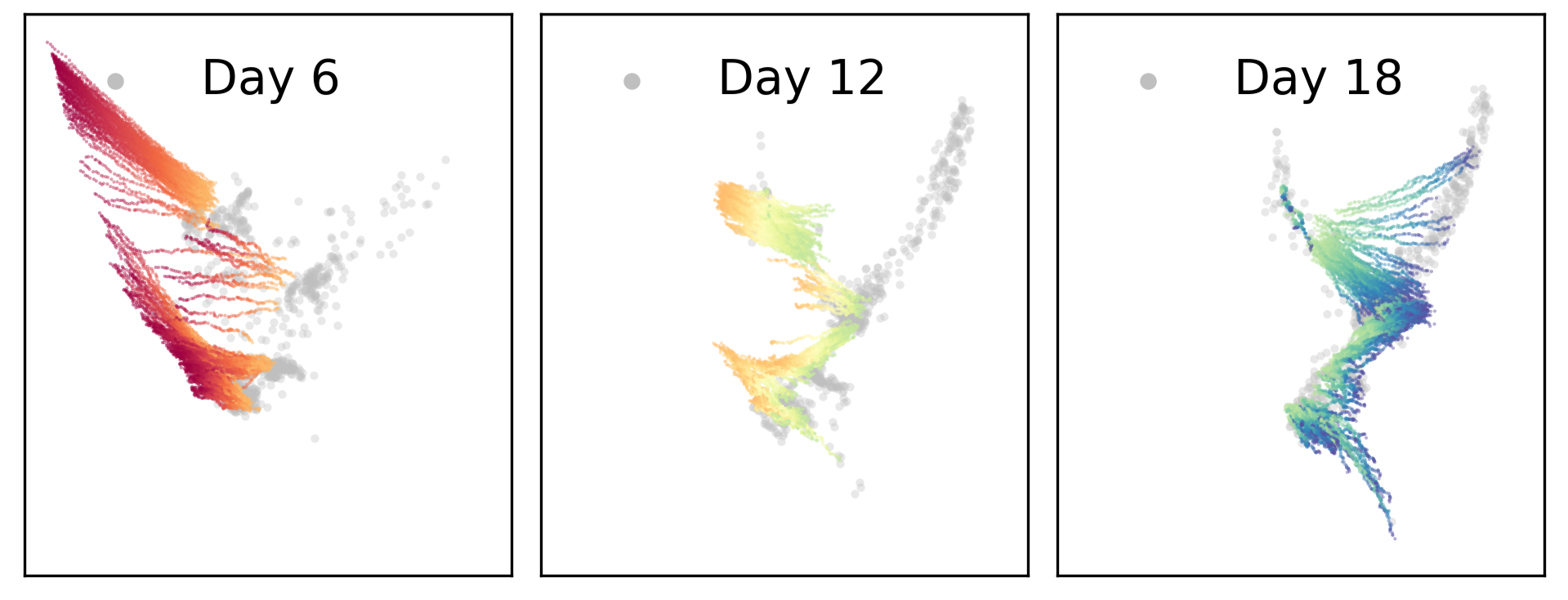}
  \end{minipage}
  \hfill
\begin{minipage}{0.49\textwidth}
    \centering
    \includegraphics[width=\linewidth]{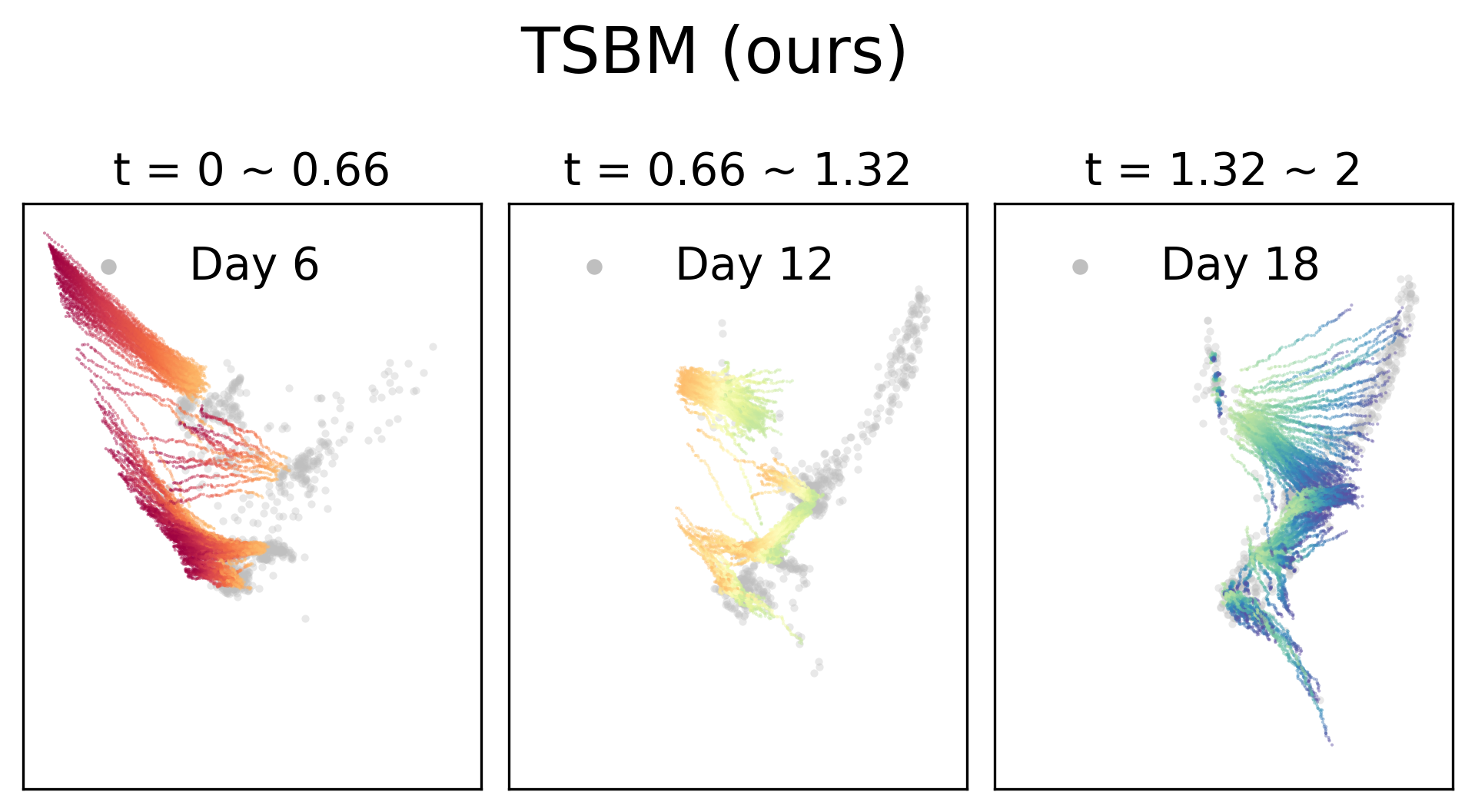}
    \vfill
    \includegraphics[width=\linewidth]{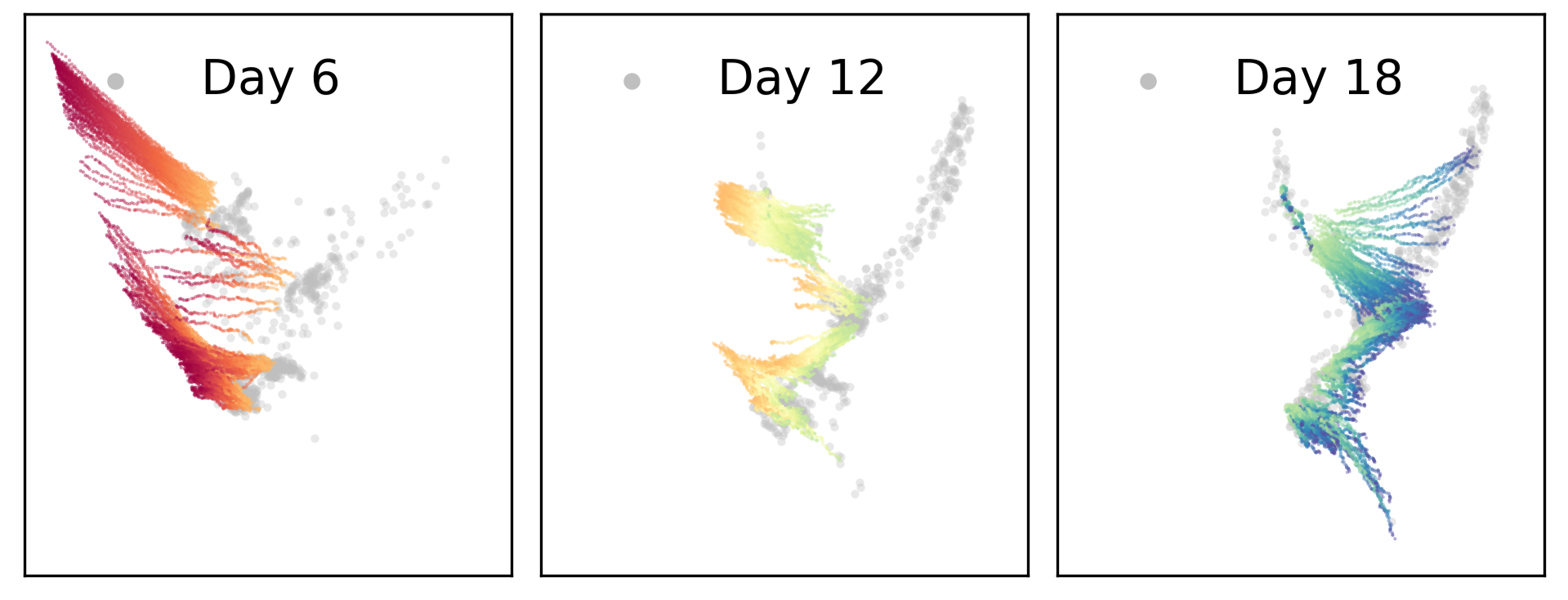}
  \end{minipage}  
  \caption{Learned TSB dynamics for 2D single-cell data for $\beta=20$ (top) and $\beta=100$ (bottom). Available intermediary observations that influence the trajectories are marked in grey.}
    \label{tsbm:fig:ablation_beta_traj}
\end{figure}

\begin{figure}[h!]
    \centering
    \includegraphics[width=\linewidth]{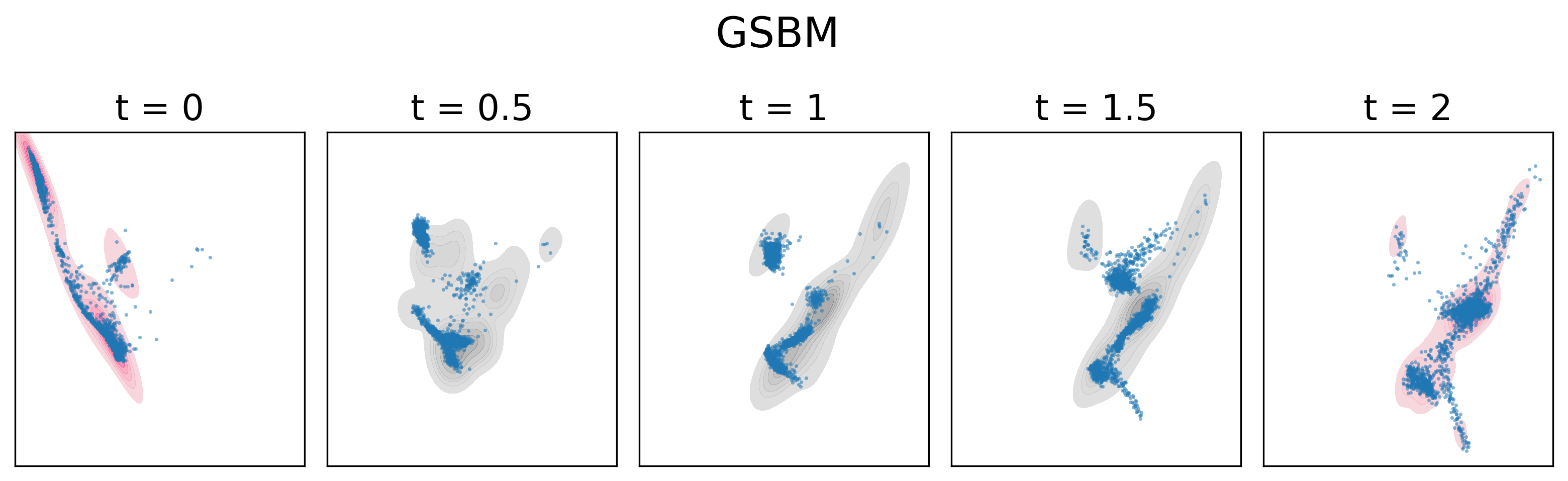}
    \caption{Snapshots of learned GSBM dynamics for 2D single-cell ($\beta=100$, $\sigma=0.1$) with respect to available observations (endpoint \emph{marginal} samples in pink, intermediary \emph{sparse} samples in grey). This is complementary to \Cref{tsbm:fig:crowd_main}.}
    \label{tsbm:fig:beta_100_snapshot}
\end{figure}

\begin{figure}[h!]
    \centering
    \includegraphics[width=\linewidth]{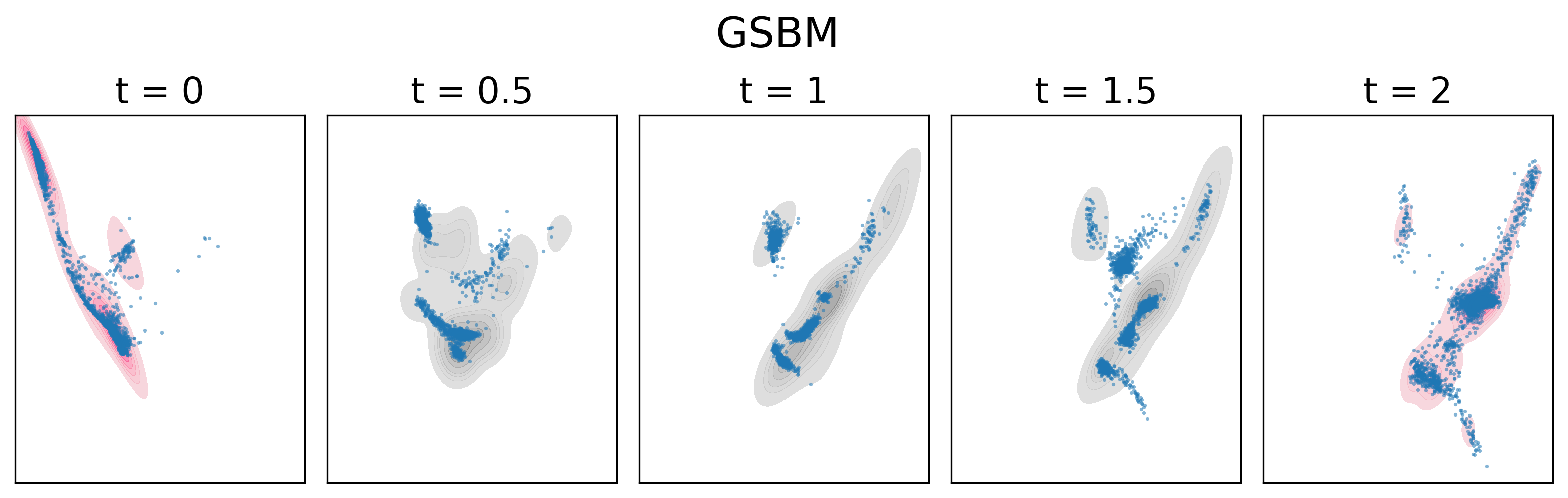}
    \vfill
    \includegraphics[width=\linewidth]{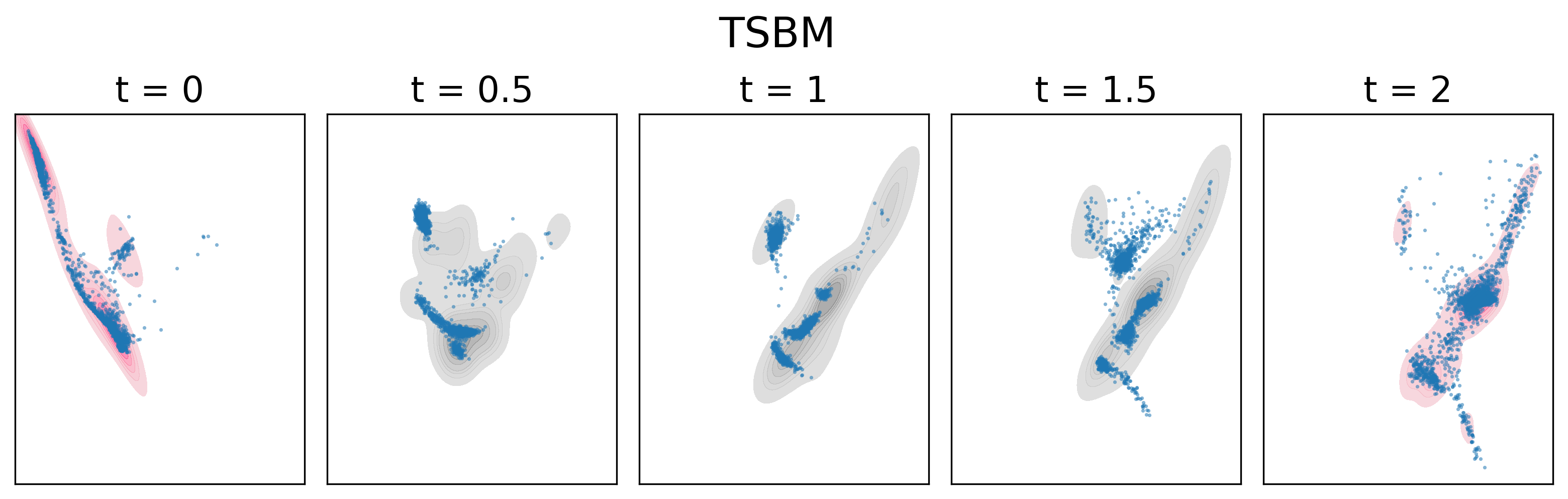}
    \caption{Snapshots of learned GSBM dynamics for 2D single-cell ($\beta=20$, $\sigma=0.1$) with respect to available observations (endpoint \emph{marginal} samples in pink, intermediary \emph{sparse} samples in grey). \textbf{(Top):} GSBM, \textbf{(Bottom):} TSBM.}
    \label{tsbm:fig:beta_20_snapshot}
\end{figure}


\end{document}